\documentclass[11pt]{article}
\usepackage[utf8]{inputenc}

\usepackage{mystyle}

\newcommand{\revise}[1]{\textcolor{black}{#1}}
\def\r#1{\textcolor{red}{#1}}
\def\bbeta{\boldsymbol \beta}
\def\bgamma{\boldsymbol \gamma}             
\def\bdelta{\boldsymbol {\delta}}        
\begin{document}

\title{\huge Understanding Implicit Regularization in Over-Parameterized Single Index Model}

\author{Jianqing Fan\thanks{Department of Operations Research and Financial Engineering, Princeton University; email: \texttt{\{jqfan, zy6,  mengxiny\}@princeton.edu}. Research supported by the NSF grant DMS-1662139 and DMS-1712591, the ONR grant N00014-19-1-2120, and the NIH grant 2R01-GM072611-16.}\quad\quad\quad Zhuoran Yang$^*$\quad\quad\quad Mengxin Yu$^*$}


\maketitle

	\begin{abstract}
		\revise{In this paper, we leverage over-parameterization to design  regularization-free algorithms for the high-dimensional single index model and provide theoretical guarantees for the induced implicit regularization phenomenon.} Specifically, we study both vector and matrix single index models where the link function is nonlinear and unknown, the signal parameter is either a sparse vector or a low-rank symmetric matrix, and the response variable can be heavy-tailed. To gain a better understanding of  the role played by  implicit regularization without excess technicality, we assume that the distribution of the covariates is known  a priori. For both the vector and matrix settings, we construct an  over-parameterized  least-squares loss function  by employing
			the score function transform  and a robust  truncation step designed specifically for heavy-tailed data. We propose to estimate the true parameter by applying regularization-free  gradient descent to the loss function.
			When the  initialization is   close to the origin and the stepsize is sufficiently small, we prove that the obtained solution achieves  minimax optimal statistical rates  of convergence in both the  vector and matrix cases. \revise{In addition, our experimental results support our theoretical findings and also demonstrate that our methods empirically outperform classical  methods with explicit regularization in terms of both $\ell_2$-statistical rate and variable selection consistency.}
	\end{abstract}

\section{Introduction}

With the astonishing empirical success in  various application domains such as computer vision \citep{voulodimos2018deep}, natural language processing \citep{otter2018survey, torfi2020natural}, and reinforcement learning \citep{arulkumaran2017deep,li2017deep}, deep learning \citep{lecun2015deep,goodfellow2016deep, fan2021selective} has become one of the most prevalent classes of machine learning methods.
When applying deep learning to supervised learning tasks such as regression and classification,
the regression function or classifier
is represented by a deep neural network,
which is learned by minimizing a  loss function of the network weights.
Here
the
loss function is
defined as the empirical risk function computed based on the training data and
the optimization problem is usually solved by gradient-based optimization methods.
Due to the nonlinearity of the  activation function and the multi-layer functional composition, the landscape of the  loss function is highly nonconvex, with many saddle points and local minima  \citep{dauphin2014identifying, swirszcz2016local, yun2019small}.
Moreover, oftentimes the neural network is over-parameterized in the sense that the total number of network weights exceeds the number of training data, making the regression or classification problem ill-posed
from a  statistical perspective.
Surprisingly, however,
it is often observed empirically that simple algorithms such as (stochastic) gradient descent tend to find the global minimum of the loss function despite  nonconvexity.
Moreover,
the obtained solution also generalizes well to unseen data with small test error
\citep{neyshabur2014search,  zhang2017understanding}.
These mysterious observations
cannot be fully explained by the classical theory of nonconvex optimization and generalization bounds via uniform convergence.

To understand such an intriguing phenomenon,
\cite{neyshabur2014search,  zhang2017understanding}
show empirically that
the generalization stems from an ``implicit regularization'' of the optimization algorithm. Specifically, they observe that, in over-parametrized statistical models,
although the optimization problems consist of bad local minima with large generalization error, the choice of optimization algorithm, usually a variant of gradient descent algorithm, usually guard the iterates from bad local minima and prefers the solution that generalizes well.
Thus, without adding any regularization term in the optimization objective, the implicit preference of the optimization algorithm itself plays the role of regularization.
Implicit regularization has been shown indispensable in training deep learning models \citep{neyshabur2014search,neyshabur2017geometry,  zhang2017understanding, keskar2017large, poggio2017theory,wilson2017marginal}.

\revise{With properly designed algorithm,}
\cite{gunasekar2017implicit} and \cite{Li2017AlgorithmicRI} provide empirical evidence and theoretical guarantees for the implicit regularization of gradient descent for least-squares regression with a two-layer linear neural network, i.e., low-rank matrix sensing.
They show that gradient descent biases towards the minimum nuclear norm solution when the initialization is close to the origin, stepsizes are sufficiently small, and no explicit regularization is imposed.
More specifically, when the true parameter is a rank $r$ positive-semidefinite matrix in $\RR^{d\times d}$, they rewrite the parameter as $\Ub\Ub^\top$, where $\Ub \in \RR^{d\times d}$, and propose to estimate the true parameter by updating $\Ub$ via gradient descent. \cite{Li2017AlgorithmicRI} proves that,
with $\tilde \cO(r^2 d)$ i.i.d. observations of the model,
gradient descent provably recovers the true parameter with accuracy, where $\tilde \cO(\cdot)$ hides absolute constants and poly-logarithmic terms.
Thus, in over-parametrized matrix sensing problems,  the implicit regularization of gradient descent can be viewed as equivalent to adding a nuclear norm penalty explicitly.  See also \cite{arora2019implicit} for a related topic on deep linear network.

Moreover,
\cite{zhao2019implicit,Vaskevicius2019} recently \revise{design a noval regularization-free algorithm and} study the
implicit regularization of gradient descent for  high-dimensional linear regression with a sparse signal parameter, which is a vector in $\RR^p$ with $s$ nonzero entries.
They propose to re-parametrize the parameter using two vectors in $\RR^p$ via the Hadamard product and estimate the true parameter via un-regularized gradient descent with proper initialization, stepsizes, and the number of iterations.
They prove independently that,  with  $ n =   \cO(s^2 \log p )$ i.i.d. observations, gradient descent yields an estimator of the true parameter with the optimal statistical accuracy.
More interestingly, when the nonzero entries of the true parameter all have sufficiently large magnitude, the proposed estimator attains the oracle $\cO(\sqrt{ s\log s /n})$ rate that is independent of the ambient dimension $p$.
Hence, for sparse linear regression, the implicit regularization of  gradient descent has the same effect as the folded concave penalties \citep{fan2014strong} such as   smoothly clipped absolute deviation
(SCAD) \citep{fan2001variable} and minimax concave penalty (MCP) \citep{zhang2010nearly}.

\revise{The aforementioned works all design algorithms and establish theoretical results for linear statistical models with light-tailed noise, which is slightly restricted since linear models with sub-Gaussian noise only comprise a small proportion of the models of interest in statistics. For example, in the field of finance, linear models only bring limited contributions and the datasets are always corrupted by heavy-tailed noise. Thus, one questions is left open:}  
\revise{
\begin{center}
	Can we leverage over-parameterization and  implicit regularization  to establish statistically accurate estimation procedures for a more general class of  high-dimensional statistical models with possibly heavy-tailed data?
	\end{center}}

In this work, we focus on the single index model, where the response variable $Y$ and the covariate $ X $ satisfy $Y= f( \langle X, \beta^* \rangle) + \epsilon$,  with  $\beta^*$ being the true parameter, $\epsilon  $ being the random noise, and $f \colon \RR \rightarrow \RR$ being an unknown (nonlinear) link function.
Here $\beta^*$ is either a $s$-sparse vector in $\RR^p$ or a rank $r$ matrix in $\RR^{d\times d}$.
Since $f$ is unknown,
the norm of $\beta $ is not identifiable.
Thus, for the vector and matrix cases respectively,
we further assume that the $\ell_2$- or Frobenius norms of $\beta^*$ are  equal to one.
Our goal is to recover the true parameter $\beta^*$ given $n$ i.i.d. observations of the model. {\revise{Such a  model can be viewed as the misspecified version of the compressed sensing \citep{donoho2006compressed, candes2008restricted} and phase retrieval \citep{shechtman2015phase,candes2015phase} models, which corresponds to the identical and quadratic link functions respectively.}}

In a single index model, due to the unknown link function, it is infeasible to directly estimate  $\beta^*$ via nonlinear least-squares. Moreover, jointly minimizing the least-squares loss function with respect to $\beta^*$ and $f$ is computationally intractable.
To overcome these challenges,
a recent line of research  proposes to estimate $\beta^*$ by the method of moments when the distribution of $X$ is known. This helps us provide a deep understanding on the implicit regularization induced by over-parameterization in the nonlinear models without excessive technicality and eliminate other complicated factors that convolve insights. Specifically, when $X$ is a standard Gaussian random variable, Stein's identity \citep{stein1972bound}  implies that the expectation of  $Y \cdot X $ is proportional to $\beta^*$.
Thus, despite the nonlinear link function,  $\beta^*$ can be accurately estimated by neglecting $f$ and fitting a regularized   least-squares regression. In particular, when $\beta^*$ is a sparse vector, \cite{plan2016generalized,plan2017high} prove that the Lasso estimator achieves the optimal statistical rate of convergence.
Subsequently, such an approach has been extended to the cases beyond Gaussian covariates.  In particular,
\cite{goldstein2018structured,wei2018structured, goldstein2019non} allow the
 covariates to  follow an elliptically symmetric distribution that can be heavy-tailed.
 In addition, utilizing a generalized version of  Stein's identity \citep{stein2004use},
 \cite{yang2017high} extends the Lasso approach to the setting where the covariate $X$ has a known density $p_0$. Specifically, when $p_0$ is known, we can define the score function $S_{p_0}(\cdot)$ as $S_{p_0}(\cdot) =- \nabla  \log p_0(\cdot)$, which enjoys the property that $\mathbb{E}[ Y \cdot S_{p_0} (X)]$ identifies the direction of $\beta^*$. Thus, the true parameter can be estimated by via an $M$-estimation problem with $S_{p_0} (X)$ served as the covariate.

\revise{To answer the question given above, in this work, we leverage over-parameterization to design regularization-free algorithms for single index model and provide theoretical guarantees for the induced implicit regularization phenomenon.} 
\revise{To be more specific, }we first adopt the quadratic loss function in \cite{yang2017high} and  rewrite the parameter of interest by over-parameterization. When $\beta^*$ is a sparse vector in $\RR^p$, we adopt a Hadamard product parameterization \citep{hoff2017lasso,zhao2019implicit,Vaskevicius2019} and write $\beta^*  $ as $\wb\odot \wb-\vb\odot \vb$, where both $\wb$ and $\vb $ are vectors in $\RR^p$.
We propose to minimize the loss function as a function of the new parameters via gradient descent, where both $\wb$ and $\vb$ are initialized near an all-zero vector and the stepsizes are fixed to be a sufficiently small constant $\eta > 0$.
Furthermore, when $\beta^*$ is a low-rank matrix, we similarly represent  $\beta^*$ as $\Wb \Wb^\top - \Vb \Vb^\top$ and propose to recover $\beta^*$ by applying the gradient descent algorithm to the quadratic loss function under the new parameterization.

Furthermore, the analysis of our algorithm faces the following two challenges. First, due to over-parameterization, there exist exponentially many stationary points of the population loss function that are far from the true parameter.
Thus, it seems that the gradient descent algorithm would be likely to return a stationary point that incurs a large error.
Second, both the response $Y$ and the score $S_{p_0}(X)$ can be heavy-tailed random variables.
Thus, the gradient of the empirical loss function can deviate significantly from its expectation, which poses an additional challenge to establishing the statistical error of the proposed estimator.

To overcome these difficulties, in our algorithm, instead of estimating $\mathbb{E}[ Y \cdot S_{p_0} (X)]$ by its empirical counterpart, we construct  robust estimators via proper truncation techniques, which have been widely applied in high-dimensional $M$-estimation problems with heavy-tailed data \citep{fan2016shrinkage, zhu2017taming, wei2017estimation, Minsker2018, fan2018robust, ke2019user, minsker2020robust}.
These robust estimators are then employed to compute the update directions of the gradient descent algorithm.
Moreover, despite the seemingly perilous loss surface, we prove that, when initialized near the origin and sufficiently small stepsizes, the gradient descent algorithm guard the iterates from bad stationary points.
More importantly, when the number of iterations is properly chosen, the obtained estimator provably enjoys (near-)optimal $\cO( \sqrt{s \log p / n} )$ and $\cO(\sqrt{r d \log d/ n})$  $\ell_2$-statistical rates under the sparse and low-rank settings, respectively.
Moreover, for sparse $\beta^*$, when the magnitude of the nonzero entries is sufficiently large, we prove that our estimator enjoys an oracle  $\cO(\sqrt{ s\log n/n})$ $\ell_2$-statistical rate, which   is independent of the dimensionality $p$.
Our proof is based on a jointly statistical and computational analysis of the gradient descent dynamics.
Specifically,
we decompose the iterates into  a signal part and a noise part, where the signal part share the same sparse or low-rank structures as the true signal and the noise part are orthogonal to the true signal.
We prove that the signal part converges to the true parameter  efficiently  whereas the  noise part accumulates at a rather slow rate and thus remains small for a sufficiently large number of iterations. Such a dichotomy between the signal and noise parts characterizes the implicit regularization of the gradient descent algorithm and enables us to establish the statistical error of the final estimator.

\revise{Furthermore, our method has several merits compared with  classical regularized methods. From the theoretical perspective, our strengths are two-fold. First, as we mentioned in the last paragraph, under mild conditions, our estimator enjoys oracle statistical rate whereas the most commonly used $\ell_1$-regularized method always results in large bias. In this case, our method is equivalent with adding folded-concave regularizers (e.g. SCAD, MCP) to the loss function. Second, for all estimators inside the wide optimal time interval, our range of choosing the truncating parameter to achieve variable selection consistency (rank consistency) is much wider than classical regularized methods. Thus, our method is more robust than all regularized methods in terms of selecting the truncating parameter. Meanwhile, from the aspect of applications, our strengths are three-fold. First, in terms of $\ell_2$-statistical rate, numerical studies show that our method generalizes even better than adding folded-concave penalties. Second, from the aspect of variable selection, experimental results also show that the robustness of our method helps reduce false positive rates greatly. Last but not least, as we only need to run gradient descent and the gradient information is able to be efficiently transferred among different machines, our method is easier to be paralleled and generalized to large-scale problems. Thus, our method can be applied to  modern machine learning applications such as federated learning.}

To summarize, our contribution is \revise{several-fold}.
First, for sparse and low-rank single index models where the random noise is possible heavy-tailed, we employ a
quadratic loss function based on a robust estimator of $\mathbb{E}[ Y \cdot S_{p_0} (X)]$ and propose to estimate $\beta^*$ by combining over-parameterization and regularization-free gradient descent.
Second, we prove that, when the initialization, stepsizes, and stopping time of the gradient descent algorithm are properly chosen, the proposed estimator achieves optimal statistical rates of convergence up to logarithm terms under both the sparse and low-rank settings. \revise{This captures the implicit regularization phenomenon induced by our algorithm. Third, in order to corroborate our theories, we did extensive numerical studies. The experimental results support our theoretical findings and also show that our method outperforms classical regularized methods in terms of both $\ell_2$-statistical rates and variable selection consistency.}  

\subsection{Related Works}

Our work belongs to the recent line of research on understanding the implicit regularization of
gradient-based optimization methods in various statistical models.
For over-parameterized logistic regression with separable data, \cite{soudry2018implicit} proves that the iterates of the  gradient descent algorithm converge to the max-margin solution. This work is  extended by
\cite{ji2018risk,ji2018gradient, gunasekar2018implicit,nacson2019convergence, ji2019refined} for studying
linear classification problems with other loss functions, parameterization, or training algorithms.
 \cite{montanari2019generalization,deng2019model} study the asymptotic generalization error of the max-margin classifier under the over-parameterized regime. Recently, for neural network classifiers, \cite{xu2018will, lyu2019gradient, chizat2020implicit} prove that gradient descent converges to the max-margin classifier under certain conditions.
In addition, various works have established the implicit regularization phenomenon for regression.
For example,
for low-rank matrix sensing, \cite{Li2017AlgorithmicRI,gunasekar2017implicit} show that, with over-parameterization, unregularized  gradient descent finds the optimal solution efficiently.
For various models including matrix factorization,
\cite{ma2019implicit} proves that the iterates of gradient descent stays in a benign region that enjoys linear convergence.
\cite{arora2019implicit, gidel2019implicit} characterize the  implicit regularization of gradient descent in deep matrix factorization.
For sparse linear regression, \cite{zhao2019implicit, Vaskevicius2019} prove  that, with re-parameterization,  gradient descent finds an estimator which attains the optimal statistical rate of convergence.
\cite{gunasekar2018characterizing} studies the implicit regularization of generic optimization methods in over-parameterized  linear regression and classification.
Furthermore, for nonlinear regression models,
\cite{du2018algorithmic} proves that, for neural networks with homogeneous action functions,  gradient descent automatically balances the weights across different layers.
 \cite{oymak2018overparameterized, azizan2019stochastic} show  that, in over-parameterized models, when the loss function satisfies certain conditions,  both gradient descent and mirror descent algorithms converge  to one of the global minima which is the closest to  the initial point.

Moreover, in linear regression, when initialized from the origin, gradient descent converges to the minimum $\ell_2$-norm (min-norm) solution. Besides, as  shown in \cite{soudry2018implicit}, gradient descent converges to the max-margin classifier in over-parameterized logistic regression.
There is a recent line of works on characterizing the risk of the min-norm and max-margin estimators under the over-parametrized setting where $p$ is larger than $n$.
See, e.g, \cite{belkin2018reconciling, belkin2019two, liang2018just,bartlett2020benign,hastie2019surprises, derezinski2019exact, ma2019generalization, mei2019generalization, montanari2019generalization, kini2020analytic,muthukumar2020harmless} and the references therein.
These works prove that, as $p$ grows to be larger than $n$, the risk first increases and then magically decreases after a certain threshold. Thus, there exists another bias-variance tradeoff in the over-parameterization regime. Such a mysterious phenomenon is coined by \cite{belkin2018reconciling} as the  ``double-descent'' phenomenon, which is conceived as an outcome of   implicit regularization and over-parameterization.

Furthermore,
there exists a large body of literature on the optimization and generalization of training over-parameterized neural works.
In a line of research, using mean-field approximation,
\cite{chizat2018global, rotskoff2018neural, sirignano2018mean,mei2018mean, mei2019mean,wei2019regularization} propose various optimization approaches with provable convergence  to the global optima of the training loss.
Besides, with different scaling, another line of works study the convergence and generalization of gradient-based methods for over-parameterized neural networks under the framework of the  neural tangent kernel (NTK) \citep{jacot2018neural}.
See, e.g.,
\cite{du2018gradient2, du2018gradient, zou2018stochastic, chizat2018note, allen2018learning, allen2018convergence,  jacot2018neural, cao2019generalization, arora2019fine,  lee2019wide, ma2019comparative,  yehudai2019power,  bai2019beyond, huang2020deep} and the references therein.
Their theory shows that a sufficiently wide neural network can be well approximated by the random feature model \citep{rahimi2008random}. Then, with sufficiently small stepsizes, (stochastic) gradient descent algorithm implicitly forces the network weights to stay in a neighborhood of the initial value. Such an implicit regularization phenomenon enables these papers to establish convergence rates and generalization errors for neural network training.

Furthermore, our work is also closely related to the large body of literature on single index models.
Single index model has been
 extensively studied in the low-dimensional setting.  See, e.g.,   \cite{ han1987non,mccullagh1989generalized,hardle1993optimal, carroll1997generalized, xia1999extended, horowitz2009semiparametric} and the references therein.
Most of these works propose to jointly estimate $\beta^*$ and $f$ based on solving the global optimum of  nonconvex $M$-estimation  problems. Thus, these methods can be computationally intractable in the worst case.
Under the Gaussian or elliptical assumption on the covariates, a more related line of research proposes efficient estimators of the direction of $\beta^*$ based on factorizing a set of moments involving $X$ and $Y$. See, e.g.,
\cite{brillinger1982generalized,li1989regression, li1991sliced, li1992principal, duan1991slicing,cook1998principal, cook1999dimension, cook2005sufficient} and the references therein.
Furthermore, for single index models in the high-dimensional setting, \cite{thrampoulidis2015lasso, genzel2016high, plan2016generalized, plan2017high, neykov2016l1, zhang2016consistency, yang2017high,goldstein2018structured,wei2018structured, goldstein2019non,na2019high}
propose to estimate the direction of $\bbeta^*$ via $\ell_1$-regularized regression.
Most of these works impose moment conditions inspired by
 \cite{brillinger1982generalized},
which ensures that the direction of $\beta^*$ can be recovered from the covariance of $Y$ and a transformation of $X$.
Among these papers, our work is closely related to  \cite{yang2017high} in that we adopt the same loss function based on generalized Stein's identity \citep{stein2004use}.
That work only studies the statistical error of the $\ell_1$-regularized estimator, which is a solution to a convex optimization problem.
In comparison, without any regularization, we construct estimators based on over-parameterization and gradient descent.
We provide both statistical and computational errors of the proposed algorithm and establish a similar statistical rate of convergence as in \cite{yang2017high}.
Moreover, when each nonzero entry of $\beta^*$ is sufficiently large, we further obtain an oracle statistical rate which cannot be obtained by the $\ell_1$-regularized estimator.
Furthermore, \cite{jiang2014variable, neykov2016agnostic, yang2017estimating, tan2018convex, lin2018consistency, yang2019misspecified, balasubramanian2018tensor, babichev2018slice, qian2019sparse,lin2019sparse} generalize models such as misspecified phase retrieval \citep{candes2015phase}, slice inverse regression \citep{li1991sliced}, and multiple index model \citep{xia2008multiple} to the high-dimensional setting.
The estimators proposed in these works are based on second-order moments involving $Y$ and $X$ and require $\ell_1$-regularization, hence are not directly comparable with our estimator.

\subsection{Notation}
In this subsection, we give an introduction to our notations. Throughout this work, we use $[n]$ to denote the set $\{1,2,\dots,n\}$. For a subset $S$ in $[n]$ and a vector $\ub$, we use $\ub_{S}$ to denote the vector whose $i$-th entry is $u_i$ if $i\in S$ and $0$ otherwise. For any vector $\ub$ and $q\ge 0$, we use $\|\ub\|_{\ell_q}$ to represent the vector $\ell_q$ norm. In addition, the inner product $\langle \ub, \vb \rangle $ between any pair of vectors $\ub,\vb$ is defined as the Euclidean inner product $\ub^\top\vb$. Moreover, we define $\ub\odot\vb$ as the Hadamard product of vectors $\ub,\vb$.
For any given matrix $\Xb\in\RR^{d_1\times d_2}$, we use $\|\Xb\|_{\oper}$, $\|\Xb\|_{F}$ and $\|\Xb\|_{*}$ to represent the operator norm, Frobenius norm and nuclear norm of matrix $\Xb$ respectively. In addition, for any two matrices $\Xb,\Yb\in\RR^{d_1\times d_2}$, we define their inner product $\langle \Xb,\Yb\rangle$ as $\langle \Xb,\Yb\rangle=\textrm{tr}(\Xb^\top\Yb)$. Moreover, if we write  $\Xb\succcurlyeq 0$ or $\Xb\preccurlyeq 0$, then the matrix $\Xb$ is meant to be positive semidefinite or negative semidefinite.
We let $\{a_n,b_n\}_{n\ge 1}$ be any two positive series.  We write $a_n\lesssim b_n$ if there exists a universal constant $C$ such that $a_n\le C\cdot b_n$ and we write $a_n\ll b_n$ if $a_n/b_n\rightarrow 0$. In addition, we write $a_n\asymp b_n$, if we have $a_n\lesssim b_n$ and $b_n\lesssim a_n$ and  the notations of $a_n=\cO(b_n)$ and $a_n=o(b_n)$ share the same meaning with $a_n\lesssim b_n$ and $a_n\ll b_n$. Moreover, $a_n=\tilde{\cO}(b_n)$ means $a_n\le Cb_n$ up to some logarithm terms.   Finally, we use $a_n=\Omega(b_n)$ if there exists a universal constant $c>0$ such that $a_n/b_n\ge c$ and we use $a_n=\Theta(b_n)$ if $c\le a_n/b_n \le C$ where $c, C>0$ are universal constants.

\subsection{Roadmap}
The organization of our paper is as follows. We introduce the background knowledge in \S\ref{prelimary}.
In
 \S\ref{sectionvec} and
 \S\ref{sectionmat} we investigate the implicit regularization effect of gradient descent
 in over-parameterized SIM under the vector and matrix settings, respectively.
  Extensive simulation studies are presented in  \S\ref{sectsimu} to corroborate  our theory.


\section{Preliminaries}\label{prelimary}
In this section, we introduce the phenomenon of implicit regularization via over-parameterization,  high dimensional single index model,
and  generalized  Stein's identity \citep{stein2004use}.

\subsection{Related Works on Implicit Regularization}\label{relatedwork}

Both \cite{gunasekar2017implicit} and \cite{Li2017AlgorithmicRI} have studied least squares objectives over positive semidefinite matrices $\beta\in \RR^{d\times d}$ of the following form
\begin{align}\label{loss1guna2017}
\min_{\beta\succcurlyeq 0}F(\beta)=\frac{1}{n}\sum_{i=1}^{n}\left(y_i-\langle \Xb_i, \beta\rangle \right)^2,
\end{align}
where the labels $\{y_i\}_{i=1}^{n}$ are generated from linear measurements $y_i=\langle\Xb_i,\beta^{*} \rangle,i\in[n],$ with $\beta^{*}\in\RR^{d\times d}$ being positive semidefinite and low rank. Here $\beta^*$ is of rank $r$ where $r$ is much smaller than $d$.
Instead of working on parameter $\beta$ directly, they  write  $\beta$ as $ \beta = \Ub\Ub ^\top$ where $\Ub \in \RR^{d\times d}$,  and study the optimization problem related to $\Ub$,
\begin{align}\label{loss2guan2017}
\min_{\Ub\in \RR ^{d\times d}}f(\Ub)=\frac{1}{2n}\sum_{i=1}^{n}\left(y_i-\langle \Xb_i, \Ub\Ub^\top\rangle \right)^2.
\end{align}
The least-squares problem in \eqref{loss2guan2017} is over-parameterized because
here $\beta$ is parameterized by $\Ub$, which has $d^2$ degrees  of freedom, whereas $\beta^*$, being a rank-$r$ matrix, has $\cO(rd)$ degrees of freedom.
 \cite{gunasekar2017implicit} proves that when  $\{\Xb_i\}_{i=1}^{m}$ are commutative and $\Ub$  is properly   initialized, if
 the gradient flow of \eqref{loss2guan2017} converges to a  solution $\hat \Ub$ such that $\hat \beta =    \hat \Ub \hat \Ub^\top$ is a globally optimal solution of \eqref{loss1guna2017},
 then $\hat \Ub$
 has the minimum nuclear norm over all global optima.   Namely,
\begin{align*}
&\hat{\beta}\in\argmin_{\beta\succcurlyeq 0}\|\beta\|_{*},
\\\text{subject  to~~}&\,\langle\Xb_i,\hat\beta\rangle=y_i,\quad  \forall i\in[n].
\end{align*}
\revise{Subsequently, \cite{Li2017AlgorithmicRI} assumes
$\{\Xb_{i}\}_{i=1}^{n}$ satisfy the restricted isometry property (RIP) condition \citep{candes2008restricted} and proves that by applying gradient descent to  \eqref{loss2guan2017}
with the initialization close to zero and sufficiently small fixed stepsizes, the near exact recovery of $\beta^*$ is achieved.} 

\revise{Recently, \cite{li2021towards} proves that the algorithm of gradient flow with infinitesimal initialization on the general covariate of \eqref{loss2guan2017} tends to be equivalent to the Greedy Low-Rank Learning (GLRL) algorithm, which is a greedy rank minimization algorithm. Results in \cite{gunasekar2017implicit} with commutable $\{\Xb_i\}_{i=1}^{m}$ serves as a special case to \cite{li2021towards}.  }

As for noisy statistical model, both \cite{zhao2019implicit} and \cite{Vaskevicius2019} study over-parameterized high dimensional noisy linear regression problem independently. Specifically,  here the  response variables $\{y_i\}_{i=1}^{n}$ are generated from a linear model
\begin{align}\label{linearreg}
y_i=\xb_i^\top\beta^{*}+\epsilon_i,\,i\in[n],
\end{align}
where $\beta^{*} \in \RR^p$  and $\{\epsilon_i\}_{i=1}^{n}$ are i.i.d.\,\,sub-Gaussian random variables that are independent with the  covariates $\{\xb_i\}_{i=1}^{n}$. Moreover,
here $\beta^*$ has only $s$ nonzero entries where $s \ll p$.
Instead of adding sparsity-enforcing penalties,   they propose to estimate $\beta^*$ via  gradient descent with respect to $\wb,\vb$ on  a  loss function $L$, 
\begin{align}\label{linearloss1}
\min_{\wb\in \RR^p,\, \vb\in\RR^p}L(\wb,\vb)=\frac{1}{2n}\sum_{i=1}^{n}[\xb_i^\top(\wb\odot\wb-\vb\odot\vb)-y_i]^2,
\end{align}
where the parameter $\beta$   is over-parameterized as  $ \beta = \wb \odot\wb-\vb\odot\vb$.
 Under the  restricted isometry property (RIP) condition on the covariates,
 these works prove that, when the hyperparameters   is proper selected,
 gradient descent on \eqref{linearloss1} finds an estimator of $\beta^*$ with optimal statistical rate of convergence.



\subsection{High Dimensional Single Index Model}
In this subsection, we  first  introduce  the score functions associated with random vectors and matrices,
which are utilized in our algorithms.
Then  we formally define the high dimensional single index model (SIM) in both the vector and matrix settings.

\begin{definition}\label{defiscore}
	Let $\xb\in \RR^{p}$ be a random vector with density function $p_0(\xb):\RR^{p}\rightarrow\RR .$ The score function $S_{p_0}(\xb):\RR^p\rightarrow \RR^p$  associated with $\xb$ is defined as
	\begin{align*}
	S_{p_0}(\xb):=-\nabla_{\xb}\log p_0(\xb)=-\nabla_{\xb} p_0(\xb)/p_0(\xb).
	\end{align*}
	\end{definition}
Here the score function $S_{p_0}(\xb)$ relies on the density function $p_0(\xb)$ of the covariate $\xb$. In order to simplify the notations, in the rest of the paper, we omit the subscript $p_0$ from $S_{p_0}$ when the underlying distribution of $\xb $ is clear to us.

\textbf{Remark:} If the  covariate $\Xb\in\RR^{d\times d}$ is a random matrix whose entries are i.i.d. with a  univariate  density $p_0(x)\colon \RR\rightarrow \RR $, we then define the  score function $S(\Xb)\in\RR^{d\times d}$ entrywisely. In other words, for any $\{i,j\}\in[d]\times[d]$, we obtain
\begin{align}\label{scoremat}
S(\Xb)_{i,j}:=-p_0'(\Xb_{i,j})/p_0(\Xb_{i,j}).
\end{align}
Next, we introduce the first-order general Stein's identity.
\begin{lemma}(First-Order General Stein's Identity, \citep{stein2004use})\label{genstein}
	We assume  that the  covariate $\mathbf{x}\in \RR^{p}$ follows a distribution with density function $p_0(\xb): \RR^{p} \rightarrow \RR$ which is differentiable and satisfies the condition that $|p_0(\xb)| $ converges to zero  as $\|\xb\|_2 $ goes to infinity. Then for any differentiable function $f(\xb)$ with $\EE[|f(\xb)S(\xb)|]\vee\EE[ \|\nabla_x f(\xb)\|_2 ]< \infty$, it holds that,
	\begin{align*}
	\EE[f(\xb)S(\xb)]=\EE[\nabla_{\xb} f(\xb)],
	\end{align*}
where  $S(\xb)=-\nabla_{\xb} p_0(\xb)/p_0(\xb)$ is the score function with respect to $\xb$ defined in  Definition~\ref{defiscore}.
\end{lemma}
\textbf{Remark:} In the case of having matrix covariate, we are able to achieve the same conclusion by simply replacing $\xb\in\RR^{p}$ by $\Xb\in\RR^{d\times d}$ in Lemma \ref{genstein} with the  definition of matrix score function in \eqref{scoremat}.


In the sequel, we  introduce   the  single index models considered in this work. We first define sparse vector single index models as follows.

\begin{definition}(Sparse Vector SIM)\label{vectorSIM}
	We assume the  response $Y\in \RR$ is generated from model
	\begin{align}\label{vecSIM}
	Y=f(\langle \mathbf{x},\beta^{*}\rangle)+\epsilon,
	\end{align}
	with unknown link $f: \RR\rightarrow \RR $, $p$-dimensional covariate $\xb$ as well as signal $\beta^{*}$ which is the parameter of interest. 
	Here, we let $\epsilon\in \RR$ be  an exogenous random noise with mean zero. 
	In addition, if not particularly indicated, we assume the  entries of $\xb$ are i.i.d. random variables with a known univariate  density $p_0(x).$
	As for the  underlying true signal $\beta^{*},$  it is assumed to be $s$-sparse with $s\ll p.$ Note that the length of $\beta^{*}$ can be absorbed by the  unknown link $f$, we then let $\|\beta^{*}\|_2=1$ for model identifiability. 
\end{definition}
By the definition of sparse vector SIM, we notice that many well-known models are included in this category, such as linear regression $y_i=\mathbf{x}_i^\top\beta^{*}+\epsilon$, phase retrieval $y_i=(\mathbf{x}_i^\top\beta^{*})^2+\epsilon$,  as well as one-bit compressed sensing $y=\sign(\xb_i^\top \beta^{*})+\epsilon$. 


Finally, we define the  low rank   matrix   SIM as follows.

\begin{definition}(Symmetric Low Rank Matrix SIM)\label{matSIM}
	For the low rank matrix SIM, we assume the  response $Y\in \RR$ is generated from
	\begin{align}\label{matrixSIM}
	Y=f(\langle\mathbf{X},\beta^{*}\rangle)+\epsilon,
	\end{align}
	in which $\beta^{*}\in \mathbb{R}^{d\times d}$ is a low rank symmetric matrix with rank $r\ll d$ and the link function $f$ is unknown. For the  covariate $\Xb\in \mathbb{R}^{d\times d}$, we assume the  entries of $\Xb$  are  i.i.d.  with a  known density $p_0(x)$.  Besides, since   $\|\beta^{*}\|_{F}$ can  be absorbed in the  unknown link function $f$, we further assume  $\|\beta^{*}\|_{F}=1 $ for model identifiability. In addition, the noise term $\epsilon$ is also assumed   additive and mean zero.
\end{definition}

As we have discussed in the  introduction, \revise{almost all existing literature designs algorithms and studies the corresponding implicit regularization phenomenon  in linear models with sub-Gaussian data. 
The scope of this work is to leverage over-parameterization to design regularization-free algorithms and delineate the induced implicit regularization phenomenon for a more general class of statistics models with possibly heavy-tailed  data.}
Specifically, in \S\ref{sectionvec} and \S\ref{sectionmat}, we design algorithms and capture the implicit regularization induced by the gradient descent algorithm for over-parameterized vector and matrix SIMs, respectively.



\section{Main Results for  Over-Parameterized Vector SIM} \label{sectionvec}
Leveraging  our conclusion from Lemma \ref{genstein} as well as our definition of sparse vector SIM in Definitions \ref{vectorSIM}, we have
\begin{align*}
\EE[Y\cdot S(\xb)]=\EE[f(\langle \xb,\beta^{*} \rangle)\cdot S(\xb)]=\EE[f'(\langle \xb,\beta^{*} \rangle)]\cdot\beta^{*}:=\mu^{*}\beta^{*},
\end{align*}
which recovers our true signal $\beta^{*}$ up to scaling.
Here we define $\mu^* = \EE[f'(\langle \xb,\beta^{*} \rangle)]$, which is assumed nonzero throughout this work.
Hence,  $Y\cdot S(\xb)$ serves as an unbiased estimator of $\mu^* \beta^*$, and we can correctly identify the direction of $\beta^*$ by solving a population level optimization problem:
\begin{align*}
\min_{\beta}\, L(\beta):= \,\langle \beta, \beta \rangle-2\langle \beta, \EE[Y\cdot S(\xb)] \rangle.
\end{align*}
Since we only have access to finite data,  we replace $\EE[Y\cdot S(\xb)]$  by its sample version estimator $\frac{1}{n}\sum_{i=1}^n y_iS(\xb_i)$, and plug the sample-based estimator into the loss function.
In a  high dimensional SIM given in Definition \ref{vectorSIM}, where the  true signal $\beta^{*}$ is assumed to be    sparse,
various works \citep{plan2016generalized,plan2017high, yang2017high}  have shown that the
 $\ell_1$-regularized estimator $\hat \beta$ given by
 \begin{align}\label{reguvec}
\hat \beta \in \underset{\beta}{\argmin}\, L(\beta):= \, \langle \beta, \beta \rangle-2\Big\langle \beta, \frac{1}{n}\sum_{i=1}^n y_iS(\xb_i) \Big\rangle+\lambda \|\beta\|_1
\end{align}
attains the optimal  statistical    rate of convergence rate to $\mu^{*}\beta^{*}$.

In contrast, instead of imposing an $\ell_1$-norm regularization term,
we propose to obtain an estimator by minimizing the loss function $L$ directly, with $\beta$ re-parameterized using two vectors $\wb$ and $\vb$ in $\RR^p$.
Specifically,
  we write $\beta$ as $\beta=\wb\odot \wb-\vb\odot \vb$
  and thus equivalently write the loss function
  $L(\beta)$ as  $L(\wb,\vb)$, which is given by
 \begin{align}\label{loss1bk}
 L(\wb,\vb)=\langle \wb\odot \wb-\vb\odot \vb, \wb\odot \wb-\vb\odot \vb \rangle -2\Big\langle \wb\odot \wb-\vb\odot \vb, \frac{1}{n}\sum_{i=1}^{n}y_iS(\xb_i) \Big\rangle.
 \end{align}
Note that the way of writing  $\beta$ in terms of  $\wb$ and $\vb$ is not unique.
In particular, $\beta$ has $p$ degrees of freedom but we use $2p$ parameters to represent $\beta$. Thus,
by using $\wb$ and $\vb$ instead of $\beta$,   we employ over-parameterization in \eqref{loss1bk}.
\revise{We briefly describe our motivation on over-parameterizing $\beta$ by $\wb\odot \wb-\vb\odot\vb$. Suppose that $\beta$ is sparse, an explicit regularization is to use $\ell_1$-penalty.  Note that $\|\beta\|_1 = \argmin_{\bgamma \odot \bdelta = \beta } \{ \|\bgamma\|^2 + \|\bdelta\|^2\}/2$, where $\odot$ denotes the Hadamard (componentwise) product. Thus, an explicit regularization is to $\min_{\bgamma, \bdelta} \sum_{i=1}^n \{Y_i - f(\xb_i^T \bgamma \odot \bdelta) \}^2 + \lambda \{ \|\bgamma\|^2 + \|\bdelta\|^2\}$ for a penalty parameter $\lambda$, following the method in \cite{hoff2017lasso}.
To gain understanding on implicit regularization by over parametrization,  we let $\wb = (\bgamma + \bdelta)/2$ and $\vb = (\bgamma -\bdelta)/2$.  Then
$\beta =\bgamma \odot \bdelta =\wb\odot \wb-\vb\odot \vb$ with $2p$ new parameters $\wb$ and $\vb$ that over parameterize the problem.  This leads to the empirical loss $L(\wb,\vb)= \sum_{i=1}^n \{Y_i - f(\xb_i^T (\wb\odot \wb-\vb\odot \vb )) \}^2$. Following the neural network training, we drop the explicit penalty and run the gradient decent to minimize $L(\wb, \vb)$.} 


To be more specific, for the sparse SIM, we propose to construct an estimator of $\beta^*$ by applying gradient descent to $L$ in \eqref{loss1bk} with respect to $\wb$ and $\vb$, without any explicit regularization.
Such an estimator, if achieves desired statistical accuracy, demonstrates the efficacy of implicit  regularization of gradient descent in over-parameterized  sparse SIM.
Specifically,
the gradient updates for the vector $(\wb^\top, \vb^\top)^\top$  for solving
\eqref{loss1bk} are given~by
\begin{align}\label{vectorup1}
\wb_{t+1}&=\wb_t-\eta\nabla_{\wb} L(\wb_t,\vb_t)=\wb_t-\eta\Big(\wb_t\odot \wb_t-\vb_t\odot \vb_t-\frac{1}{n}\sum_{i=1}^{n}S(\xb_i)y_i\Big)\odot \wb_t,\\
\vb_{t+1}&\,=\vb_t\,+\eta\nabla_{\vb}L(\wb_t,\vb_t)=\vb_t\,+\eta\Big(\wb_t\odot \wb_t-\vb_t\odot \vb_t-\frac{1}{n}\sum_{i=1}^{n}S(\xb_i)y_i\Big)\odot \vb_t.\label{vectorup2}
\end{align}
Here $\eta >0 $ is the stepsize.
By the parameterization of $\beta$, $\{ \wb_{t}, \vb_t\}_{t\geq 0}$ leads to a sequence of estimators $ \{ \beta_t\}_{t\geq 0}$ given by
\begin{align}
	\beta_{t+1}&=\wb_{t+1}\odot \wb_{t+1}-\vb_{t+1}\odot\vb_{t+1}.\label{vectorup3}
\end{align}

Meanwhile, \revise{in terms of chooisng initial values,} since the zero vector  is a stationary point of the algorithm,
we cannot set the initial values of $\wb$ and $\vb$ to the zero vector.
To utilize the   structure of $\beta^*$,
ideally we would like to initialize  $\wb$ and $\vb$ such that they share the same sparsity pattern as $\beta^*$.
That is,
we would like to set the entries in the support of $\beta^*$ to nonzero values,  and set those outside of the support to zero.
  However, such an initialization scheme is  infeasible since  the support of $\beta^{*}$ is unknown.
	Instead, we initialize $\wb_0$ and $\vb_0$ as $\wb_0=\vb_0=\alpha\cdot \mathbf{1}_{p\times 1}$, where $\alpha > 0 $ is a small constant and $\mathbf{1}_{p\times 1}$ is an all-one vector in $\RR^p$.
	By setting $\wb_0 = \vb_0$, we equivalently set $\beta_0 $ to the zero vector.
	And more importantly, such a construction provides a good compromise:  zero components get nearly zero initializations, which are the majority under the sparsity assumption, and nonzero components get nonzero initializations.
	Even though we initialize every component at the same value, the nonzero components move quickly to their stationary component, while zero components remain small.  This is how over-parameterization differentiate active components from inactive components. 
	We illustrate this by a  simulation experiment.

\vspace{5pt}
{\noindent{\bf A simulation study.}
	In this simulation, we fix sample size $n=1000$, dimension $p=2000$, number  of non-zero entries $s=5$. Let $S:=\{i:\left|\beta_i^{*}\right|>0   \}$.  The responses $\{y_i\}_{i=1}^{n}$ are generated from $y_i=f(\langle \xb,\beta^{*}\rangle)+\epsilon_i,\, i\in[n]$ with link functions $f_1(x)=x$ (linear regression) and $f_2(x)=\sin(x).$ Here we assume $\beta^{*}$ is  $s$-sparse with $\beta_i=1/\sqrt{s},i\in S$,  and $\{\xb_i\}_{i=1}^{n}$ are standard Gaussian random vectors. We  over-parameterize  $\beta$ as $\wb\odot \wb-\vb\odot \vb$ and set  $\wb_0=\vb_0=10^{-5}\cdot \mathbf{1}_{p\times 1}.$
Then we update $\wb$, $\vb$ and $\beta$ regarding equations \eqref{vectorup1}, \eqref{vectorup2},  and \eqref{vectorup3} with stepsize $\eta=0.01$. The evolution of the distance between our  unnormalized iterates $\beta_t$ and $\mu^{*}\beta^{*}$, trajectories of $\beta_{j,t}$ for $j \in S$ and $\max_{j \in S^c} |\beta_{j,t}|$ are depicted in Figures~\ref{examfig1} and \ref{examfig2}.

	 \begin{figure}[htpb]
		\centering
		\begin{tabular}{ccc}
			\hskip-30pt\includegraphics[width=0.33\textwidth]{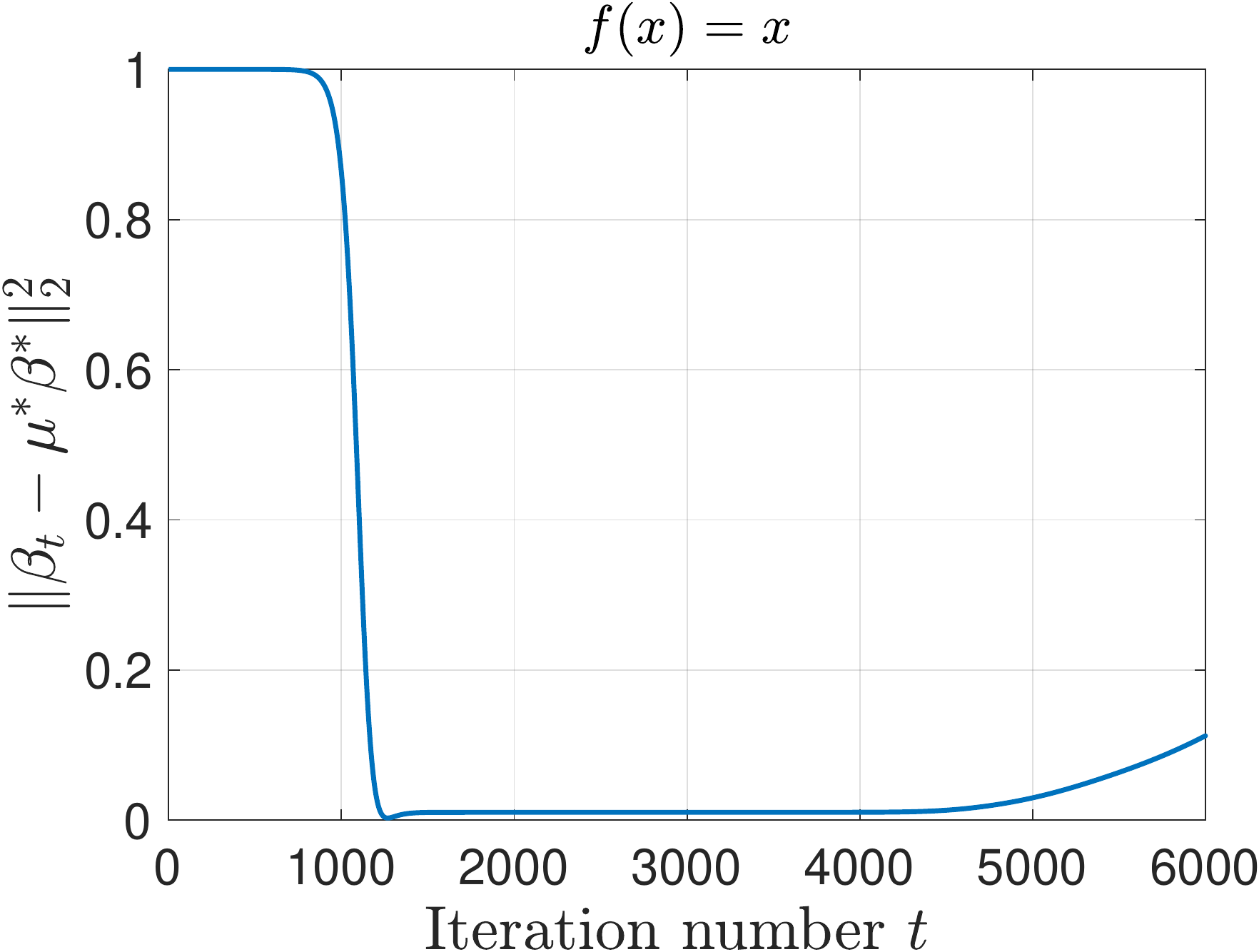}
			&
			\hskip-6pt\includegraphics[width=0.33\textwidth]{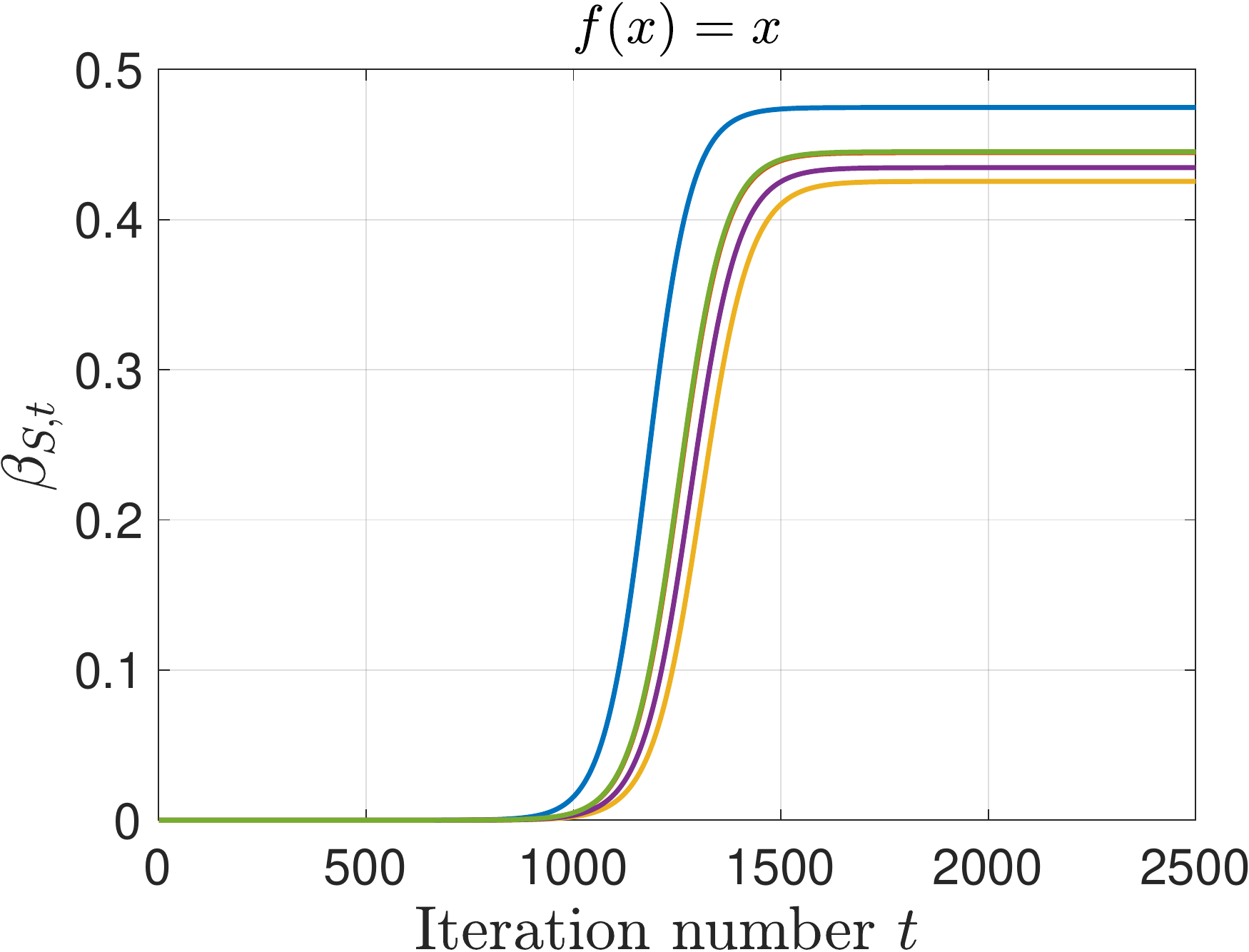}
			&
			\hskip-5pt\includegraphics[width=0.33\textwidth]{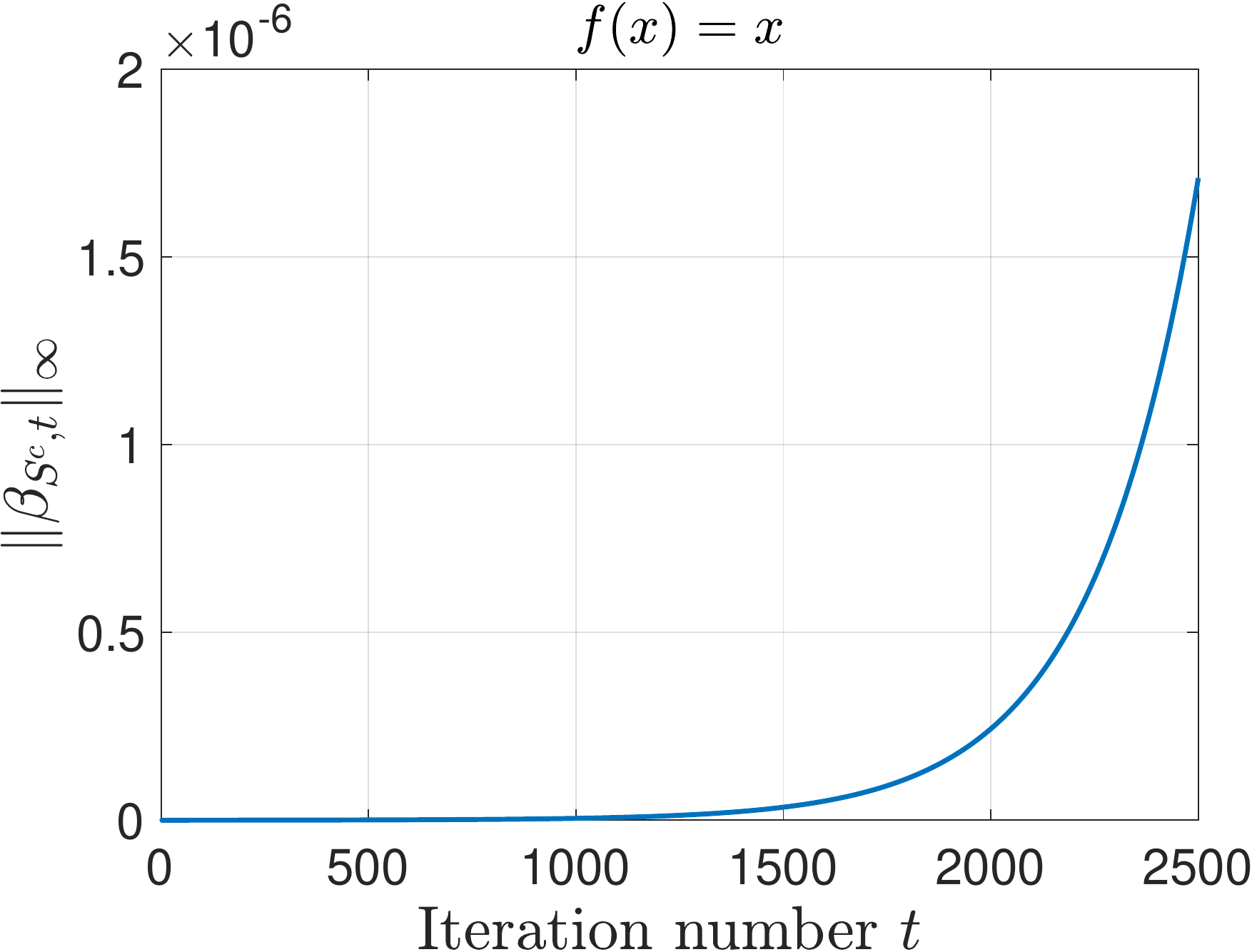} \\
				  (a) & (b) & (c)
		\end{tabular}\\
	\caption{With link function $f(x)=x$, (a) characterizes the evolution of distance $\|\beta_t-\mu^{*} \beta^{*}\|^2_2$ against iteration number $t$; (b) depicts the trajectories $\beta_{j, t}$ ($j \in S$) for five nonzero components, and (c) presents the trajectory $\max_{j \in S^c}|\beta_{j,t}|$.}
	\label{examfig1}
	\end{figure}
	\begin{figure}[htpb]
		\centering
		\begin{tabular}{ccc}
			\hskip-30pt\includegraphics[width=0.33\textwidth]{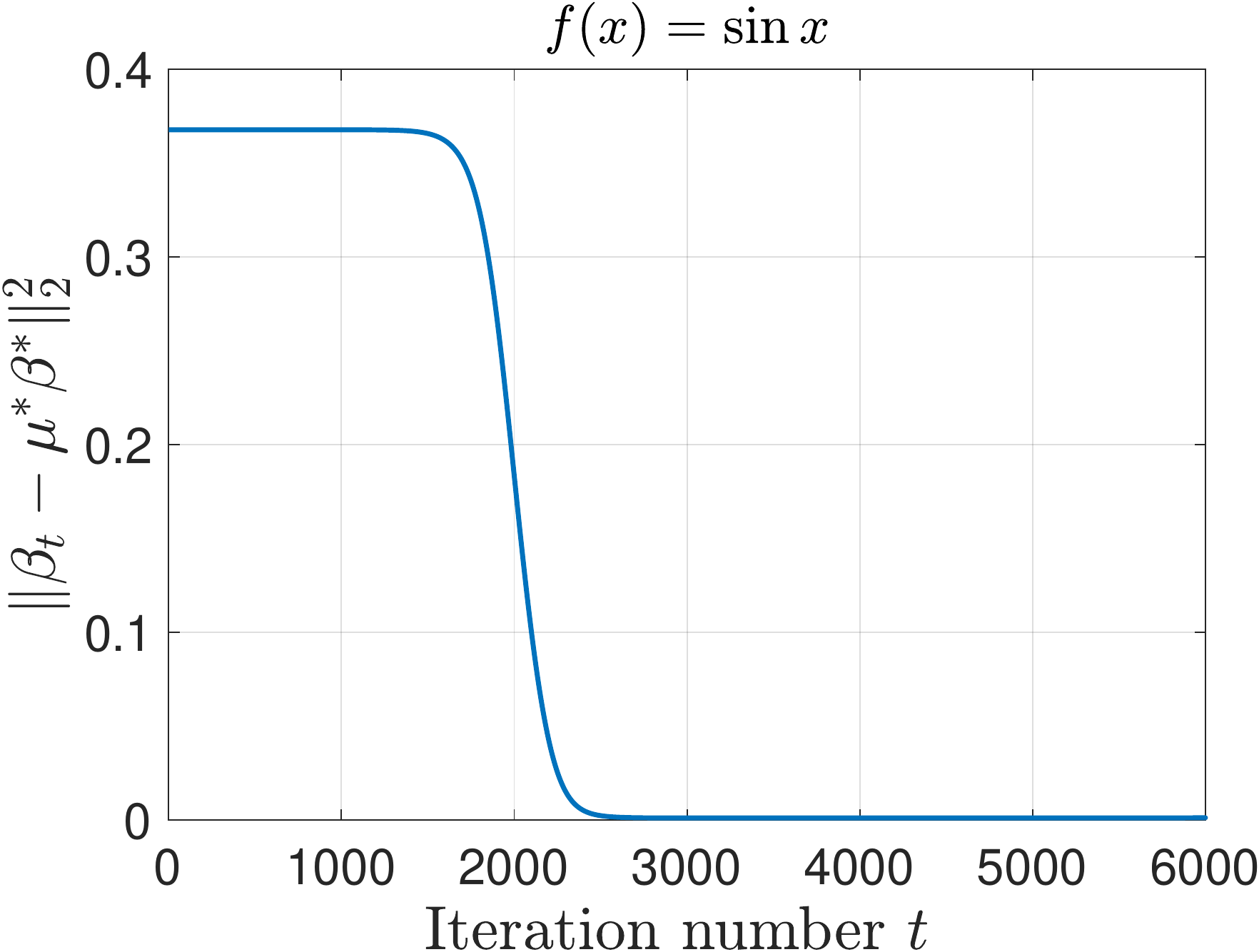}
			&
			\hskip-6pt\includegraphics[width=0.33\textwidth]{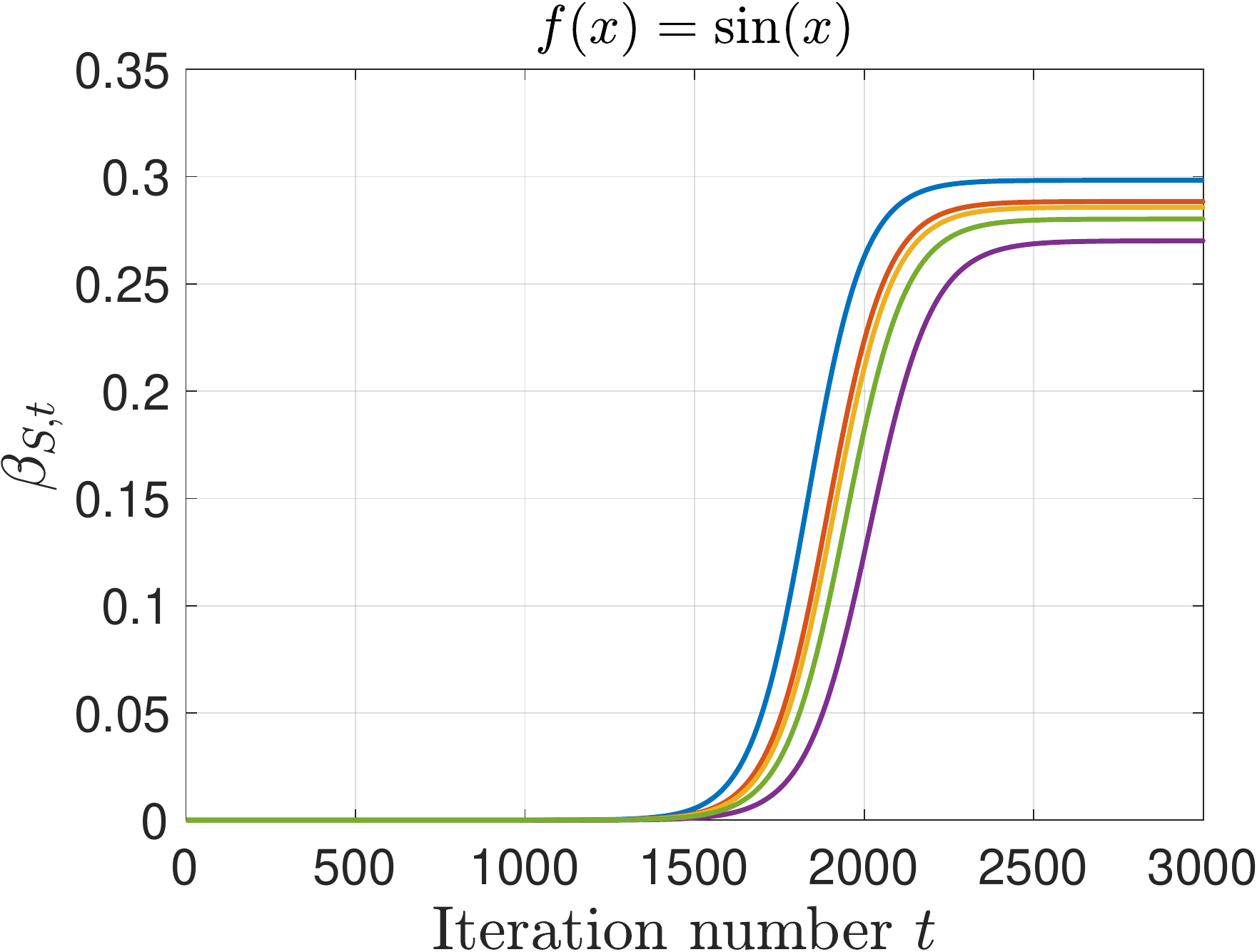}
			&
			\hskip-5pt\includegraphics[width=0.33\textwidth]{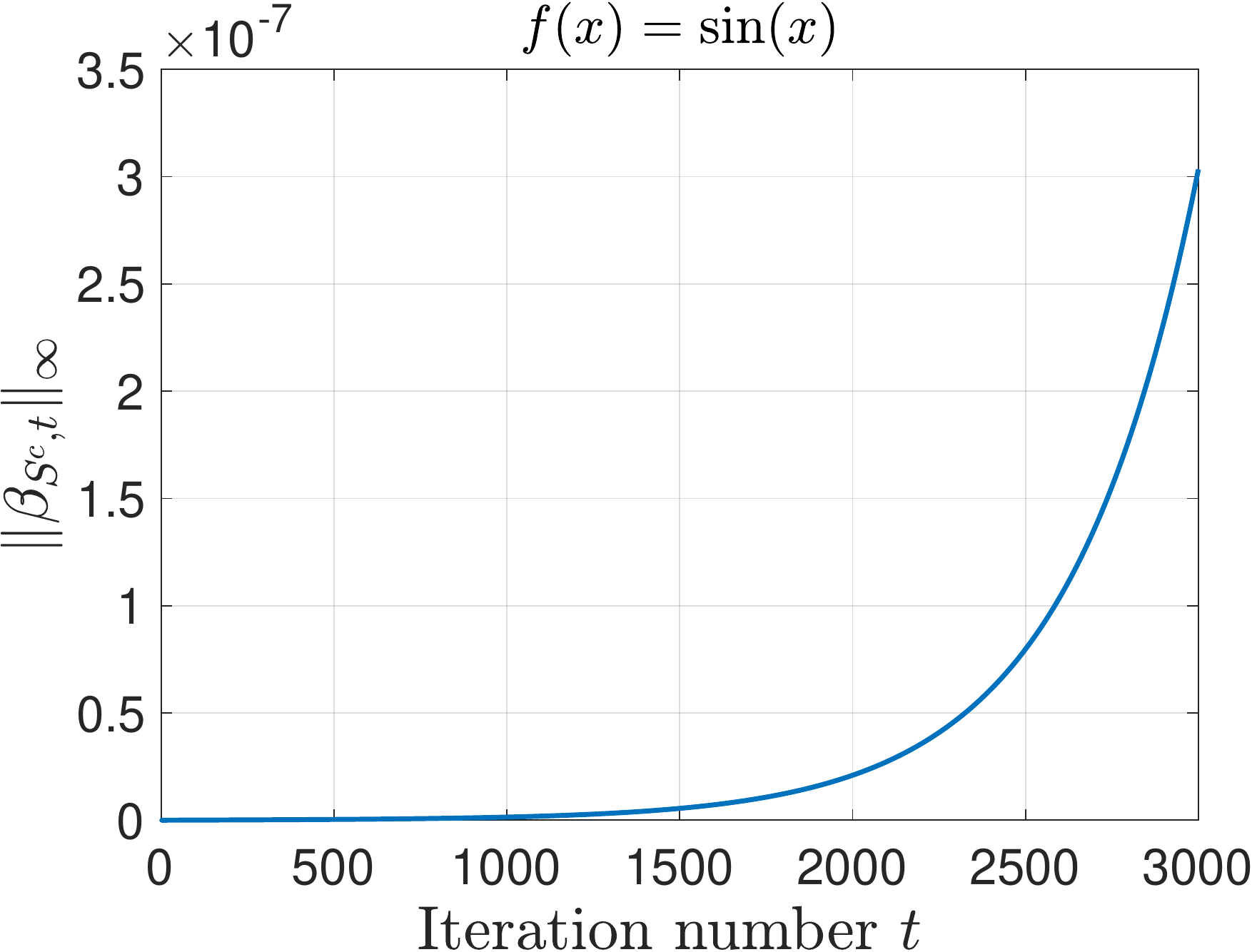} \\
		(a)& (b) & (c)
		\end{tabular}\\
		\caption{With link function $f(x)=\sin(x) $,  similar to Figure \ref{examfig1}, here (a) characterizes the evolution of distance $\|\beta_t-\mu^{*} \beta^{*}\|^2_2$ against iteration number $t$; (b) depicts the trajectories $\beta_{j, t}$ ($j \in S$) for five nonzero components, and (c) presents the trajectory $\max_{j \in S^c}|\beta_{j,t}|$.}
	\label{examfig2}
	\end{figure}
	
	From the simulation results given in Figure \ref{examfig1}-(a) and Figure \ref{examfig2}-(a), we notice that there exists a time interval,  where we can nearly recover $\mu^{*}\beta^{*}$. From plots (b) in Figures \ref{examfig1} and \ref{examfig2}, we can see with over-parameterization, five nonzero components all increase rapidly and converge quickly to their stationary points.   Meanwhile, the maximum estimation error for inactive component, represented by $\|\beta_{S^{c},t}\|_{\infty}$, still remains small, as shown in Figure \ref{examfig1}-(c) and Figure \ref{examfig2}-(c). In other words,
running gradient descent with respect to over-parameterized parameters  helps us distinguish non-zero components from zero components, while applying  gradient descent to  the ordinary loss can not.

It is worth noting that, with over-parameterization, there are $\Omega(2^p)$ stationary points of $L$ satisfying $\nabla_{\wb} L(\wb, \vb)  = \nabla_{\vb } L(\wb, \vb) = {\bf 0}_{p\times 1}$, where ${\bf 0}_{p\times 1}$ is the zero vector.
To see this,
for any subset $ I \subseteq [p]$,  we define vectors $\overbar\wb$ and $\overbar \vb$ as follows.
For any $j \notin I $, we set the $j$-th entries of $\overbar \wb $ and $\overbar \vb$ to zero.
Meanwhile, for any $j \in I $,  we choose $\overbar \wb_j$ and $\bar \vb_j$ such that  $
\overbar \wb_j^2 - \overbar \vb_j^2 = n^{-1}  \sum_{i=1}^n S (\xb_i) _j y_i,
$
where $\overbar \wb_j$, $\bar \vb_j$, and $S (\xb_i) _j$ are the $j$-th entries of $\overbar \wb$, $\overbar \vb$, and $S(\xb_i)$, respectively.
By direct computation, it can be shown that $(\overbar \wb, \overbar \vb)$ is a stationary point of $L$, and thus there are at least $2^p$ stationary points.
However, our numerical results demonstrate that not all of these stationary points are  likely to be found by the gradient descent algorithm ---  gradient descent favors the stationary points that correctly recover  $\mu^* \beta^*$.
Such an intriguing observation captures the implicit regularization induced by the optimization algorithm and over-parameterization.

\subsection{Gaussian Design}\label{gaussdesign}
In this subsection, we  discuss over-parameterized SIM with Gaussian covariates.
In this subsection, we assume  the  distribution  of  $\xb$ in \eqref{vecSIM} is
$  N(\mu, \Sigma)$,
where both $\mu$ and $\Sigma$ are assumed known.
Moreover, only in this subsection, we slightly modify  the identifiability condition in Definition \ref{vectorSIM} from assuming $\|\beta^{*}\|_2=1$ to $\| \Sigma^{1/2}\beta^{*} \|_2  = 1$.

\subsubsection{ Theoretical Results for Gaussian Covariates }\label{algstdgauss}

We first introduce an structural assumption on the SIM.

\begin{assumption} \label{ass1}
	{Assume that $\mu^* = \EE [ f'( \la \xb, \beta^* \ra )] \neq 0$ is a constant and the following} two conditions hold.
\begin{itemize}
	\item [(a).]  Covariance matrix $\Sigma$ is positive-definite and has bounded spectral norm. To be more specific, there exist constants $C_{\min}$ and $C_{\max}$ such that $C_{\min}\II_{p\times p}\preccurlyeq \Sigma\preccurlyeq C_{\max}\II_{p\times p}$ holds, where $\II_{p\times p}$ is the identity matrix.
 \item [(b).] Both $\{f(\langle\xb_i,\beta^{*} \rangle)\}_{i=1}^{n}$ and $\{\epsilon_i\}_{i=1}^{n}$  are i.i.d. sub-Gaussian random variables, with the sub-Gaussian norms  denoted by $\|f\|_{\psi_2}=\cO(1)$ and $\sigma=\cO(1)$ respectively.
 Here we let $\|f\|_{\psi_2}$ denote the sub-Gaussian norm of $f(\langle\xb_i,\beta^{*} \rangle)$.
In addition, we further assume that  $|\mu^{*}|/\|f\|_{\psi_2}=\Theta(1),|\mu^{*}|/\sigma=\Omega(1)$.
\end{itemize}
\end{assumption}
The score function for the Gaussian distribution $ N(\mu,\Sigma)$  is $S(\xb)=\Sigma^{-1}(\xb-\mu)$ and Assumption \ref{ass1}-(a) makes the Gaussian distributed covariates non-degenerate.     Assumption \ref{ass1}-(b)  enables the the empirical  estimator $n^{-1} \sum_{i=1}^{n}y_iS(\xb_i)$ to concentrate to its  expectation   $\mu^{*}\beta^{*}$, and  also sets a lower bound to the signal noise ratio $|\mu^{*}|/\sigma$.  Note that this assumption is quite standard and easy to be satisfied by a broad class of models as long as there exists a lower bound on the signal noise ratio, which include models with link functions $f(x)=x,\sin x, \tanh(x)$, and etc.
In addition, in \S\ref{secgenvec}, the  assumption that both $ f(\langle x, \beta ^*\rangle ) $ and the noise $\epsilon$ are sub-Gaussian random variables   will be further   relaxed to simply
assuming they have bounded finite moments with perhaps
heavy-tailed distributions.

We   present the details of the proposed method   for the Gaussian case  in Algorithm \ref{alg1}.
In the following, we
  present the  statistical rates of convergence for the estimator constructed by Algorithm~\ref{alg1}.  Let us divide the support set $ S=\{i:\left|\beta_i^{*}\right|>0   \}$
 into  $S_0=\{i: |\beta_i|\ge C_s\sqrt{\log p/n} \}$ and  $S_1=\{i: 0<|\beta^{*}|< C_s\sqrt{ \log p/n }  \}$, which correspond to the sets of strong and weak signals, respectively.
 Here $C_s$ is an absolute constant.
 We let $s_0$ and $s_1$ be the cardinalities of $S_0$ and $S_1$, respectively. In addition, we let $s_m = \min_{i \in S_0} |\beta^{*}_i|$ be the smallest value  of strong signals.


\begin{algorithm}[htpb]
	\KwData{Training covariates $\{\xb_{i}\}_{i=1}^{n},$  response variables $\{y_i\}_{i=1}^{n}$, initial value $\alpha$, step size $\eta$;}
	Initialize variables $\wb_0=\alpha\cdot\mathbf{1}_{p\times 1}$, $\vb_0=\alpha\cdot\mathbf{1}_{p\times 1}$ and set iteration number $t=0$;\\
	\While{$t<T_1$}{
		$\mathbf{w}_{t+1}=\mathbf{w}_{t}-\eta\big[\mathbf{w}_t\odot \mathbf{w}_t-\mathbf{v}_t\odot \mathbf{v}_t-\frac{1}{n}\sum_{i=1}^{n}\Sigma^{-1}(\xb_i-\mu)y_i\big]\odot \mathbf{w}_t$;\\
		$\, \mathbf{{v}}_{t+1}=\mathbf{v}_{t}\,+\eta\big[\mathbf{w}_t\odot \mathbf{w}_t-\mathbf{v}_t\odot \mathbf{v}_t-\frac{1}{n}\sum_{i=1}^{n}\Sigma^{-1}(\xb_i-\mu)y_i\big]\odot \mathbf{v}_t$; \\
		$\beta_{t+1}=\mathbf{w}_t\odot \mathbf{w}_t-\mathbf{v}_t\odot \mathbf{v}_t;$\\
		$\, t=t+1$;\\
	}
	\KwResult{ 
		Output the final estimate $\widehat \beta^{*}=\beta_{T_1}$.
	}\label{alg1}
	\caption{ Algorithm for  Vector SIM with Gaussian Design}
\end{algorithm}

\begin{theorem}\label{thmvec1}
	Apart from Assumption \ref{ass1}, if we further let our initial value $\alpha$ satisfy $0<\alpha\le {M_0^2}/{p}$ and set stepsize $\eta$ as  $0<\eta \le {1}/ ( 12(|\mu^{*}|+M_0) ) $ in   Algorithm \ref{alg1} with $M_0$ being a constant proportional to $\max\{\|f\|_{\psi_2},\sigma\}$, 
	there   exist absolute constants $a_1,a_2>0$ such that,
	with   probability at least $1-2p^{-1}-2n^{-2} $,
	we have
	\begin{align*}
		\left\|\beta_{T_1}-\mu^{*}\beta^{*}\right\|^2_2\lesssim \frac{s_0\log n}{n}+\frac{s_1\log p}{n},
		\end{align*}
	for all   $T_1\in[a_1{\log(1/\alpha)}/ ( \eta(|\mu^{*}|s_m-M_0\sqrt{\log p/n}) ) ,a_2\log(1/\alpha)\sqrt{n/\log p}/(\eta M_0)]$.
  Meanwhile,   the statistical  rate
 of convergence
  for the normalized   iterates are given by
	\begin{align*}
	\left\|\frac{\beta_{T_1}}{\left\|\Sigma^{1/2}\beta_{T_1}\right\|_2}-\frac{\mu^{*}\beta^{*}}{\left|\mu^{*}\right|}\right\|^2_2 & \lesssim \frac{s_0\log n}{n}+\frac{s_1\log p}{n}. 
	\end{align*}
\end{theorem}
\begin{theorem}\label{signconsist}(Variable Selection Consistency) Under the setting of Theorem \ref{thmvec1},
		for all   $$T_1\in[a_1{\log(1/\alpha)}/ ( \eta(|\mu^{*}|s_m-M_0\sqrt{\log p/n}) ) ,a_2\log(1/\alpha)\sqrt{n/\log p}/(\eta M_0)],$$ we let $[\tilde{\beta}_{T_1}]_i=[\beta_{T_1}]_i\cdot \mathbb{I}_{|[\beta_{T_1}]_i|\ge \lambda}$, for all $i\in [p]$. Then, with   probability at least $1-2p^{-1}-2n^{-2},$ for all $\lambda\in [\alpha,(C_s|\mu^{*}|-2M_0)\sqrt{\log p/n}]$, we have supp$(\tilde{\beta}_{T_1})\subset \textrm{supp}(\beta^{*})$. Moreover, when there only exists strong signals in $S_0$. We further have supp$(\tilde{\beta}_{T_1})= \textrm{supp}(\beta^{*})$ and $\textrm{sign}(\tilde{\beta}_{T_1})=\textrm{sign}(\beta^{*})$.
	\end{theorem}

 Theorem \ref{thmvec1} shows that if we just have strong signals, then with high probability, for any  $T_1 \in [a_1{\log(1/\alpha)}/ (\eta(|\mu^{*}|s_m-M_0\sqrt{\log p/n}) ) ,a_2\log(1/\alpha)\sqrt{n/\log p}/(\eta M_0)]$, we get  the oracle statistical rate  $\cO  ( \sqrt{s\log n / n})$  in terms of the $\ell_2$-norm, which is  independent of the ambient dimension $p$.
Besides, when  $\beta^*$
also consists of weak signals, we achieve $\cO(\sqrt{ s \log p / n})$  statistical rate in terms of the $\ell_2$-norm, where $s $ is the sparsity  of $\beta^*$.
Such a statistical rate  matches the minimax rate  of sparse linear regression \citep{raskutti2011minimax} and is  thus minimax optimal.
Notice that  the oracle rate is  achievable via  explicit regularization using folded concave penalties  \citep{fan2014strong} such as
 SCAD  \citep{fan2001variable} and  MCP  \citep{zhang2010nearly}. Thus, Theorem  \ref{thmvec1} shows that, with over-parameterization, the implicit regularization of gradient descent has the same effect as adding a folded concave penalty function to the loss function in \eqref{loss1bk} explicitly.

Furthermore, comparing our work to   \cite{plan2016generalized,plan2017high}, which study  high dimensional SIM with    $\ell_1$-regularization,
thanks to the implicit regularization phenomenon, we  avoid bias brought by the $\ell_1$-penalty and attain the  oracle statistical  rate.  \revise{ Moreover, our another advantage over  regularized methods is shown in Theorem \ref{signconsist}. It shows that by properly truncating $\beta_{T_1}$ when $T_1$ falls in the optimal time interval, we are able to recover the support of $\beta^{*}$ with high probability. Comparing to existing literatures on support recovery via using explicit regularization on single index model \citep{neykov2016l1}, our method offers a wider range for choosing tuning parameter $\lambda$ with a known left boundary $\alpha$, instead of only using $\lambda=\Theta(\sqrt{\log p/n})$. This efficiently reduces false discovery rate, see \S\ref{comparison} for more details. Last but not least, as we only need to run gradient descent, comparing to regularized methods, it is easier to parallel our algorithm since the gradient information is able to be efficiently transferred among different machines.  The use of implicit regularization allows our methodology to be generalized to large-scale problems easily \citep{MMRHA2017,Richards2020,Richards2020D}. The detailed discussions are given in \S\ref{scalable}.}

Theorem \ref{thmvec1} \revise{and Theorem \ref{signconsist}} generalizes the results in
  \cite{zhao2019implicit} and \cite{Vaskevicius2019} for the  linear model to high-dimensional SIMs.
  In addition,
	to satisfy the RIP condition, their sample complexity is  at least $\mathcal{O}(s^2\log p)$ if their covariate $\xb$ follows the Gaussian distribution.
	Whereas, by using the loss function in \eqref{loss1bk} motivated by the Stein's identity \citep{stein1972bound, stein2004use},
	the RIP condition is unnecessary in our analysis.
	Instead, our theory only requires that $n^{-1} \sum _{i=1}^n S(\xb_i) \cdot y_i$ concentrates at a fast rate.
	As a result,
	 our sample complexity is  $\mathcal{O}(s\log p)$ for    $\ell_2$-norm consistency,
	 which is better than $\mathcal{O}(s^2\log p)$.

 The ideas of proof behind Theorem \ref{thmvec1} \revise{and Theorem \ref{signconsist}} are as follows. First, we are able to control the strengths of error component, denoted by $\|\beta_t\odot \mathbf{1}_{S^{c}}\|_{\infty}$,  at the same order with the square root of their initial values until $\mathcal{O}(\log(1/\alpha)\cdot \sqrt{n/\log p}/(\eta M_0))$ steps. \revise{This gives us the right boundary of the stopping time $T_1.$}
Meanwhile, every entry of strong signal part $\beta_t\odot \mathbf{1}_{S_0}$ grows at exponential rates to $\epsilon=\cO(\sqrt{\log n/n})$ accuracy around $\mu^{*}\beta^{*}\odot \mathbf{1}_{S_0}$ within $\mathcal{O}({\log(1/\alpha)}/ ( \eta(|\mu^{*}|s_m-M_0\sqrt{\log p/n})) )$ steps, \revise{ which offers us the left boundary of the stopping time $T_1$}. Finally, we prove for weak signals, their strengths will not exceed $\cO(\sqrt{\log p/n})$ for all steps as long as we properly choose the stepsize. \revise{Thus, by letting the stopping time $T_1$ be in the interval given in Theorem \ref{thmvec1}, we obtain converged signal component and well controlled  error component.}  The final statistical rates are obtained by combining
the results on the active and inactive components together. \revise{Moreover, the conclusion of Theorem \ref{signconsist} holds by truncating the $\beta_t$ properly, since we are able to control the error component of $\beta_t$ uniformly as mentioned above.}
See Appendix \S\ref{gensig} for the detail. \revise{As shown in the proof, we observe that with small initialization and over-parameterized loss function, the signal component converges rapidly to the true signal, while the  the error component grows in a relatively slow pace.
Thus, gradient descent rapidly isolates the signal components from the noise, and with a proper stopping time, finds a near-sparse solution with high statistical accuracy. Thus, with proper initialization, over-parameterization plays the role of an implicit regularization by favoring approximately sparse saddle points of the loss  function in \eqref{loss1bk}.}

Finally,  we remark that Theorem \ref{thmvec1} establishes optimal statistical rates for the estimator $\beta_{T_1}$, where $T_1$ is any stopping time that belongs to the interval given in Theorem \ref{thmvec1}.
However, in practice,
such an interval is infeasible to compute as it depends on unknown constants.
To make the proposed method practical,
in the following, we introduce a method for  selecting a proper stopping time $T_1$.



\subsubsection{Choosing the  Stopping Time $T_1$}\label{stoptime}
We split the  dataset into training data and testing data. We utilize the training data  to implement Algorithm \ref{alg1} and get the  estimator $\beta_t$ as well as the value of the training loss \eqref{loss1bk} at step $t$. We notice $\beta_t$ varies slowly inside the optimal time interval specified in  Theorem \ref{thmvec1}, so that the fluctuation of the  training loss \eqref{loss1bk} can be smaller than a threshold.
Based on that, we  choose $m$ testing points on the flatted curve of the training loss \eqref{loss1bk} and denote their corresponding number of iterations as $\{t_j\},j\in[m]$.
For each $j\in[m]$, we then reuse the training data and normalized estimator $\beta_{t_j}/\|\Sigma^{1/2}\beta_{t_j}\|_2,j\in[m]$ to fit the link function $f$. Let the obtained estimator be  $\hat f_j$. 
For the testing dataset, we perform out-of-sample prediction and get $m$ prediction losses: 
\begin{align*}
l_{j}=\frac{1}{n_{\textrm{test}}}\sum_{i=1}^{n_{\textrm{test}}} \big [Y_{i}-\hat{f}_j(\langle \xb_{i},\beta_{t_j}/\|\Sigma^{1/2}\beta_{t_j}\|_2\rangle)\bigr ]^2, \qquad \forall j\in[m].
\end{align*}
Next, we choose $T_1$  as $t_{j^*}$ where we define  $j^* = \argmin_{j \in [m] } l_{j}$.

We remark that each $\hat f_j$ can be obtained by any nonparametric regression methods.
To show case our method, in the following, we apply univariate kernel regression to obtain  each $\hat f_j$ and  establish its theoretical guarantee.


\subsubsection{Prediction Risk}\label{predriskvec}

We now consider estimating the nonparametric component and the prediction risk.
Suppose we are given an estimator $\hat \beta$ of $\beta$ and $n$ i.i.d.\,\,observations $\{ y_i, \xb_i\}_{i=1}^{n}$ of the model.
For simplicity of the technical analysis, we assume that $\hat \beta $  is   independent   of  $\{y_i,\mathbf{x}_i\}_{i=1}^{n}$, which can be achieved by data-splitting.
Moreover, we assume that $\hat \beta$ is an estimator of $\beta^*$ such that
\#\label{eq:assumption}
\big \| \hat \beta - \beta ^* \big \|_2 =o({n^{-1/3}}), \quad \big \| \Sigma^{1/2}\hat \beta \big \|_2 = 1,\quad \textrm{and} \quad \big \| \Sigma^{1/2} \beta^{*} \big \|_2 = 1.
\#
 Our goal is to construct an estimate the regression function $f(\la \cdot, \beta^*\ra)$ based on $\hat \beta$ and $\{y_i,\mathbf{x}_i\}_{i=1}^{n}$.

Note that, when $\beta^*$ is known, we can directly estimate $f$ based on $y_i$ and  $Z_i^*: = \xb_i^\top \beta^*, i\in[n]$ via standard non-parametric regression.
When $\hat \beta$ is accurate, a direct idea is to
replace $Z_i^*$ by $Z_i:=\xb_i^\top\hat\beta $ and follow the similar route. For a new observation $\xb$, we define $Z$ as $Z:=\xb^\top\hat\beta$ and $Z^{*}$ as $Z^{*}:=\xb^\top\beta^{*}$ respectively.

To predict $Y$, we estimate function $g(z)$ using kernel regression with data $\{ ( y_i, \mathbf{x}_i^\top \hat \beta )\}_{i=1}^{n}$. 
Specifically, we
let the function $K_h(u)$ be $K_h(u) := 1/h \cdot K(u/h)$, in which $K\colon \RR \rightarrow \RR$ is a kernel function with $K(u)=\II_{\{|u|\le 1\}}$  and $h$ is a bandwidth.
By the definitions of $Z^{*},Z,$ and $Z_i,i\in[n]$ given above, 
 the prediction function $\hat{g}(Z)$ is defined as
\begin{equation}  \label{predfuncthm}
\hat{g}(Z) =
\begin{cases}
 \frac{\sum_{i=1}^{n}y_iK_h(Z-Z_i)}{\sum_{i=1}^{n}K_h(Z-Z_i)}, & |Z-\mu^\top\hat\beta|\le R,  \\
  0,  &\text{otherwise},
\end{cases}
\end{equation}
where we follow the convention that $0/0=0$. In what follows, we consider the $\ell_2$-prediction risk of $\hat g$, which is given by
\$
\EE \left [  \left\{ \hat g\big ( \langle\xb, \hat \beta \rangle \big ) - f \big ( \langle \xb, \beta ^*\rangle \big ) \right\}^2\right ],
\$
where the expectation is taken with respect to $\xb$ and $\{\xb_i,y_i\}_{i=1}^{n}$. Before proceeding to the theoretical guarantees, we make the following assumption on the regularity of $f$.

\begin{assumption}\label{assume:nonparam}
 There exists an $\alpha_1>0$ and a constant $C>0$ such that $|f(x)|,\,|f'(x)|\le C+|x|^{\alpha_1}.$
\end{assumption}
For the rationality of the Assumption \ref{assume:nonparam}, we note that  the constraint on $f'(x)$ and $f(x)$ given above is weaker than assuming $f'(x)$ and $f(x)$ are bounded functions directly. 
Next, we present   Theorem \ref{thmpred} which characterizes the convergence rate of mean integrated error of our prediction function $\hat{g}(Z).$ 

\begin{theorem}\label{thmpred}
	If we set $R=2\sqrt{\log(n)}$ and $h\asymp n^{-1/3}$ in \eqref{predfuncthm}, under Assumption \ref{assume:nonparam}, the $\ell_2$-prediction risk of $\hat{g}$ defined in \eqref{predfuncthm} is given by
	\$
	\EE \left[  \left\{\hat g\big (\langle \xb, \hat \beta \big \rangle) - f \big (\langle \xb,\beta ^*   \rangle \big) \right\}^2\right]\lesssim\frac{\textrm{polylog}(n) }{n^{2/3}},
	\$
	where $\hat\beta$ is any vector that satisfies \eqref{eq:assumption} and $\textrm{polylog}(n)$ contains terms that are polynomials of $\log n$.
\end{theorem}

It is worth noting that the estimator $\hat\beta=\beta_{T_1}/\|\Sigma^{1/2}\beta_{T_1}\|_2$ constructed in Theorem \ref{thmvec1} with any $T_1$ belongs to the optimal time interval given in Theorem \ref{thmvec1} satisfy \eqref{eq:assumption}. Thus, under such regimes, Theorem \ref{thmpred} also holds.  The proof of Theorem \ref{thmpred} is given in  \S\ref{MSEproof}. Note that it is possible to refine the  analysis on the prediction risk for $f$ with higher order derivatives by utilizing higher order kernels (see \cite{nonparamtsyb2008} therein) this is not the key message of our paper.

\subsection{General Design}\label{secgenvec}
In this subsection, we extend our methodology to  the setting with covariates generated from a general distribution.
Following our discussions at the beginning of \S\ref{sectionvec}, ideally we aim at solving the loss function with over-parameterized variable given in \eqref{loss1bk}.
However, when the distribution of $\xb $ has density $p_0$, the score
$S(\xb)$
can be  heavy-tailed  such that
$\mathbb{E}[Y\cdot S(\mathbf{x})]$ and its empirical counterpart may not be  sufficiently close.

To remedy this issue, we modify the loss function in
  \eqref{loss1bk}
 by replacing  $y_i$ and $S(\xb_i)$ by their truncated (Winsorized) version $\widecheck{y_i}$ and $\widecheck{S}(\xb_i) $, respectively.
 Specifically, we  propose to apply
   gradient descent to  the following modified loss function  with respect to $\mathbf{u}$ and $\mathbf{v}$:  
\begin{align}\label{lossgenvec}
\min_{\mathbf{w},\mathbf{v}} L(\mathbf{w},\mathbf{v}):=\langle \mathbf{w}\odot\mathbf{w}-\mathbf{v}\odot\mathbf{v},\mathbf{w}\odot\mathbf{w}-\mathbf{v}\odot\mathbf{v} \rangle-\frac{2}{n}\sum_{i=1}^{n}\widecheck{y_i}\big\langle \mathbf{w}\odot \mathbf{w}-\mathbf{v}\odot\mathbf{v},\widecheck{S}(\mathbf{x}_i)  \big\rangle.
\end{align}
Let  $\widecheck{\mathbf{a}}\in \mathbb{R}^d$ denote  the truncated version of vector $\mathbf{a}\in \mathbb{R}^{d}$ based on a parameter $\tau$ \citep{fan2018robust}.
That is, its entries are given by $[\widecheck{\mathbf{a}}]_j=[\mathbf{a}]_j$ if $|\mathbf{a}_i|\le\tau$ and $\tau$ otherwise.
 Applying elementwise truncation to  $\{y_i\}_{i=1}^{n}$ and $\{S(\xb_i)\}_{i=1}^{n}$ in  \eqref{lossgenvec},
  we allow the  score $S(x)$ and the response $Y$ to both have heavy-tailed distributions.
 By choosing a proper threshold $\tau$, such a truncation step ensures
  $n^{-1} \sum_{i=1}^{n}\widecheck{y}_i\widecheck{S}(\xb_i)$ converge  to $\EE[Y\cdot S(\xb)]$ with a desired rate in $\ell_{\infty}$-norm.
Compared with Algorithm \ref{alg1}, here we only modify the definition of the loss function. Thus, we defer the
details of the proposed algorithm for this setting to
Algorithm \ref{alg5} in \S\ref{alggenvec}.

Before stating our main theorem,  we first present an assumption on the distributions of the covariate and the response variables.

\begin{assumption} \label{genass1}
Assume there exists a   constant $M$ such that
	\begin{align*}
	\mathbb{E}\left[Y^4 \right]\le M, \qquad \mathbb{E}\left[S(\mathbf{x})_j^4\right]\le M,\qquad \forall j\in[p].
	\end{align*}
	Here $S(\xb)_j$ is the $j$-th entry of $S(\xb)$.
	Moreover, recall that we denote  $\mu^* = \EE [ f' (\langle \xb, \beta^*\rangle )]$. We assume that $\mu^*$ is a nonzero constant such that
  $M/|\mu^{*}|=\Theta(1)$.
\end{assumption}

Assuming the fourth moments exist and are bounded is significant weaker  than the sub-Gaussian assumption. Moreover, such an  assumption is  prevalent in  robust statistics literature  \citep{fan2016shrinkage,fan2018elliptical, fan2019robustfactor}.
Now we are ready to introduce the
theoretical results for the setting with
general design.

  \begin{theorem}\label{thmvec2}
	Under  Assumption \ref{genass1}, we set the  thresholding  parameter $\tau=( M\cdot n /\log p)^{1/4}/2$, let  the initialization parameter  $\alpha$ satisfy $0<\alpha\le {M_g^2}/{p}$,   and set the  stepsize $\eta$ such that  $0<\eta \le {1}/ ( 12(|\mu^{*}|+M_g) ) $ in Algorithm \ref{alg5} given in \S\ref{alggenvec} where $M_g$ is a constant proportional to $M$. There exist absolute constants $a_3,\,a_4$, such that, with  probability at least  $1-2p^{-2}$,
	\begin{align*}
	\left\|\beta_{T_1}-\mu^{*}\beta^{*}\right\|^2_2\lesssim \frac{s\log p}{n}
	\end{align*}
	 holds for all  $T_1\in[a_3{\log(1/\alpha)}/ ( \eta(|\mu^{*}|s_m-M_g\sqrt{\log p/n})) ,a_4\log(1/\alpha)\sqrt{n/\log p}/(\eta M_g)]$. 
	  Here  $s$ is the cardinality of the support set $S$ and $s_m = \min_{i\in S_0} | \beta_j^*| $, where $S_0 = \{ j \in  i \colon | \beta_ i | \geq C_s \sqrt{\log p /n}\}$ is the  set of strong signals.
	In addition,    for the normalized iterates, we further have
	\begin{align*}
	\left\|\frac{\beta_{T_1}}{\|\beta_{T_1}\|_2}-\frac{\mu^{*}\beta^{*}}{|\mu^{*}|}\right\|^2_2\lesssim \frac{s\log p}{n},
	\end{align*}
with  probability at least  $1-2p^{-2}.$
\end{theorem}

Compared with Theorem \ref{thmvec1}
for the Gaussian design,
here we achieve the $\cO(\sqrt{ s \log p /n})$ statistical rate of convergence in terms of the $\ell_2$-norm.
These rates are the same of those achieved by adding an $\ell_1$-norm regularization explicitly \citep{plan2016generalized, plan2017high,yang2017high} and are   minimax optimal \citep{raskutti2011minimax}.
Moreover, we note that here $S(\xb)$ and $Y$ can be both heavy-tailed and our truncation procedure successfully tackles such a challenge without sacrificing the statistical rates.
Moreover, similar to the Gaussian case, here $C_s$ can be set as a sufficiently large absolute constant, and the statistical rates established in Theorem \ref{thmvec2} holds for all choices of  $C_s$. \revise{In addition, for heavy-tailed case, we also let $[\tilde{\beta}_{T_1}]_i=[\beta_{T_1}]_i\cdot \mathbb{I}_{|[\beta_{T_1}]_i|\ge \lambda}$, for all $i\in [p]$. Then for all $\lambda\in [\alpha,(C_s|\mu^{*}|-2M_g)\sqrt{\log p/n}],$ we obtain similar theoretical guarantees as in  Theorem \ref{signconsist}. }

\section{Main Results for  Over-Parametrized Low Rank SIM}\label{sectionmat}
In this section, we   present the results for  over-parameterized low rank matrix SIM
introduced in  Definition \ref{matSIM}
with both standard Gaussian and generally distributed covariates.
Similar to the results in \S\ref{sectionvec},
here we also focus on matrix SIM with first-order links, i.e., we assume that $\mu^* = \EE[f'(\langle \Xb,\beta^{*} \rangle)] \neq 0 $, where $\beta^*$ is a low rank matrix with rank $r$.
Note that we assume that the entries of covariate $\Xb \in \RR^{d\times d}$ are i.i.d. with a univariate  density $p_0$.
Also recall that we define the score function $S(\Xb) \in \RR^{d\times d}$ in \eqref{scoremat}.
Then, similar to the loss function in \eqref{loss1bk},
we consider the loss function
$$
L(\beta):= \, \langle \beta, \beta \rangle-2\Big\langle \beta, \frac{1}{n}\sum_{i=1}^n y_iS(\Xb_i) \Big\rangle,
$$
where $\beta \in \RR^{d\times d}$ is a symmetric matrix.
Hereafter, we rewrite  $\beta $ as $ \Wb \Wb^\top - \Vb \Vb^\top$,
where both $\Wb$ and $\Vb$ are matrices in $\RR^{d\times d}$. \revise{The intuitions of re-parameterizing $\beta=\Wb\Wb^\top-\Vb\Vb^\top$ are as follows. Any (low rank) symmetric matrix is able to be written as the difference of  two positive semidefinite matrices, namely $\Wb\Wb^\top-\Vb\Vb^\top$ with $\Wb,\Vb\in \RR^{d\times d}$. Re-parameterizing the symmetric matrix this way is a generalization of re-parameterizing its eigenvalues by the Hadamard products. Thus this can be regarded as an extension of  the re-parameterization mechanism from the vector case to the spectral domain.}
With such an  over-parameterization, we propose to estimate $\beta^*$ by applying  gradient descent to the loss function
\begin{align}\label{matoverpara}
	L(\Wb,\Vb):=\, \langle \Wb\Wb^\top-\Vb\Vb^\top,\Wb\Wb^\top-\Vb\Vb^\top\rangle-2\Big\langle\Wb\Wb^\top-\Vb\Vb^\top,\frac{1}{n}\sum_{i=1}^{n}y_iS(\Xb_i) \Big\rangle.
\end{align}
Since the rank of $\beta^{*}$ is unknown, we initialization
$\Wb_0$ and $\Vb_0$ as $\Wb_0=\Vb_0=\alpha \cdot\II_{d\times d}$ for a small $\alpha > 0$ and construct a sequence of iterates $\{ \Wb_t, \Vb_t, \beta_t\}_{t\geq 0 }$
via the gradient decent method as follows:
\begin{align}\label{gradupw}
	\mathbf{W}_{t+1}&=\mathbf{W}_{t}-\eta\Big(\mathbf{W}_t \mathbf{W}_t^\top-\mathbf{V}_t \mathbf{V}_t^\top-\frac{1}{2n}\sum_{i=1}^{n}S(\mathbf{X}_i)y_i-\frac{1}{2n}\sum_{i=1}^{n}S(\mathbf{X}_i)^\top y_i\Big) \mathbf{W}_t,\\
	\mathbf{V}_{t+1}&=\mathbf{V}_{t}\,+\,\eta\Big(\mathbf{W}_t \mathbf{W}_t^\top-\mathbf{V}_t \mathbf{V}_t^\top-\frac{1}{2n}\sum_{i=1}^{n}S(\mathbf{X}_i)y_i-\frac{1}{2n}\sum_{i=1}^{n}S(\mathbf{X}_i)^\top y_i\Big) \mathbf{V}_t,\label{gradupv} \\
	\beta_{t+1}&=\mathbf{W}_t \mathbf{W}_t^\top-\mathbf{V}_t\Vb_t^\top,\nonumber
	\end{align}
	where $\eta $ in \eqref{gradupw} and \eqref{gradupv}  is the stepsize.
Note that here the algorithm does not impose any explicit regularization. 
In the rest of this section, we show that
such a procedure yields an
estimator of the true parameter $  \beta^*$ with near-optimal statistical rates of convergence.

Similar to  the vector  case, for theoretical analysis, here we also divide eigenvalues of $\beta^{*}$ into different groups by their strengths.  We let $r_i^{*},i\in[d]$ be the $i$-th eigenvalue of $\beta^{*}$. The support set $R$ of the eigenvalues is defined as $R:=\{i:\left|r_i^{*}\right|>0   \}$, whose cardinality is    $r$.  We then divide the support set $ R$
	into  $R_0:=\{i: |r_i^{*}|\ge C_{ms}\sqrt{{d\log d}/{n}} \}$ and  $R_1:=\{i: 0<|r_i^{*}|< C_{ms}\sqrt{ {d\log d}/{n}  }  \}$, which correspond to collections  of  strong and weak signals with cardinality denoting by $r_0$ and $r_1$, respectively.
	Here $C_{ms} > 0$ is an absolute constant and we have
	$R = R_0 \cup R_1$.
	Moreover, we use $r_m$ to denote the minimum strong eigenvalue in magnitude, i.e. $r_m=\min_{i\in R_0}|r_i^{*}|$.

\subsection{Gaussian Design}\label{secgaussmat}
In this subsection, we focus on the model in   \eqref{matrixSIM} with the  entries of covariate $\Xb$ being  i.i.d.  $N(0,1)$ random variables.  In this case,
$S(\Xb_i) = \Xb_i$.  This leads to Algorithm \ref{alg2} given in \S\ref{gauss_mat_alg}, where we place $S(\Xb_i)$ by $\Xb_i$ in  \eqref{matoverpara}-\eqref{gradupv}. 


Similar to the case in \S\ref{gaussdesign}, here we also impose the following assumption for the function class of the low rank SIM. 

\begin{assumption} \label{ass1mat}
We assume that $\mu ^* = \EE [ f' (\langle  \Xb, \beta^* \rangle ) ]$ is a nonzero constant. Moreover, we assume that
 both $\{f(\langle\Xb_i,\beta^{*} \rangle)\}_{i=1}^{n}$ and $\{\epsilon_i\}_{i=1}^{n}$  are i.i.d. sub-Gaussian random variables, with sub-Gaussian norm denoted by $\|f\|_{\psi_2}=\cO(1)$ and $\sigma=\cO(1)$ respectively.
 Here we let $\| f\|_{\psi_2} $ denote the sub-Gaussian norm of $f(\langle \Xb, \beta^*\rangle)$.
  In addition, we further assume $|\mu^{*}|/\|f\|_{\psi_2}=\Theta(1),|\mu^{*}|/\sigma=\Omega(1)$.
\end{assumption}

The following theorem establishes the statistical rates of convergence for the estimator constructed by Algorithm \ref{alg2}.

  \begin{theorem} \label{thmmat1}
	We set $0<\alpha\le M_{m}^2/d$ and stepsize $0<\eta\le1/[12(|\mu^{*}|+M_m)]$ in  Algorithm \ref{alg2}, where $M_m$ is a constant proportional to $\max\{\|f\|_{\psi_2},\sigma\}$.
	Under Assumption \ref{ass1mat},  there   exist constants $a_5,\,a_6$ such that, with  probability at least  $1-1/(2d)-3/n^2,$ we have
		\begin{align*}
	\big\|\beta_{T_1}-\mu^{*}\beta^{*}\big\|^2_{F}\lesssim \frac{rd\log d}{n} 
	\end{align*}
	for all
	 $T_1\in[a_5{\log(1/\alpha)}/ (\eta(|\mu^{*}|r_m-M_m\sqrt{d\log d/n})),a_6\log(1/\alpha)\sqrt{n/(d\log d)}/(\eta M_m)].$
	Moreover, for the normalized iterates $\beta_t / \| \beta _t \|_{F}$, we have
	\begin{align*}
	\left\|\frac{\beta_{T_1}}{\|\beta_{T_1}\|_{F}}-\frac{\mu^{*}\beta^{*}}{|\mu^{*}|}\right\|^2_{F}\lesssim  \frac{r d\log d}{n}.
	\end{align*}
\end{theorem}
  Similar to the vector case given in \S\ref{gaussdesign},
  as shown in the proof in Appendix \S\ref{proofmat},
  here we require  $C_{ms}$ to satisfy  $C_{ms}\ge\max\{(a_5/a_6+1)M_m/|\mu^{*}|,2M_m/|\mu^{*}|\}$
  in order to let the strong signals in $R_0$ dominate the noise and let the interval for $T_1$ to exist.
  The statistical rates hold for all such a $C_{ms}$.
As shown in
 Theorem \ref{thmmat1},
 with the  proper choices of  initialization parameter $\alpha$,  stepsize $\eta$, and the stopping time $T_1$,
 Algorithm \ref{alg2} constructs an estimator that achieves
 near-optimal statistical rates of convergence (up to logarithmic factors compared to minimax lower bound \citep{rohde2011}).
 Notice that the statistical rates established in Theorem \ref{thmmat1} are also enjoyed by the $M$-estimator
 based on the least-squares  loss function with
 nuclear norm penalty \citep{plan2016generalized,plan2017high}.
Thus, in terms of statistical estimation, applying
gradient descent to  the over-parameterized loss function in \eqref{matoverpara} is equivalent to adding a nuclear norm penalty explicitly, hence demonstrating the implicit regularization effect. \revise{Except for obtaining the optimal $\ell_2$-statistical rate, we are able to recover the true rank with high-probability by properly truncating the eigenvalues of $\beta_{T_1}$ for all $T_1\in[a_5{\log(1/\alpha)}/ (\eta(|\mu^{*}|r_m-M_m\sqrt{d\log d/n})),a_6\log(1/\alpha)\sqrt{n/(d\log d)}/(\eta M_m)]$. Comparing with the literature \cite{Lee_consistent} which studies the rank consistency via $\ell_1$-regularization, we offer a wider range for choosing the tuning parameter with known left boundary $\alpha$, instead of only setting the nuclear tuning parameter $\lambda=\tilde{\Theta}(\sqrt{rd/n})$.
\begin{theorem}\label{rankconsist}(Rank Consistency) Under the setting of Theorem \ref{thmmat1},
	for all   $$T_1\in[a_5{\log(1/\alpha)}/ (\eta(|\mu^{*}|r_m-M_m\sqrt{d\log d/n})),a_6\log(1/\alpha)\sqrt{n/(d\log d)}/(\eta M_m)],$$ we let $\tilde{\beta}_{T_1}=\sum_{i=1}^{d}\ub_i\ub_i^\top\lambda_i(\beta_{T_1})\cdot \mathbb{I}_{\{|\lambda_i(\beta_{T_1})|\ge \lambda\}}$, for all $i\in [d]$. Here $\ub_k,k\in [d]$ are eigenvectors of $\beta_{T_1}$.  Then, with   probability at least $1-2d^{-1}-3n^{-2},$  for all $\lambda\in [\alpha,(C_{ms}|\mu^{*}|-2M_m)\sqrt{d\log d/n}]$, we have $\tilde{\beta}_{T_1}$ enjoys the conclusion of Theorem \ref{thmmat1}, and rank$(\tilde{\beta}_{T_1})\le \textrm{rank}(\beta^{*})$. Moreover, when there only exists strong signals in $R_0$, we further have rank$(\tilde{\beta}_{T_1})= \textrm{rank}(\beta^{*})$. 
\end{theorem}  }



Furthermore,
  our method extends the existing works that focus on designing algorithms and studying implicit regularization phenomenon in noiseless linear matrix sensing  models with positive semidefinite signal matrices \citep{gunasekar2017implicit,Li2017AlgorithmicRI,arora2019implicit,gidel2019implicit}.
  Specifically,
  we allow  a more general class of (noisy) models and symmetric signal matrices. \revise{ Compared with \cite{Li2017AlgorithmicRI}, our methodology possesses several strengths, which include achieving low sample complexity ($\tilde{\cO}(rd)$ instead of $\tilde{\cO}(r^2d)$), allowing weak signals ($\min_{i\in R}|r_i^{*}|\gtrsim \cO((1/n)^{1/2})$ instead of $\min_{i\in R}|r_i^{*}|\gtrsim \cO((1/n)^{1/6})$), getting tighter statistical rate under noisy models ($\tilde{\cO}(dr/n)$ instead of $\tilde{\cO}(\kappa rd/n)$), and applying to a more general class of noisy statistical models. These strengths are achieved by the use of score transformation together with a refined trajectory analysis, which involves studying the dynamics of eigenvalues inside the strong signal set elementwisely with multiple stages instead of only studying the dynamics of the minimum eigenvalue with two stages.}

The way of choosing stopping time $T_1$ in the case of matrix SIM is almost the same with our method in \S\ref{stoptime}.
The only difference between them is that here we replace $\xb^\top\beta^{*}$ by {$ \tr(\Xb^\top\beta^{*})$ }
 Indeed, as we  assume  $\|\Sigma^{1/2}\beta^{*}\|_2=1$ in vector SIM and $\|\beta^{*}\|_F=1$ in matrix version for model identifiability, both $\xb^\top\beta_t$ and $\tr(\Xb^\top\beta_t)$ follow the standard normal distribution. Thus,  our results on the prediction risk in \S\ref{predriskvec} can be applied here directly.

\subsection{General Design}\label{secgenmat}
In the rest of this section, we focus on the
 low rank matrix SIM beyond Gaussian covariates.
 Hereafter,  we assume the entries of  $\Xb$ are i.i.d. random variables with a  known density function $p_0 \colon \RR\rightarrow \RR$.
 Recall that, according to the remarks following Definition \ref{defiscore},  the   score function $S(\Xb) \in \RR^{d\times d}$   is defined as
\begin{align*}
S(\mathbf{X})_{j,k}:=S(\mathbf{X}_{j,k})={-p_0'(\mathbf{X}_{j,k})}/{p_0(\mathbf{X}_{j,k})},
\end{align*}
where $  S(\Xb)_{j,k}$ and $\Xb_{j,k} $ are   the
$(j,k)$-th entries  of $S(\Xb)$ and $\Xb$  for all $j,k \in [d]$.
However, similar to the results in \S\ref{secgenvec},
the entries of $S (\Xb)$ can have heavy-tailed distributions and thus $n^{-1} \sum_{i=1}^n y_i \cdot S(\Xb_i)$ may not converge its expectation $\EE[Y\cdot S(\Xb)]$ efficiently in terms of spectral norm.
Here $\Xb_i$ is the $i$-th observation of the covariate $\Xb$.
To tackle such a challenge,
we employ   a shrinkage approach
   \citep{Catoni12,fan2016shrinkage,Minsker2018} to construct a robust estimator of  $\EE[Y\cdot S(\Xb)]$.
Specifically, we let
\begin{equation*}  
\phi(x) =\left\{
\begin{aligned}
\log(1-x+x^2/2),  &\qquad  x\le 0, \\
\log(1+x+x^2/2), &  \qquad x> 0
\end{aligned}
\right.,
\end{equation*}
which is approximately $x$ when $x$ is small and grows at logarithmic rate for large $x$.  The rescaled version $\lambda^{-1} \phi(\lambda x)$ for $\lambda \to 0$ behaves like a soft-winsorizing function, which has been widely used in statistical mean estimation with finite bounded moments \citep{Catoni12,brownlees2015}.
  For any matrix $\mathbf{X}\in \mathbb{R}^{d\times d}$, we apply  spectral decomposition to  its Hermitian dilation and obtain
\begin{align*}\mathbf{X}^{*}:=
\left[
\begin{array}{cc}
\mathbf{0}& \mathbf{X} \\
\mathbf{X}^\mathsf{T}& \mathbf{0}
\end{array}
\right]=\mathbf{Q}\mathbf{\Sigma}^{*}\mathbf{Q}^\mathsf{T},
\end{align*}
where $\Sigma^* \in \RR^{2d\times 2d}$ is a diagonal matrix.
Based on such a decomposition, we define  $\tilde{\Xb}=\Qb\phi(\mathbf{\Sigma^{*}})\Qb^\mathsf{T}$,
where $\phi $ applies elementwisely to $\Sigma^*$.
Then we write $\tilde \Xb$ as a block matrix as
\begin{align*}\tilde{\Xb}:=
\left[
\begin{array}{cc}
\tilde{\Xb}_{11}& \tilde{\Xb}_{12} \\
\tilde{\Xb}_{21}& \tilde{\Xb}_{22}
\end{array}
\right],
\end{align*}
where each block of $\tilde \Xb $ is in $\RR^{d\times d}$.
We further define a mapping $\phi_1 \colon \RR^{d\times d} \rightarrow \RR^{d\times d}$ by letting  $\phi_1(\Xb):=\tilde{\Xb}_{12}$,  which is a regularized version of $\Xb$. Given data $y_1,\Xb_1$, we finally define $\mathcal{H}(\cdot)$ as
\begin{align} \label{eq:thresholding_oper}
\mathcal{H}(y_1S(\Xb_1),\kappa):=1/\kappa\cdot\phi_1(\kappa y_1\cdot S(\Xb_1)),
\qquad  \forall \kappa >0,
\end{align}
where $\kappa  $ is a thresholding parameter, converging to zero.  This method is in a similar spirit of robustifying the singular value of $\Xb$.  Based on the operator $\cH$ defined in \eqref{eq:thresholding_oper},
we define a loss function $L(\Wb, \Vb)$ as
\begin{align}\label{lossheavymat}
L(\Wb,\Vb):=\langle \Wb\Wb^\top-\Vb\Vb^\top,\Wb\Wb^\top-\Vb\Vb^\top \rangle-\frac{2}{n}\sum_{i=1}^{n} \bigl \langle \Wb\Wb^\top-\Vb\Vb^\top,\mathcal{H}(y_i S(\Xb_i),\kappa))  \bigr \rangle.
\end{align}
After over-parameterizing $\beta$ as  $\Wb\Wb^\top-\Vb\Vb^\top$,
we propose to construct an estimator of $\beta^*$
by applying
gradient descent on the following loss function   in \eqref{lossheavymat} with respect to $\Wb,\Vb$.
See
Algorithm~\ref{alg6} in \S\ref{alggenmat} for the details of the algorithm.

In the following, we present the statistical rates of convergence for the obtained estimator.
We first introduce the  assumption on $Y$ and $p_0$.

\begin{assumption} \label{genass1mat}
 We assume that both the  response variable $Y$ and entries of $S(\mathbf{X})$ have bounded fourth moments.
	 Specifically, there exists an absolute  constant $M$ such that
	\begin{align*}
	\mathbb{E}\left[Y^4 \right]\le M, \qquad \mathbb{E}\left[S(\mathbf{X})_{i,j}^4\right]\leq  M,\quad \forall\,(i,j)\in[d]\times [d].
	\end{align*}
	Moreover, we assume that $\mu^* = \EE[ f'( \langle \Xb, \beta^* \rangle )]$ is a nonzero constant such that
  $|\mu^{*}|/M=\Theta(1)$.
\end{assumption}

Next, we  present the main theorem for low rank matrix SIM.
  \begin{theorem}\label{genthmmat1}
	In Algorithm \ref{alg6}, we set  parameter $\kappa$ in \eqref{eq:thresholding_oper} as $\kappa=\sqrt{\log(4d)/(nd\cdot M)}$ and let the initialization parameter $\alpha$ and the stepsize $\eta$ satisfy $0<\alpha\le M_{mg}^2/d$ and {$0<\eta\le1/[12(|\mu^{*}|+M_{mg})]$}, where $M_{mg}$ is a constant proportional to $M$.
	Then,
	under Assumption \ref{genass1mat},  there exist absolute constants $a_7,\,a_8$ such that,
	with probability at least $1 - (4d)^{-2}$, we have
		\begin{align*}
	\big\|\beta_{T_1}-\mu^{*}\beta^{*}\big\|^2_{F}\lesssim \frac{r d\log d}{n},
	\end{align*}	
	 for all  $T_1\in[a_7{\log(1/\alpha)}/ ( \eta(|\mu^{*}|r_m-M_{mg}\sqrt{d\log d/n}) ) ,a_8\log(1/\alpha)\sqrt{n/(d\log d)}/(\eta M_{mg})]$,
	 Moreover, for the normalized   iterate $\beta_{t} / \| \beta_{t}  \|_{F}$, we have
	\begin{align*}
	\left\|\frac{\beta_{T_1}}{\|\beta_{T_1}\|_{F}}-\frac{\mu^{*}\beta^{*}}{|\mu^{*}|}\right\|^2_{F}\lesssim  \frac{r d\log d}{n}.
	\end{align*}
\end{theorem}

For low rank matrix SIM, when the hyperparameters of the gradient descent algorithm are properly chosen,
we also capture the implicit regularization phenomenon by applying a simple optimization procedure to  over-parameterized loss function with heavy-tailed measurements.
Here, applying the thresholding operator $\cH$ in \eqref{eq:thresholding_oper} can also be viewed as a data pre-processing step, which arises due to handling heavy-tailed observations.   Note that the way of choosing $C_{ms}$ here is similar with the way in Theorem \ref{thmmat1}, in order to ensure the convergence rate and existence of a time interval, so we omit the details.
Note that the $\ell_2$-statistical rate given in Theorem \ref{genthmmat1} are minimax optimal up to a logarithmic term \citep{rohde2011}.
Similar results were also obtained by \cite{plan2016generalized,yang2017high,goldstein2018structured,na2019high} via adding explicit nuclear norm regularization.
Thus, in terms of statistical recovery,
when employing  the thresholding in \eqref{eq:thresholding_oper} and over-parameterization, gradient descent enforces implicit regularization that has the same effect as the nuclear norm penalty. \revise{In addition, in terms of the rank consistency result for the heavy-tailed case, if we also let $\tilde{\beta}_{T_1}=\sum_{i=1}^{d}\ub_i\ub_i^\top\lambda_i(\beta_{T_1})\cdot \mathbb{I}_{\{|\lambda_i(\beta_{T_1})|\ge \lambda\}},$ 
 then for all $\lambda\in [\alpha,(C_s|\mu^{*}|-2M_{mg})\sqrt{d\log d/n}],$ we achieve the same results with Theorem \ref{rankconsist}. }







\section{Numerical Experiments}
\label{sectsimu}
In this section, we  illustrate the performance of the proposed estimator in different settings via simulation studies. We let $\epsilon \sim N(0,0.5^2)$ in our models defined in \eqref{vecSIM} and \eqref{matrixSIM} and choose the link function to be one of    $\{ f_j \}_{j=1}^8$, whose details are given in    Figures \ref{functions1} and \ref{functions2}.
 \begin{figure}[H]
	\centering
	\begin{tabular}{cccc}
		\hskip-30pt\includegraphics[width=0.25\textwidth]{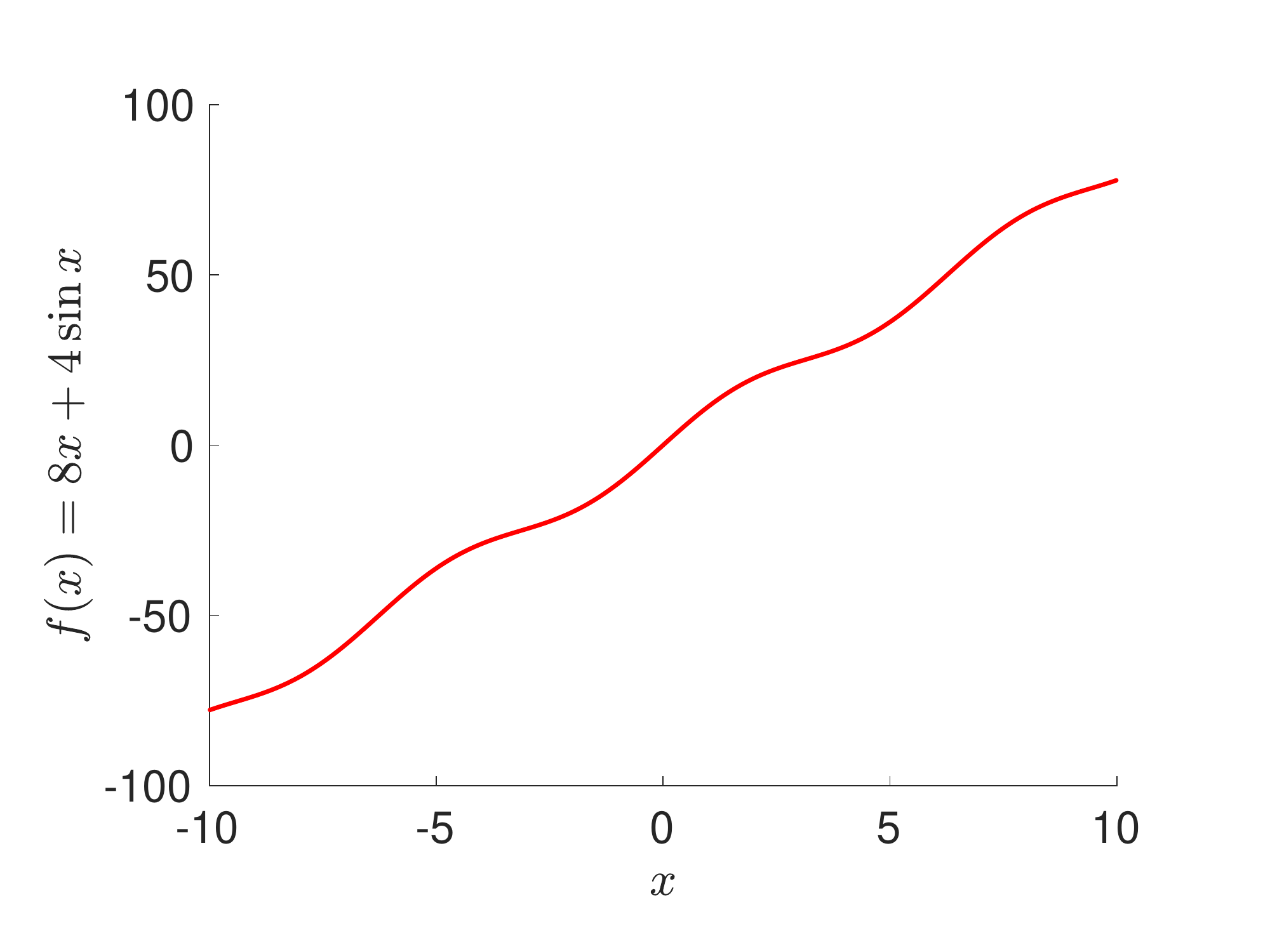}
		&
		\hskip-6pt\includegraphics[width=0.25\textwidth]{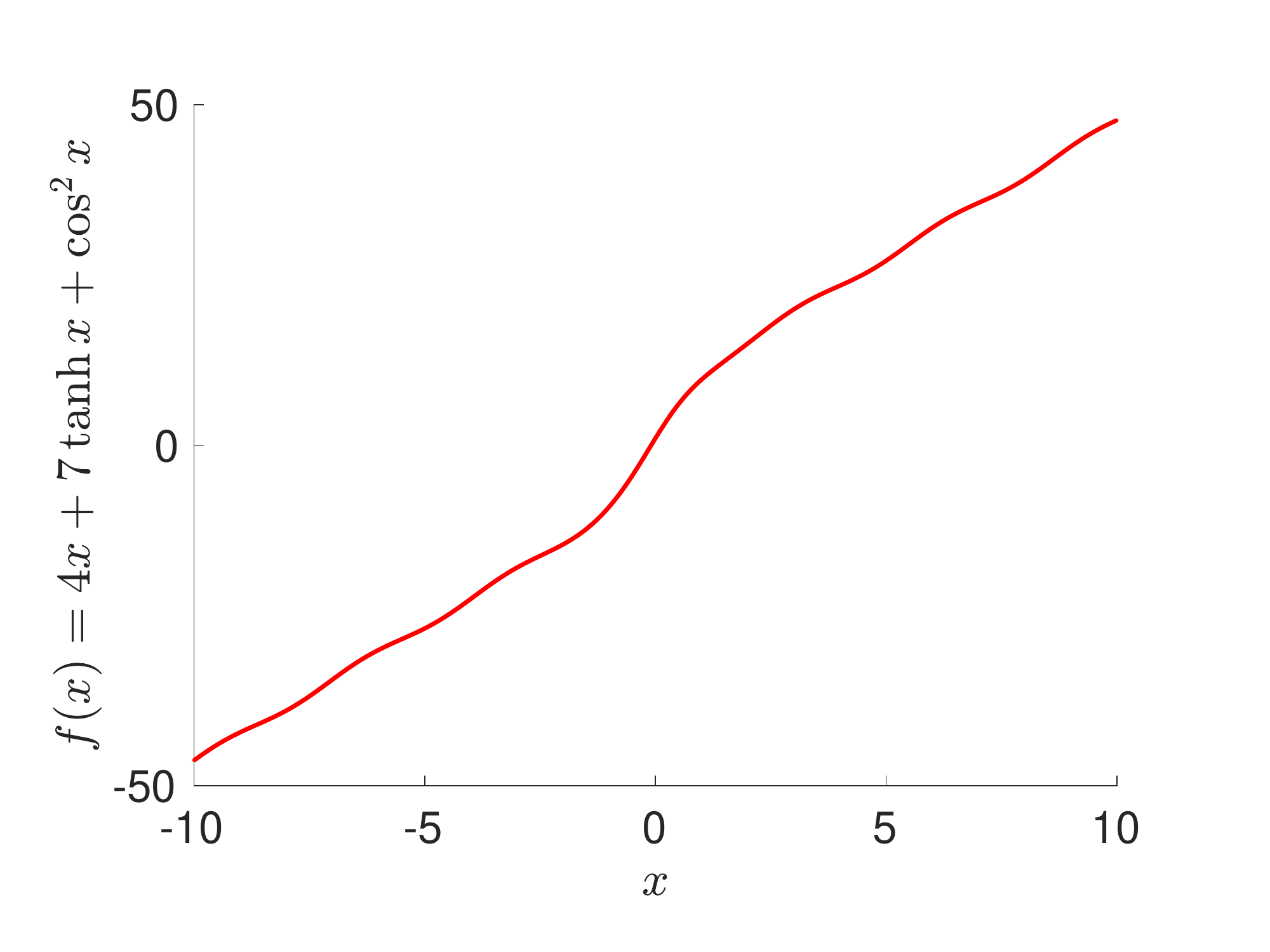}
		&
		\hskip-5pt\includegraphics[width=0.25\textwidth]{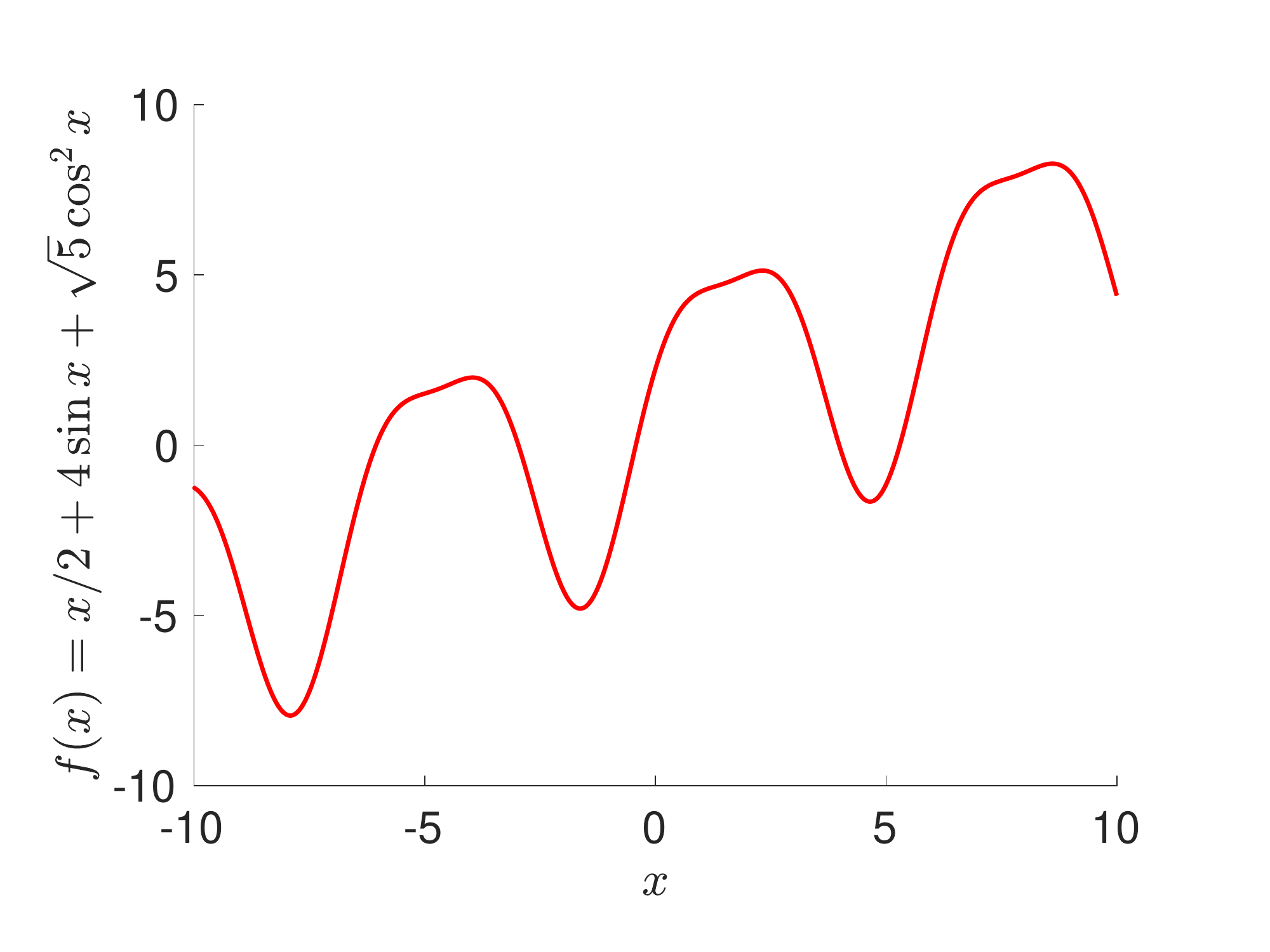}
		&
		\hskip-5pt\includegraphics[width=0.25\textwidth]{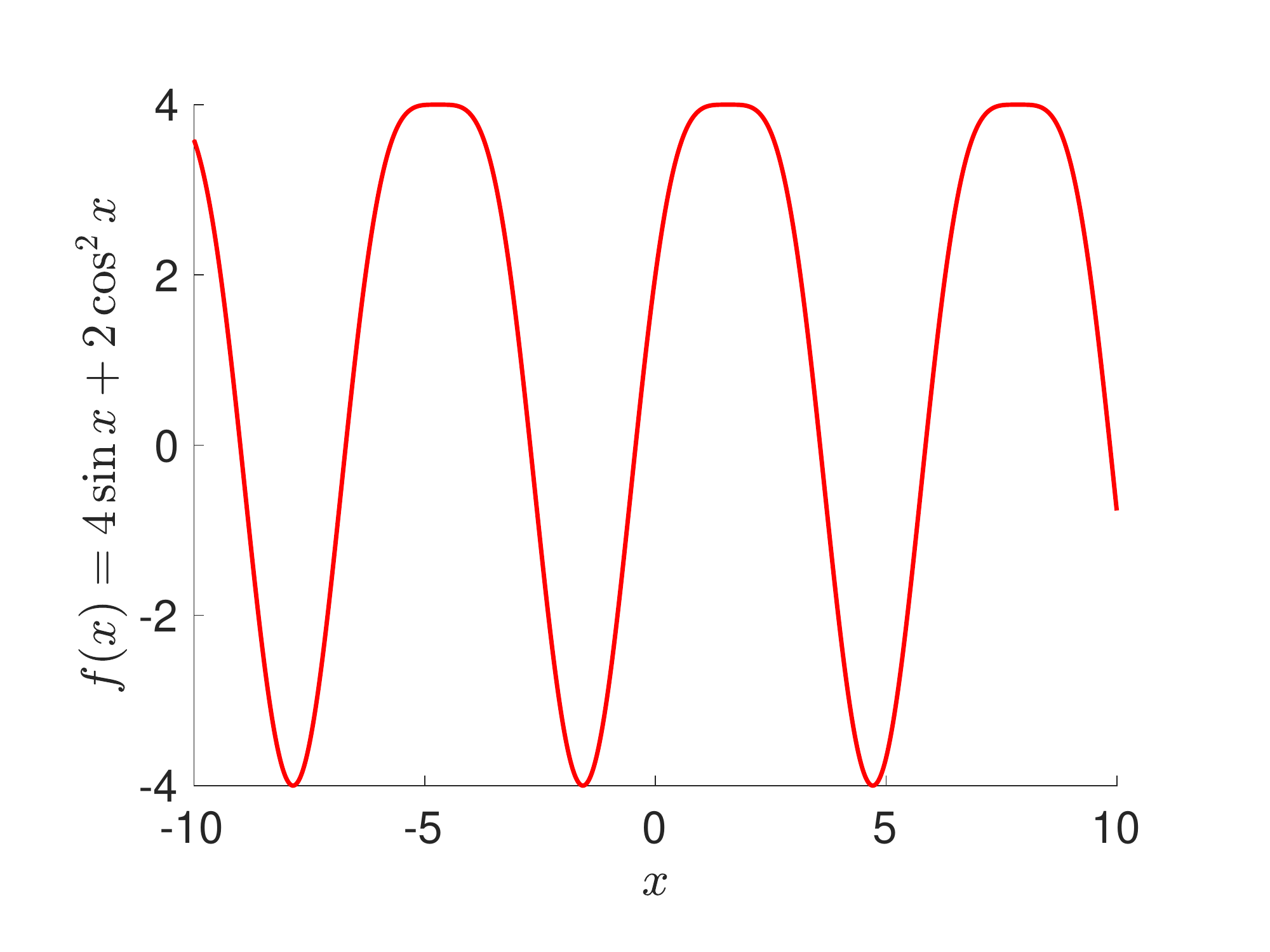}\\
		(a)  & (b) &(c)
	\end{tabular}\\
	\caption{Plot of link functions (a): $f_1(x)=8x+4\sin x$, (b): $f_2(x)=4x+7\tanh x+\cos^2 x$, (c): $f_3(x)=x/2+4\sin x+\sqrt{5}\cos^2 x$ and (d): $f_4(x)=4\sin x+2\cos^2 x$ }
	\label{functions1}
\end{figure}
 \begin{figure}[H]
	\centering
	\begin{tabular}{cccc}
		\hskip-30pt\includegraphics[width=0.25\textwidth]{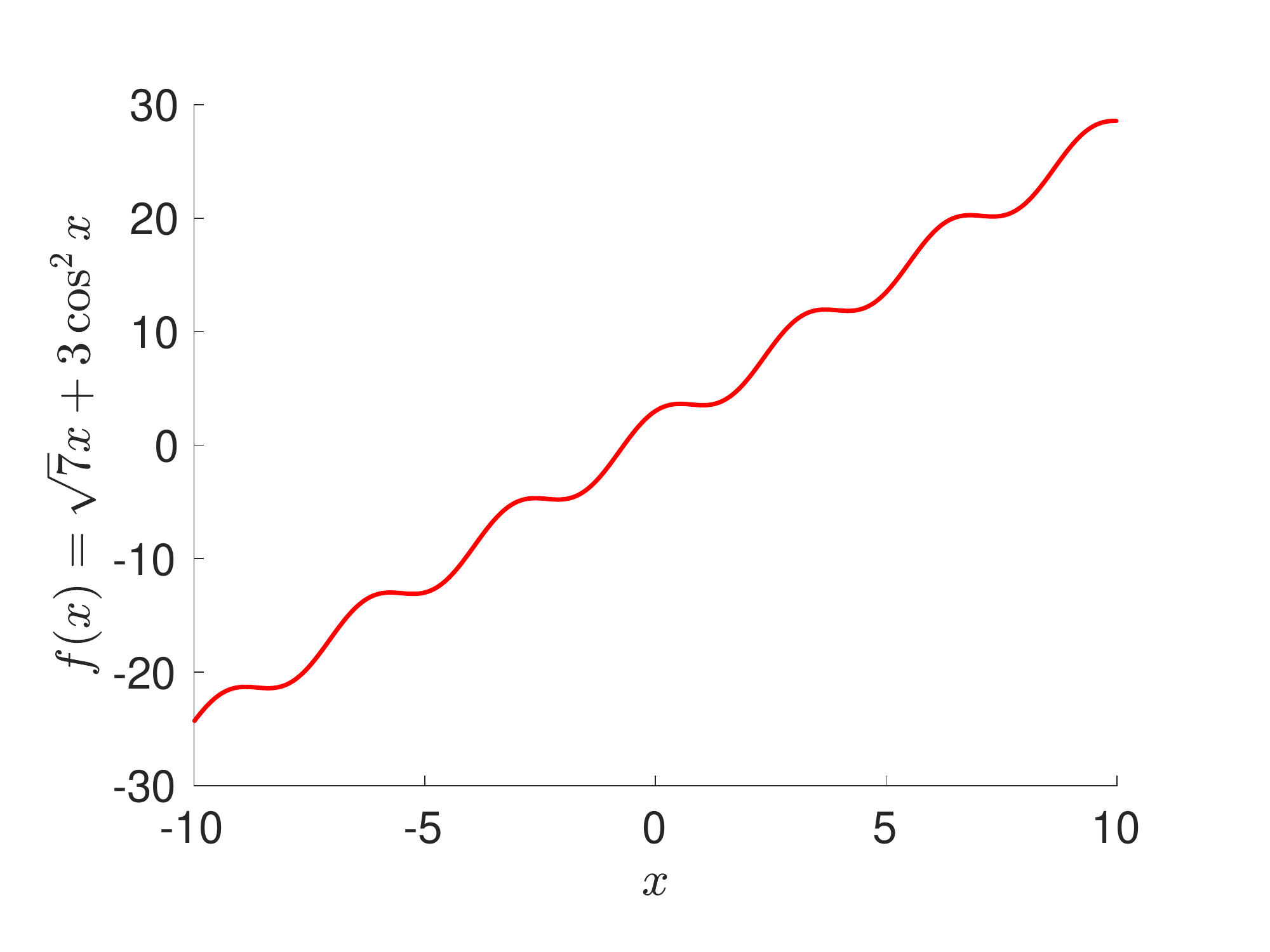}
			&\hskip-6pt\includegraphics[width=0.25\textwidth]{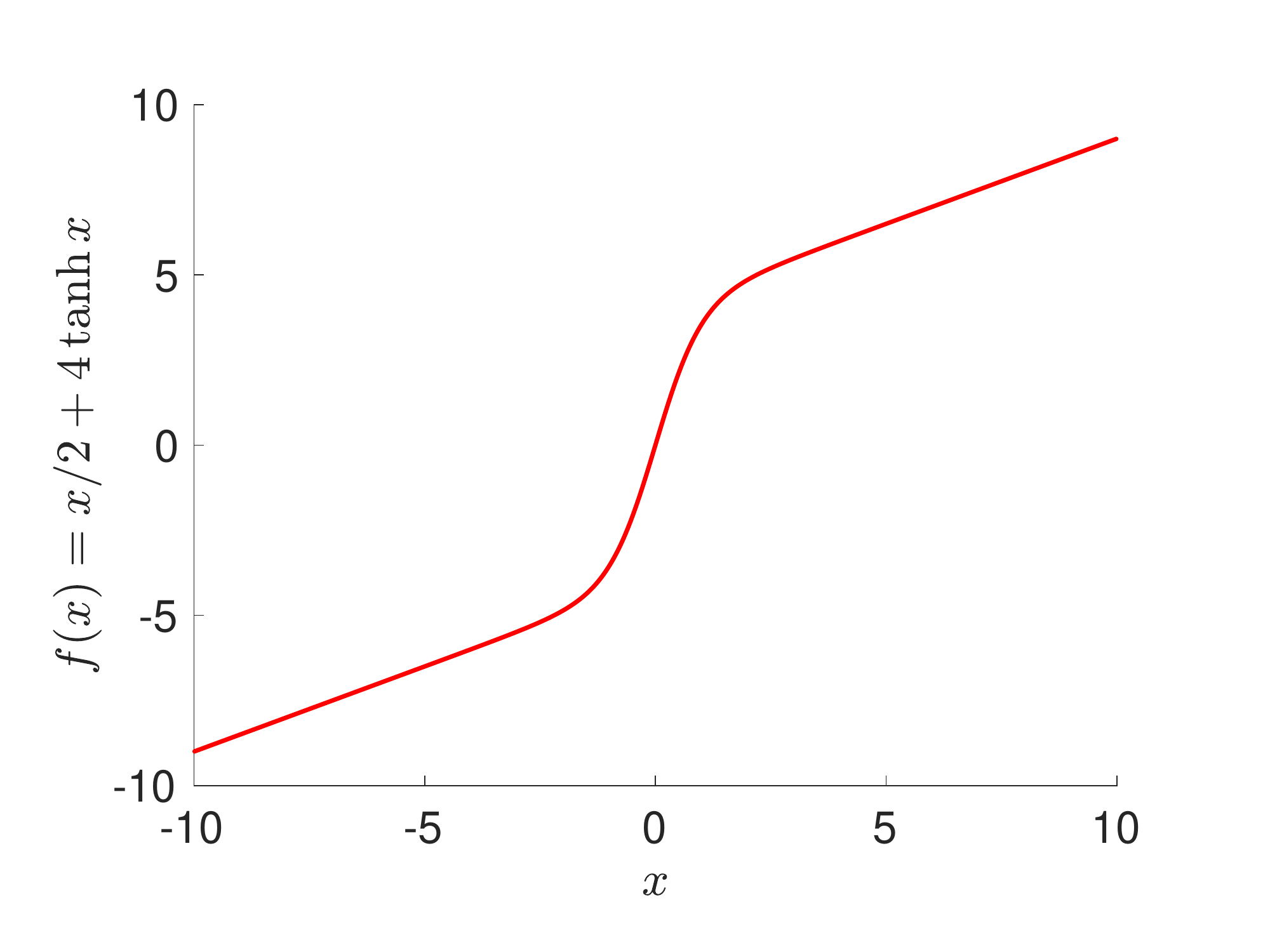}
	&	\hskip-5pt\includegraphics[width=0.25\textwidth]{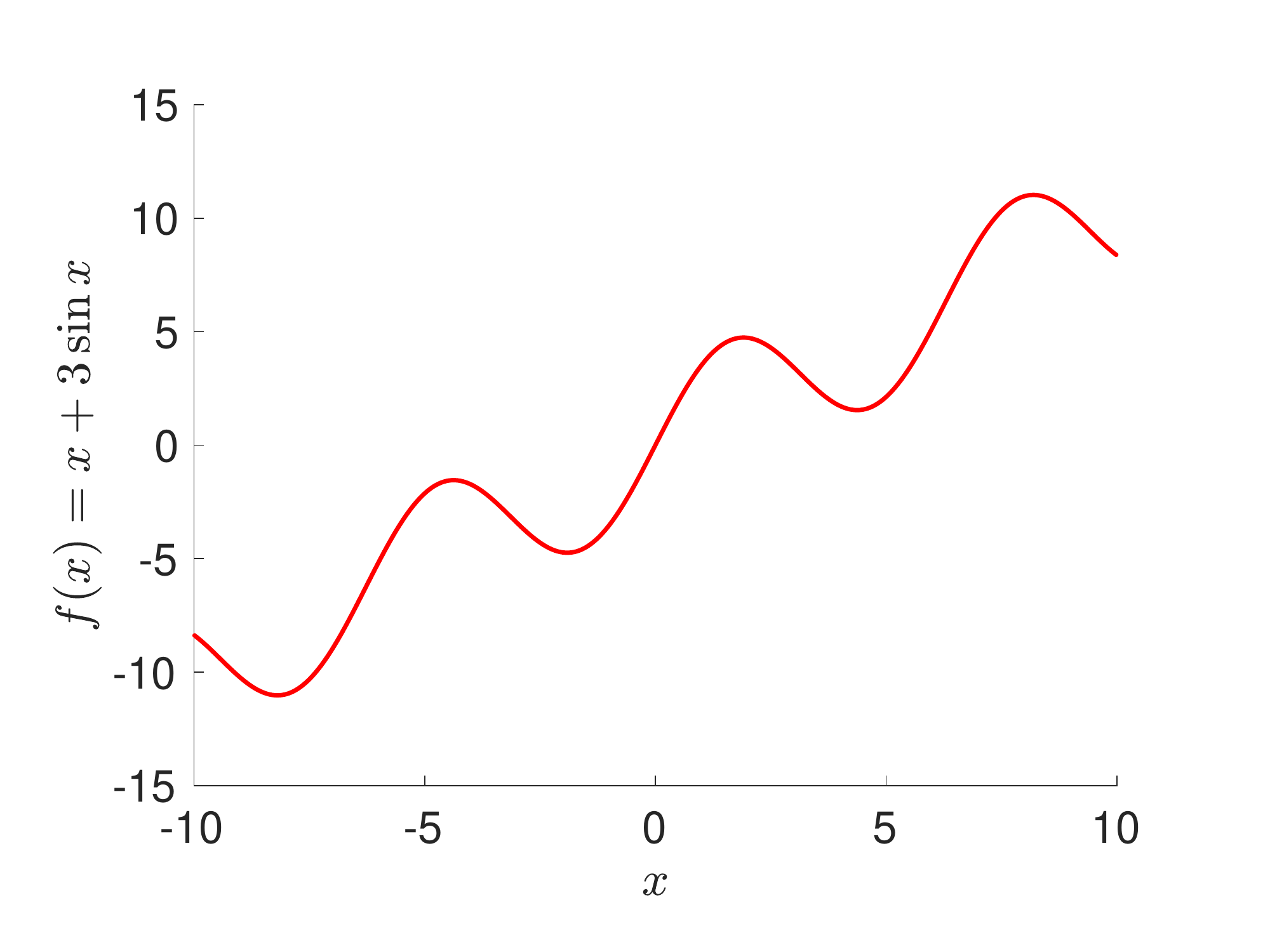}
		&\hskip-5pt\includegraphics[width=0.25\textwidth]{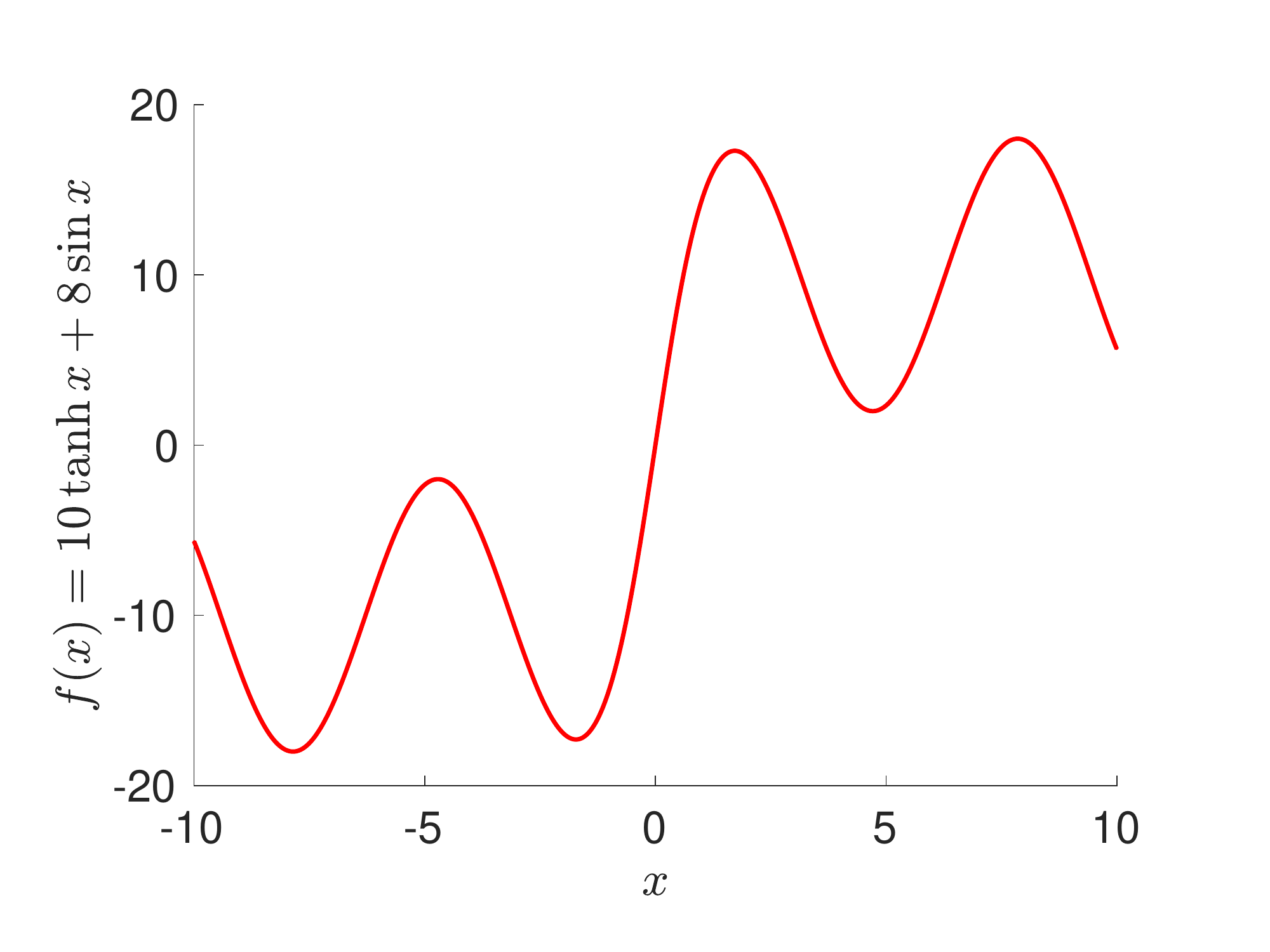}\\
		(a)  & (b) &(c)  &(d)
	\end{tabular}\\
	\caption{Plot of link functions  (a): $f_5(x)=\sqrt{7}x+3\cos^2 x$,  (b): $f_6(x)=x/2+4\tanh x$, (c): $f_7(x)=x+3\sin x$ and (d): $f_8(x)=10\tanh x+8\sin x$. }
	\label{functions2}
\end{figure}

To measure the estimation accuracy, we use $\text{dist}(\hat\beta,\beta^{*})=\min\{\|\hat\beta/\|\hat\beta\|_{F}-\beta^{*}\|_{F}, \|\hat\beta/\|\hat\beta\|_{F}+\beta^{*}\|_{F}\}$, where the subscript $F$ stands for Frobenius norm, which reduces to the Euclidean norm in the vector case.  The number of simulations is 100. 

\subsection{Simulations on Sparse Vectors}\label{simu_vec}
Recall that Theorems \ref{thmvec1} and \ref{thmvec2} establish the $\sqrt{s \log p/n}$ statistical rate of convergence in the  $\ell_2$-norm.
To vary this, we fix  $p=2000$, $s$ to be one of $ \{8,10,12\}$, and use the value of $\sqrt{ s \log p /n}$ to determine $n$. In addition,
we choose the support of $\beta^{*}$ randomly among all subsets of $\{ 1, \ldots, p\}$ with cardinality $s$.  For each $j\in\text{supp}(\beta^{*})$, we set $\beta_j^{*}=1/\sqrt{s}\cdot\text{Uniform}(\{-1,1\})$.
Besides, we let the entries of the covariate $\xb$ have i.i.d. distributions, which are either the standard Gaussian distribution,  Student's t-distribution with $5$ degrees of freedom, or the Gamma distribution with shape parameter $8$ and scale parameter 0.1.
Based on $\beta^*$, the distribution of $\xb$,  and one of the aforementioned univariate functions $\{ f_j\}_{j=1}^4$, we
generate $n$ i.i.d. samples $\{\xb_i,y_i\}_{i=1}^{n}$ from the vector SIM given in \eqref{vecSIM}. 
As for the  optimization procedure, throughout \S\ref{simu_vec}, we set the initialization parameter $\alpha=10^{-5}$, stepsize $\eta=0.005$ in Algorithms \ref{alg1} and \ref{alg5}.  Our estimator $\hat\beta$ is chosen by $\hat\beta=\argmin_{\beta_t}\text{dist}(\beta_t,\beta^{*}),$ where $\beta_t$ is the  $t$-th iterate of  Algorithm \ref{alg1} and Algorithm \ref{alg5}. The choice of stoping time is ideal but serves purposes.  As shown in our asymptotic results, there is an intervals of sweet stopping time.  By using the data driven choice, we get similar results, but take much longer time.

With the standard Gaussian distributed covariates, we plot the average distance dist$(\hat\beta,\beta^{*})$ against $\sqrt{s\log p/n}$ in Figure \ref{gaussvec} for $f_1$ and $f_2$ respectively, based on $100$ independent trails for each $n$. The results show that the estimation error is bounded effectively by a linear function of signal strength $\sqrt{s\log p/n}$.  Indeed, the linearity holds surprisingly well, which corroborates our theory.

 \begin{figure}[H]
	\centering
	\begin{tabular}{cc}
		\hskip-30pt\includegraphics[width=0.4\textwidth]{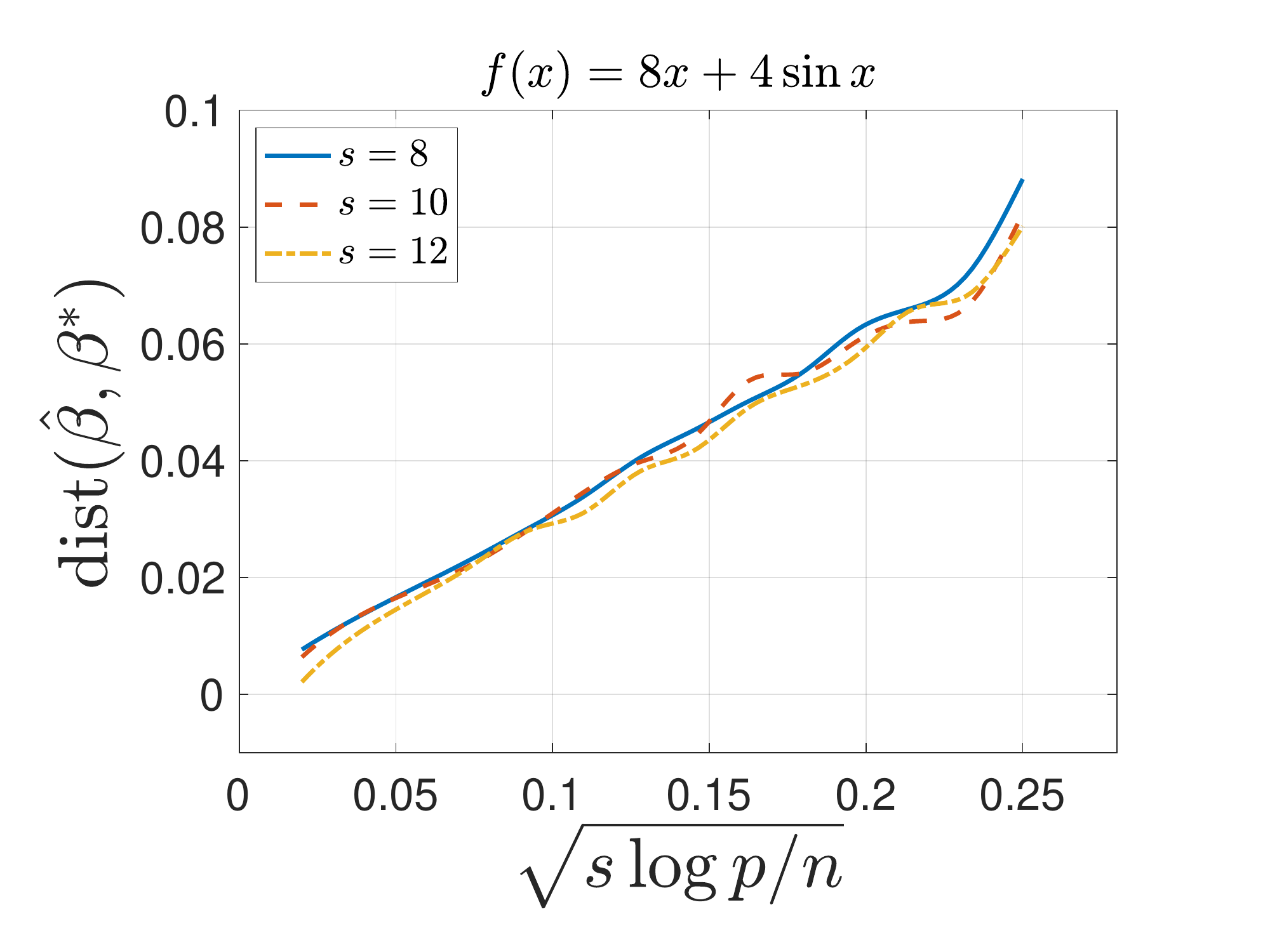}
		&\hskip-6pt\includegraphics[width=0.4\textwidth]{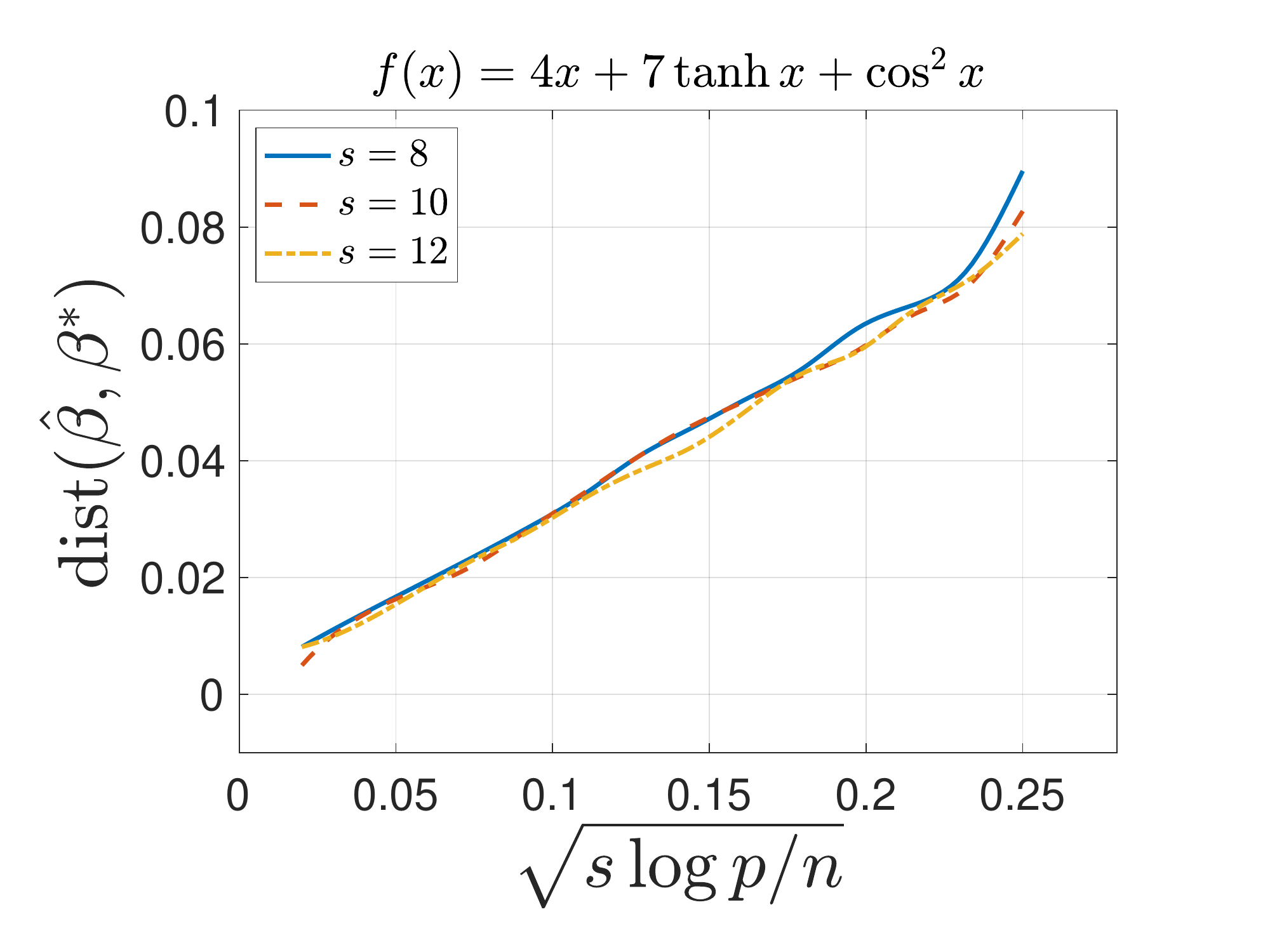}\\
		(a) &(b)
		
			\end{tabular}\\
	\caption{The average $\ell_2$-distances between the true  parameters $\beta^{*}$ and estimated parameters $\hat\beta$ in vector SIM with standard Gaussian distributed covariates and (a) link function $f_1$ and (b) link function $f_2$.}
\label{gaussvec}
\end{figure}
 As for generally distributed covariates, we set $p_0(x)$ given in Definition \ref{vectorSIM} to be one of the following distributions:  (i)  Student's t-distribution with 5 degrees of freedom and (ii) Gamma distribution with shape parameter $8$ and scale parameter 0.1.   The score functions of these two distributions are given by  $S(x)=6x/(5+x^2)$ and $S(x)=10-7/x$, respectively. In addition, the truncating parameter $\tau$ in Algorithm \ref{alg5} is taken as $\tau=2(n/\log p)^{1/4}$. We then plot distance dist$(\hat\beta,\beta^{*})$ against $\sqrt{s\log p/n}$ in Figure \ref{t_gammavec} for link functions $f_3$ and $f_4$ with t$(5)$ and Gamma$(8,0.1)$ distributed covariates respectively, based on 100 independent experiments. It also worths noting that the estimation errors align well with a linear function of $\sqrt{s\log p/n}$.
  \begin{figure}[H]
 	\centering
 	\begin{tabular}{cc}
 		\hskip-30pt\includegraphics[width=0.4\textwidth]{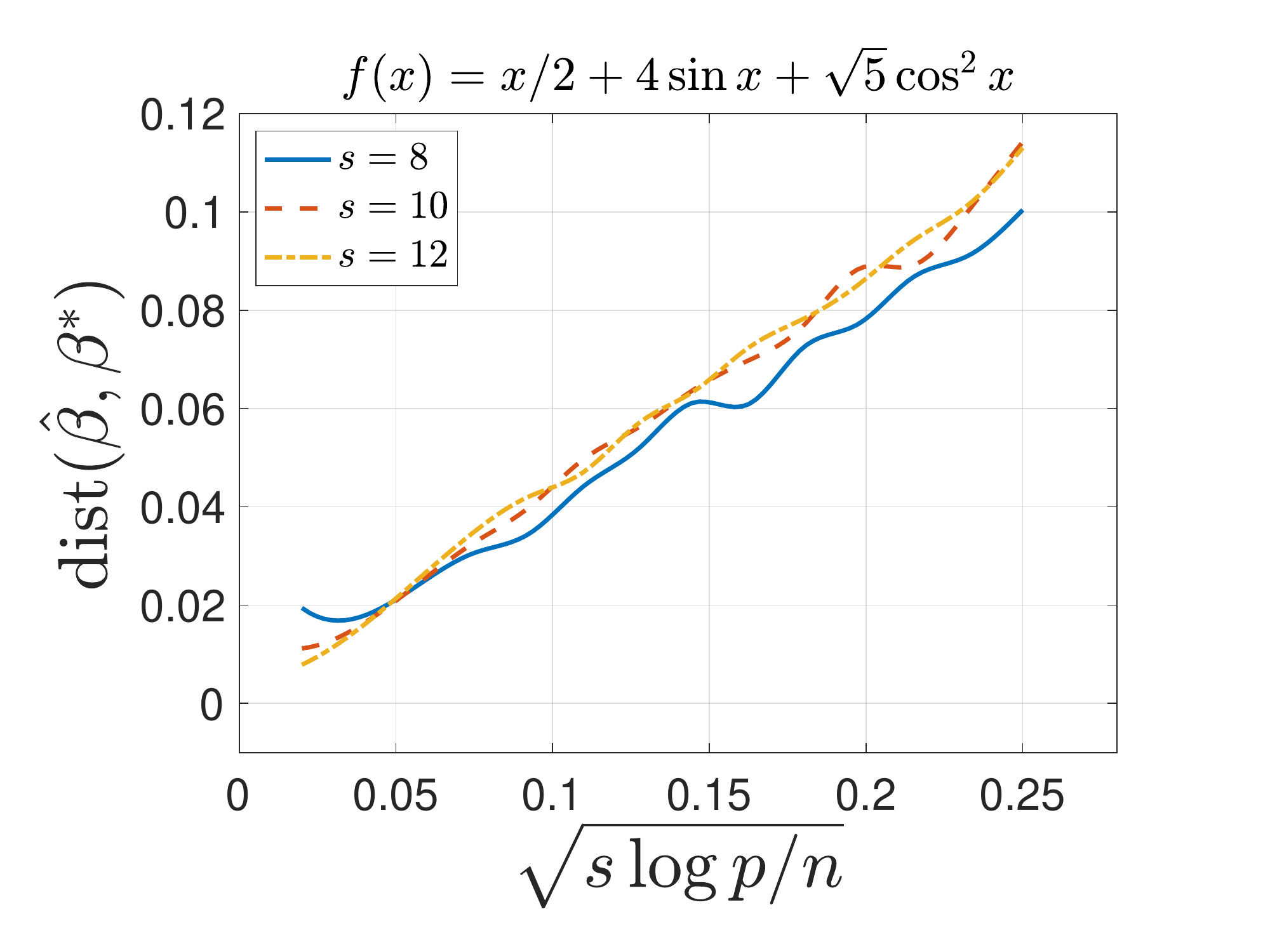}
 		&\hskip-6pt\includegraphics[width=0.4\textwidth]{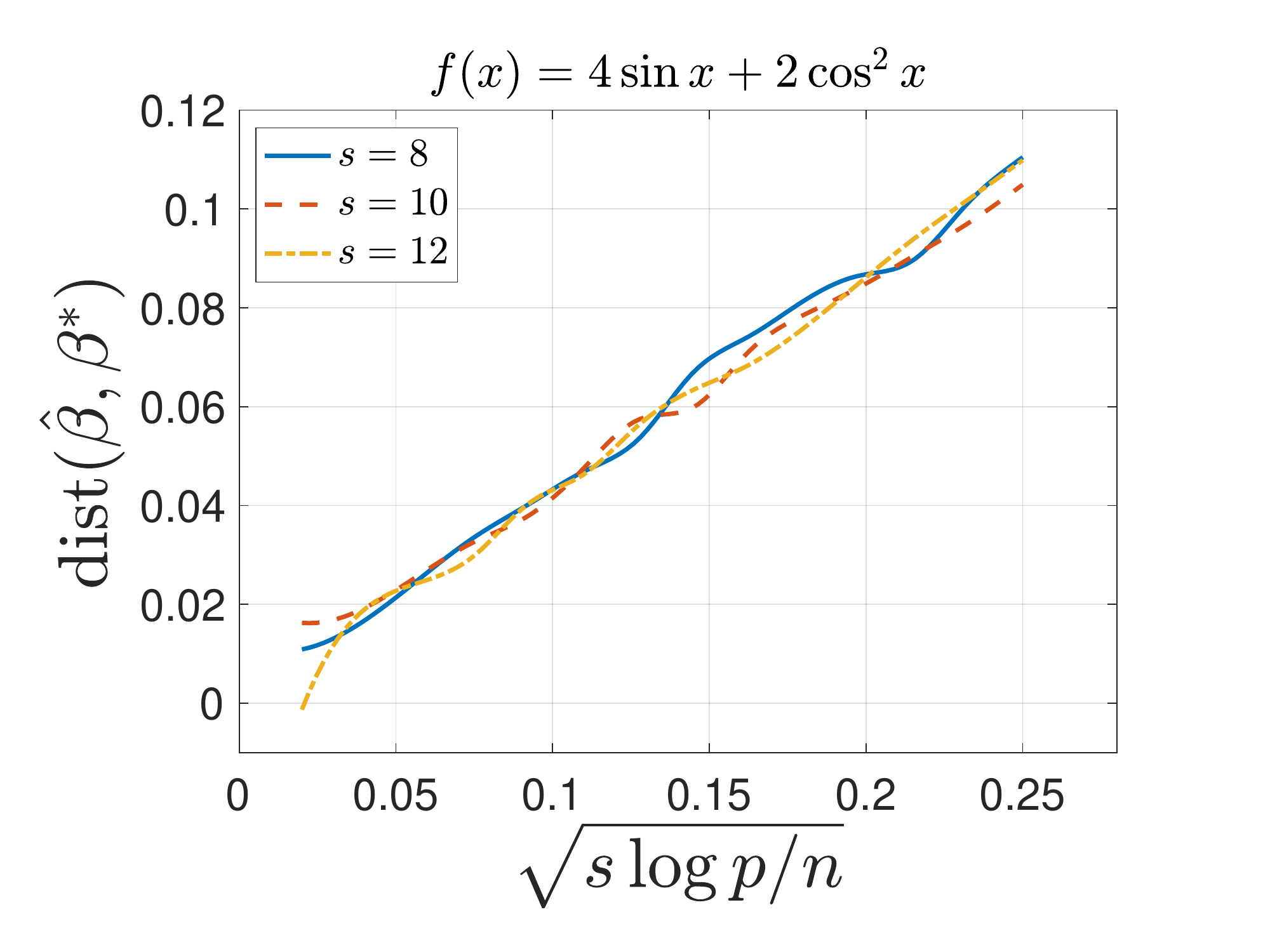}\\
 		(a) &(b)
 	\end{tabular}\\
 	\caption{The averaged $\ell_2$-distances between the true parameter and estimated parameters in vector SIM for
 (a) $t(5)$ distributed covariates with the link function $f_3$ and (b) Gamma$(8,0.1)$ distributed covariates and the link function $f_4$.}
 	\label{t_gammavec}
 \end{figure}

\subsection{Simulations on Low Rank Matrices}\label{simu_mat}
In the scenario of low rank matrix, statistical rate in Frobenius norm is $\sqrt{rd\log d/n}$, according to Theorems \ref{thmmat1} and \ref{genthmmat1}. Throughout \S\ref{simu_mat}, we fix dimension $d=25$, and for each $r\in\{1,3,5\}$, we use $\sqrt{rd\log d/n}$ to determine $n$. The true parameter matrix $\beta^{*}$ is set to be $\Ub \mathbf{S} \Ub^\top$, where $\Ub\in\RR^{d\times d}$ is any random orthogonal matrix and $\mathbf{S}$ is a diagonal matrix with $r$ nonzero entries chosen randomly among the index set  $\{1,\dots, d\}$. Moreover, we set the nonzero diagonal entries of $\mathbf{S}$ as $1/\sqrt{r}\cdot\text{Uniform}(\{-1,1\})$. Besides, we also let every entry of the covariate $\Xb$ have i.i.d. distribution, which is one of the same three distributions in \S\ref{simu_vec}. Finally, we utilize our true parameter $\beta^{*}$, the distribution of $\Xb$ and one of $\{f_j\}_{j=5}^{8}$ to generate $n$ i.i.d. data $\{\Xb_i,y_i\}_{i=1}^{n}$ based on \eqref{matrixSIM}. As for the optimization procedure, throughout \S\ref{simu_mat}, we set the initialization parameter $\alpha=10^{-3}$, stepsize $\eta=0.005$ and implement the Algorithm \ref{alg2} and Algorithm \ref{alg6} for Gaussian and general design respectively. Our estimator $\hat\beta$ is also chosen by $\hat\beta=\argmin_{\beta_t}\text{dist}(\beta_t,\beta^{*})$, where $\beta_t$ is the the $t$-th iterate given  in the Algorithm \ref{alg2} and Algorithm \ref{alg6}.  Again, this is the ideal choice of stopping time, but serves the purpose as the result does not depend very much on the proper choice of stopping time.

With the standard Gaussian distributed covariates, we plot the averaged distance dist$(\hat\beta,\beta^{*})$ against $\sqrt{rd\log d/n}$ in Figure \ref{gaussmat} for $f_5$ and $f_6$ respectively, based on $100$ independent trails for each case. The estimation error again follows linearly on $\sqrt{rd\log d/n}$.   The simulation results are consistent what is predicted by the theory.
  \begin{figure}[H]
	\centering
	\begin{tabular}{cc}
		\hskip-30pt\includegraphics[width=0.40\textwidth]{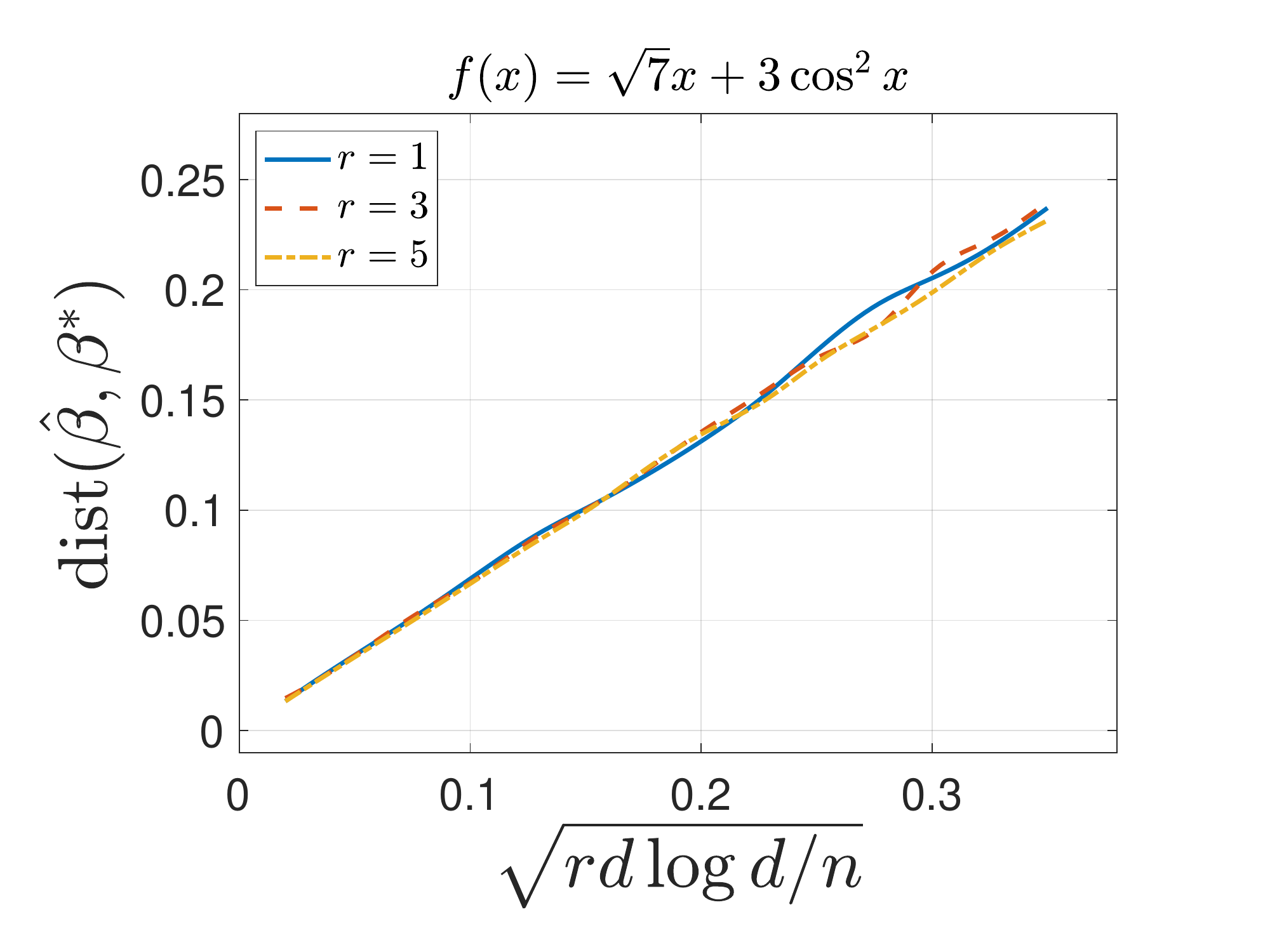}
		&\hskip-6pt\includegraphics[width=0.40\textwidth]{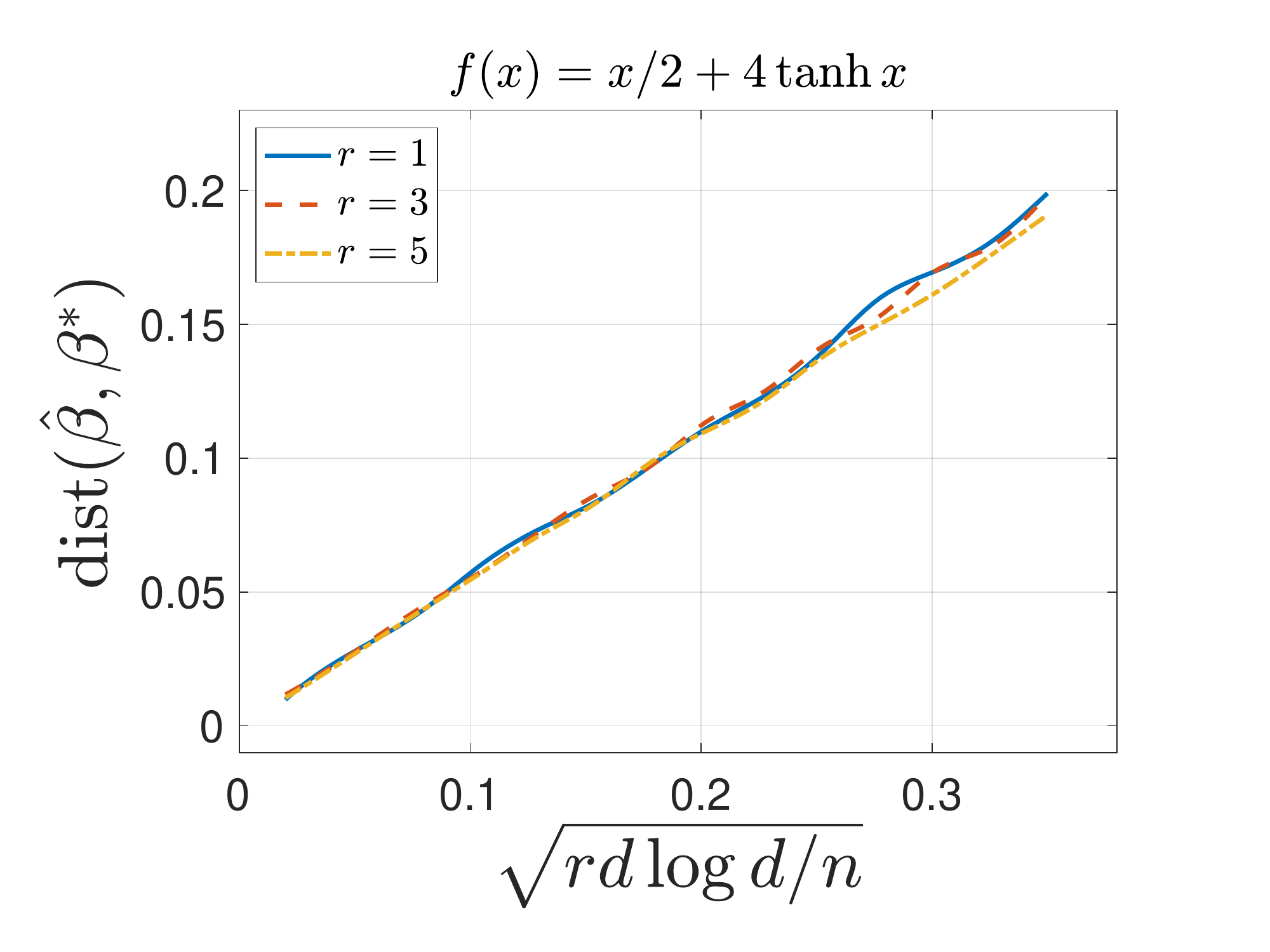}\\
		(a) &(b)
	\end{tabular}\\
	\caption{The averaged $\ell_2$-distances between the true parameter $\beta^{*}$ and estimated parameter matrices $\hat\beta$ in SIM with standard Gaussian distributed covariates and (a) the link function  $f_5$  and (b) the link function $f_6$.}
\label{gaussmat}
\end{figure}
We also show distance dist$(\hat\beta,\beta^{*})$ against $\sqrt{rd\log d/n}$ in Figure \ref{t_gammamat} for $f_7$ and $f_8$ with t$(5)$ and Gamma$(8,0.1)$ distributed covariates respectively, based on 100 independent experiments, which is in line with the theory.  Here the shrinkage parameter $\kappa$ in Algorithm \ref{alg6} is set to be $\kappa=2\sqrt{\log(4d)/(nd)}$.

  \begin{figure}[H]
	\centering
	\begin{tabular}{cc}
		\hskip-30pt\includegraphics[width=0.40\textwidth]{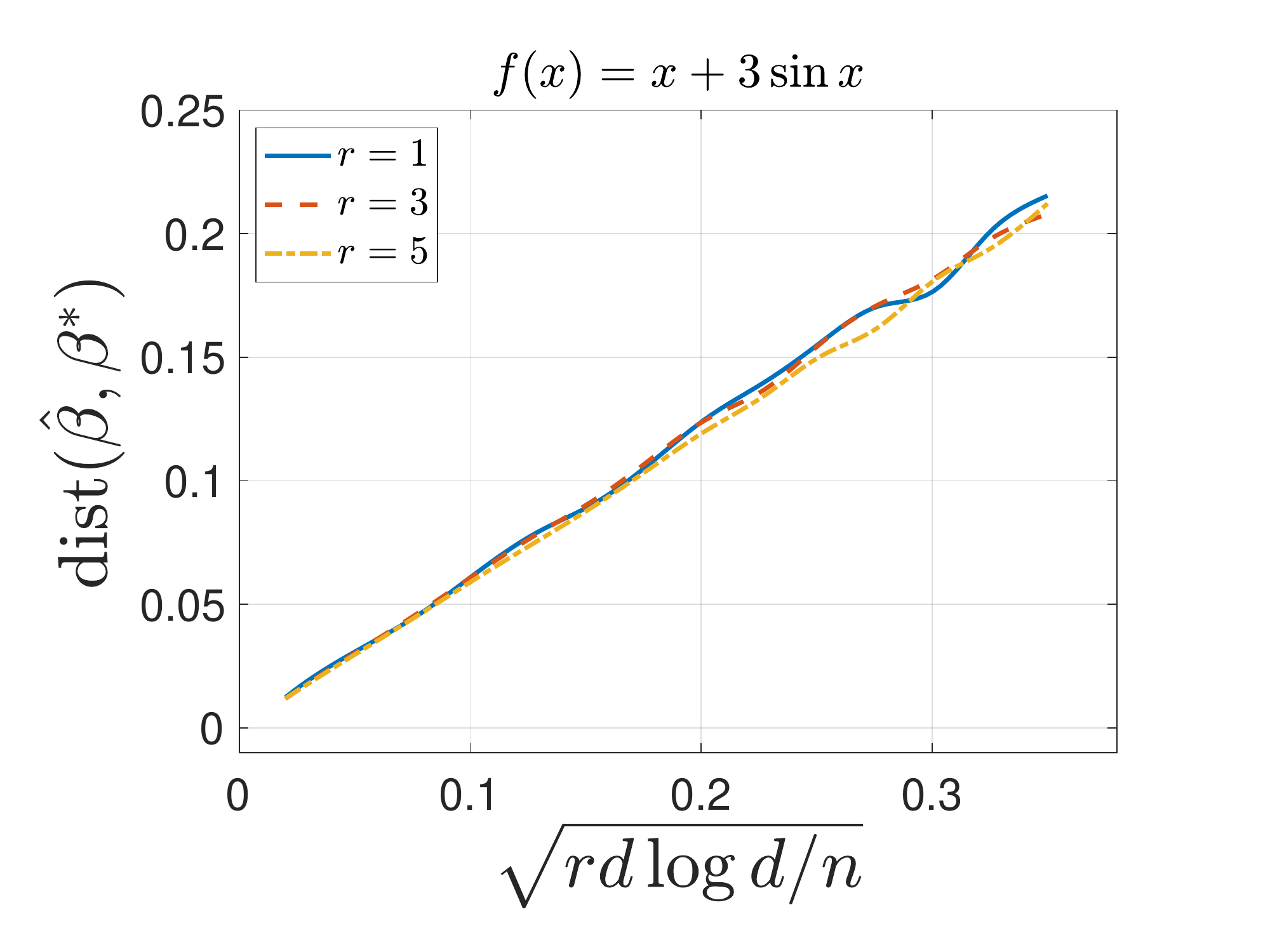}
		&\hskip-6pt\includegraphics[width=0.40\textwidth]{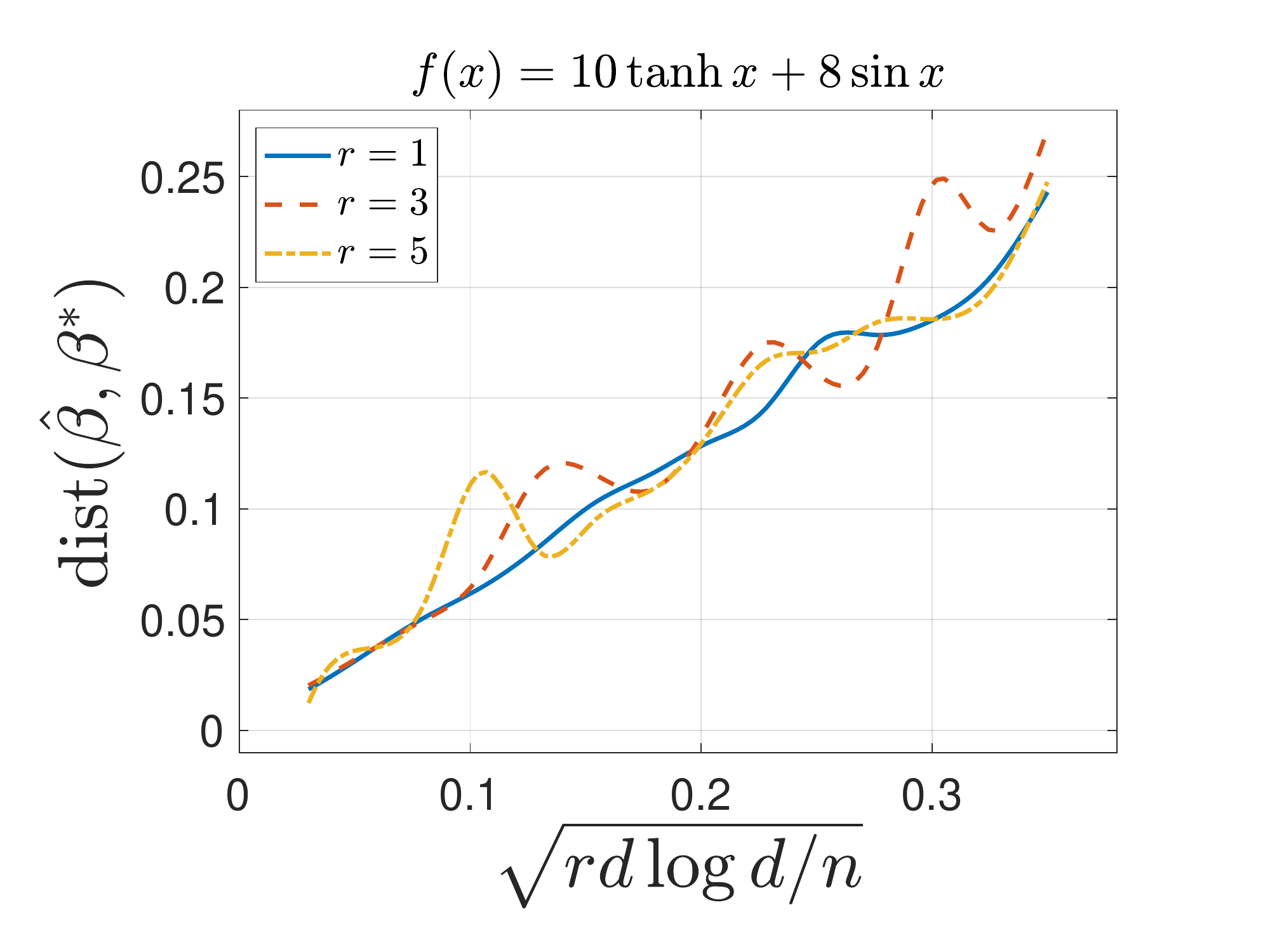}\\
		(a) &(b)
	\end{tabular}\\
	\caption{ The averaged $\ell_2$-distances between true parameter $\beta^{*}$ and estimated parameter matrices $\hat\beta$ for (a) $t(5)$ distributed covariates with link function $f_7$ and (b) Gamma$(8,0.1)$ distributed covariates with the link function $f_8$.}
	\label{t_gammamat}
\end{figure}

\section{Conclusion}
	\revise{In this paper, we leverage over-parameterization to design regularization-free algorithms for single index model and provide theoretical guarantees for the induced implicit regularization phenomenon.}
	We consider the case where the link function is unknown, the distribution of the covariates is known  as a prior,
	and the signal parameter is either a $s$-sparse vector in $\RR^p$ or a rank-$r$ matrix in $\RR^{d\times d}$.
Using the score function and the Stein's identity, we propose an over-parameterized nonlinear least-squares loss function.
To handle the possibly heavy-tailed distributions of the score functions and the response variables, we adopt additional truncation techniques that robustify the loss function.
For both the vector and matrix SIMs, we construct an estimator of the signal parameter by applying gradient descent to the proposed loss function, without any explicit regularization.
We prove that, when initialized near the origin, gradient descent with a small   stepsize finds an estimator that enjoys minimax-optimal statistical rates of convergence.
Moreover, for vector SIM with Gaussian design, we further obtain the oracle statistical rates that are independent of the ambient dimension. \revise{Furthermore, our experimental results support our theoretical findings and also demonstrate that our methods empirically outperform classical methods with explicit regularization in terms of both $\ell_2$-statistical rate and variable selection consistency.}

\newpage
\appendix{}

\section{Discussion}
\revise{In this section, we add more discussions on several key points in this paper, namely stopping time, stepsize, intialization, extension and scalability.}
\subsection{Stopping time $T_1$}
\revise{In this subsection, we discuss the behavior of our algorithm when we let $T_1\rightarrow\infty$ and the reason why we adopt early stopping. \\We prove that if we let $T_1\rightarrow\infty$ in our algorithm, we are able to achieve the vanilla sample estimator $\sum_{i=1}^{n}y_iS(\xb_i)/n$, which is the global min of loss function  \eqref{loss1bk}: \begin{align*}
		L(\wb,\vb)=\langle \wb\odot \wb-\vb\odot \vb, \wb\odot \wb-\vb\odot \vb \rangle -2\Big\langle \wb\odot \wb-\vb\odot \vb, \frac{1}{n}\sum_{i=1}^{n}y_iS(\xb_i) \Big\rangle.
	\end{align*}
	This can be demonstrated by showing that the loss function \eqref{loss1bk} does not contain local maximum or non-strict stationary points. Thus gradient descent always tends to find the global minimum \citep{gradientmin_2016} if the stepsize is small enough. Similar situation also holds under matrix case. 
	The tradeoff that we do not let $T_1$ go to infinity is because the unregularized estimator $\sum_{i=1}^{n}y_iS(\xb_i)/n$ is only consistent to $\mu^{*}\beta^{*}$ in terms of $\ell_{\infty}$-norm or operator norm. In the high-dimensional regime, the
	$\ell_2$- ($\ell_{\fro}$-) statistical rate
	of such an estimator can be diverging.  However, we aim at getting the $\ell_2$-statistical rates in order to guarantee our estimators generalize well in terms of out of sample predictions.  Thus, we adopt early stopping in our Algorithm \ref{alg1}, Algorithm \ref{alg2} to prevent overfitting and to take advantage of sparsity.}
	\subsection{Initial value $\alpha$}
	\revise{In this subsection, we discuss what will happen if we choose other intial values (Recall, we set $\alpha=\cO(1/p)$ in our paper). We only discuss the vector case, the situation for the matrix case is similar. \\In terms of other initial values,
	our algorithm works as long as the strength of perturbation parameter $\alpha$ satisfies $p\cdot \alpha^2=\cO(s/n)$. However, if we make initialization with a larger order, the noise component will be overfitted easily before $\cO(\log(1/\alpha)/ ( \eta(|\mu^{*}|s_m-M_0\sqrt{\log p/n}) ))$ steps, especially when the minimal true signal in the strong signal set $s_m$ is close to the threshold $C_s\sqrt{\log p/n}$ (This is the threshold which distinguishes the strong signal and weak signal set). In this case, the optimal stopping time does not exist, as the error component is overfitted before the signal component converges. }
	\subsection{Stepsize $\eta$}
	\revise{In this subsection, we describe the reason we use constant stepsize, and also illustrate the pros and cons of using decreasing stepsize. \\
	First, as mentioned in the first point of our discussion, we are analyzing non-asymptotic results for the iterates $\beta_t$ (with finite $t$), as $t\rightarrow \infty$ will result in an overfitted estimator. 
	Second, from the theoretical perspective, the assumptions on the size of the learning rate is only required in proving the dynamics of strong signal and weak signal components. The dynamics of noise component is able to adaptive to the stepsize with any size.
	To be more specific, for strong signals (signals in $S_0$), we prove that as long as the stepsize is smaller than some fixed constant, it will keep increasing in absolute value first, i.e. $|\beta_{t+1,i}|\ge |\beta_{t,i}|$ for all $i\in S_0$. After it converges to the area around $\mu^{*}\beta_{S_0}^{*}$ ($\|\beta_{t,S_0}-\mu^{*}\beta_{S_0}^{*}\|_2\le \sqrt{|S_0|/n}$), our constant stepsize will guarantee that it will never leave that area. In terms of the weak signal component $\beta_{t,i},i\in S_1$, we prove that if the stepsize is smaller than some fixed constant, it will never exceed the order of $\cO(\sqrt{\log p/n})$ throughout the whole iterations. Thus, we use fixed constant stepsize in this paper since it is enough to guarantee the main theoretical results. For more details, please refer to our Lemma \ref{prop1.5} and Lemma \ref{monotone_lem} in \S\ref{sectappa}. \\ 
	In terms of decreasing our stepsize while iterating, it will help enlarge our optimal time interval for the stopping time $T_1$. To be more specific, if we choose $\eta=\cO(1/t^{\alpha})$, with $t$ being the $t$-th step and $0< \alpha<1$, the optimal stopping interval will become $$\Big[\cO\big(\log(1/\alpha)^{1/(1-\alpha)}/ (|\mu^{*}|s_m)^{1/(1-\alpha)}\big),\cO\big(\log(1/\alpha)^{1/(1-\alpha)}\sqrt{n/\log p}^{1/(1-\alpha)}\big)\Big],$$ by following similar theoretical analysis. The statistical rates remain the same with our current results inside the optimal time interval. Although the length of the optimal interval increases, we need more time to let strong signal converge (need $\cO(\log(1/\alpha)^{1/(1-\alpha)}/ (|\mu^{*}|s_m)^{1/(1-\alpha)})$ steps instead of only $\cO(\log(1/\alpha)/ (|\mu^{*}|s_m))$ steps). This involves a tradeoff between  the number of iterations of the algorithm and the flexibility of choosing stopping time. In this paper, we focus on the setting  with constant stepsize. }}
\subsection{Extension}	
\revise{Our algorithm also works under a more generalized setting. To be more specific, the Algorithm \ref{alg1} is fit for the following generalized optimization problem
	\begin{align*}
		\min_{\wb,\vb}{\langle \wb\odot \wb-\vb\odot \vb,\wb\odot \wb-\vb\odot \vb\rangle}-\frac{2}{n}\bigg\langle \wb\odot\wb-\vb\odot\vb, \sum_{i=1}^{n}\ba_i^{*}\bigg\rangle
	\end{align*}
	in which we have $\ba_i,i\in [n]$ are i.i.d. with $\EE[\ba_i^{*}]=c\beta^{*},$ for all $i\in [n]$ with bounded sub-exponential norm.
	Here $\beta^{*}$ is the unknown sparse vector parameter we aim at recovering and $c$ is a non-zero constant. The situation for the matrix case is similar, our conclusion also holds for the optimization problem
	\begin{align*}
		\min_{\Wb,\Vb}{\Big\langle \Wb \Wb^\top-\Vb \Vb^\top,\Wb \Wb^\top-\Vb \Vb^\top\Big\rangle}-\frac{2}{n}\bigg\langle \Wb\Wb^\top-\Vb\Vb^\top, \sum_{i=1}^{n}\Ab_i^{*}\bigg\rangle
	\end{align*}
	whenever $\Ab_i,i\in [n]$ are i.i.d. and possess bounded spectral norm with high-probability and $\EE[\Ab_i^{*}]=c\beta^{*}$. Here $\beta^{*}$ is the unknown low rank matrix we aim at recovering. Thus, as long as these two general frameworks are satisfied, our estimators will keep the same behaviors as given in Theorem \ref{thmvec1} and  \ref{thmmat1}. The key point for aforementioned assumptions on bounded sub-exponential and operator norm with high probability is to guarantee that the true $c\beta^{*}$ lies in the high-confidence set $\{\beta:\|\nabla L(\beta)\|_{\textrm{norm}}\le \lambda_n\}$ with norm being either $\ell_{\infty}$ norm or operator norm. The $\lambda_n$ is chosen to be equivalent to order of the maximum noise strength \citep{candes2008restricted}. In terms of the heavy-tailed case, with properly winsorized tail components, we get robust estimators whose tail distributions behave like sub-exponential tail distributions. Thus, similar conclusions also hold for heavy-tailed distributions.	}
\subsection{Scalability}	\label{scalable}
\revise{In this subsection, we discuss the scalability of our method.\\
To be more specific, our methodology is more scalable than regularized methods in terms of two commonly used settings of  distributed computing, namely \emph{centralized} and \emph{decentralized} settings. \\
Under the \emph{centralized} setting, we have a central controller (parameter server), which stores parameters, and $m$ local machines, which store distributed datasets with size $n_j,j\in [m]$ respectively. For each iteration, the local machines transmit their local gradients to the central controller and central controller sends back the updated parameters to every local machine after aggregating the information. As this procedure only involves transferring gradient or parameter information, our Algorithm \ref{alg1} is applicable to this \emph{centralized} setting. Moreover, if the central machine makes an update after collecting all gradient information from all local machines, this is equivalent with running our Algorithm \ref{alg1} with the full datasets. In this case, our Theorem \ref{thmvec1}  also holds. However, in terms of the distributed regularized methods in studying the high-dimensional sparse statistical models under this \emph{cenetalized} setting, one needs to solve every regularized problem on every local machine and then averages the outputs from all local machines \citep{Lee2017,Battey2018,Jordan2019,Fan2021}  to generate the next iterate.  This will put more burden on every local machine. Moreover, the aforementioned literatures only work with $\ell_1$-regularization, the literature that studies the distributed estimation via non-convex regularizers under the \emph{centralized} setting is sparse. Since we obtain oracle statistical rate in Theorem \ref{thmvec1} by only conducting gradient descent, our method is also able to achieve the oracle rate in a distributed manner, which is equivalent with adding folded-concave regularizers on every local machine.\\
Furthermore, under the \emph{decentralized} setting, we have $m$ machines connected via a communication network \citep{MMRHA2017,Richards2020,Richards2020D}, and each machine $j \in [m]$ stores $n_j$ i.i.d. observations of the single index model in \eqref{vecSIM}.  Algorithm \ref{alg1} can be easily modified for such a  decentralized setting by letting each machine send its local  parameter or local gradient to its neighbors. See, e.g., \cite{shi2015extra, yuan2016convergence} for more details of consensus-based first-order methods.
Besides, $\beta_{T_1}$ in Theorem \ref{thmvec1} corresponds to the centralized estimator obtained by aggregating all the data across the $m$ machines. Thus Theorem \ref{thmvec1} still holds with $n = \sum_{j \in [m]} n_j$.
Then, when the communication network is sufficiently well-conditioned, we can expect that the local parameters on the $m$ machines
reaches   consensus rapidly and are all close to $\beta_{T_1}$, thus achieving optimal statistical rates.
In contrast, explicit regularization such as $\ell_1$-norm or SCAD produces an  exactly sparse solution. In the decentralized setting, imposing explicit regularization to produce a shared sparse solution with statistical accuracy seems to require novel algorithm design and analysis. }
\section{Additional Simulations and Real Data}
\revise{In this section, we provide more numerical studies on the comparisons between our methodology and classical regularized methods.}
\subsection{Comparisons with Regularized Methods}\label{comparison}\revise{
In this section, we aim at comparing the $\ell_2$-statistical rates achieved by our methodology and classical regularized methods (Lasso and SCAD). We focus on the senario where we have fixed number of observations but increasing dimensionality of the covariates. To be more specific, we fix $n=300,s=4$ and choose $p$ such that $\sqrt{s\log p/n}$ ranges uniformly from $0.25\sim 0.4$
(corresponding to $\sqrt{\log p}$ ranges uniformly from $2.15\sim 3.47$.) In terms of $\beta^{*}$, we choose its support randomly among all subsets of $\{1,\dots,p\}$ with cardinality $s$ and let $\beta_{j}^{*}=1/\sqrt{s}\cdot \textrm{Uniform}(\{-1,1\}).$ We let every entry of $\xb\in\RR^{p}$ have i.i.d. distribution, which are either standard Gaussian, Student's t-distribution with $5$ degrees of freedom. Given $\beta^{*}$ and distribution of $\xb$, we generate $n=300$ i.i.d samples $\{\xb_i,y_i\}_{i=1}^{n}$ from the vector SIM with aforementioned link functions $\{f_j\}_{j=1}^4$. As for the optimization procedure, we  let $\alpha=10^{-3}$, stepsize $\eta=0.01$ in Algorithms \ref{alg1} and \ref{alg5}. The stopping time $T_1$ is chosen by following our methodology in \S\ref{stoptime}, where we take $T_1$ which minimizes the out-of-sample prediction risk. Then we use the whole samples to conduct Algorithms \ref{alg1}, and \ref{alg5}, and return $\beta_{T_1}/\|\beta_{T_1}\|_2$ with that pre-fixed $T_1$. Note that we are also able to use $k$-fold cross-validation to choose the stopping time $T_1$, however, since the stopping time interval is wide, for simplicity, we just use out-of-sample prediction and it already offers a good $T_1$. As for the regularized method (LASSO and SCAD), we first use 5-fold cross-validation to choose the tuning parameter $\lambda$ and then use the whole dataset together with the pre-selected $\lambda$ to report $\hat\beta_{\textrm{lasso}}/\|\hat\beta_{\textrm{lasso}}\|_2$ and $\hat\beta_{\textrm{scad}}/\|\hat\beta_{\textrm{scad}}\|_2$ by minimizing the emperical quadratic  loss function with extra regularizers \eqref{reguvec}  ($\ell_1,\textrm{ or folded concave regularizer}$). We summarized the performance of every method in the following two figures. We are able to see our methodology outperforms classical regularized methods, and achieve oracle convergence rates even when we have increasing dimensionality of covariates.
\begin{figure}[H]
	\centering
	\begin{tabular}{cc}
		\hskip-30pt\includegraphics[width=0.40\textwidth]{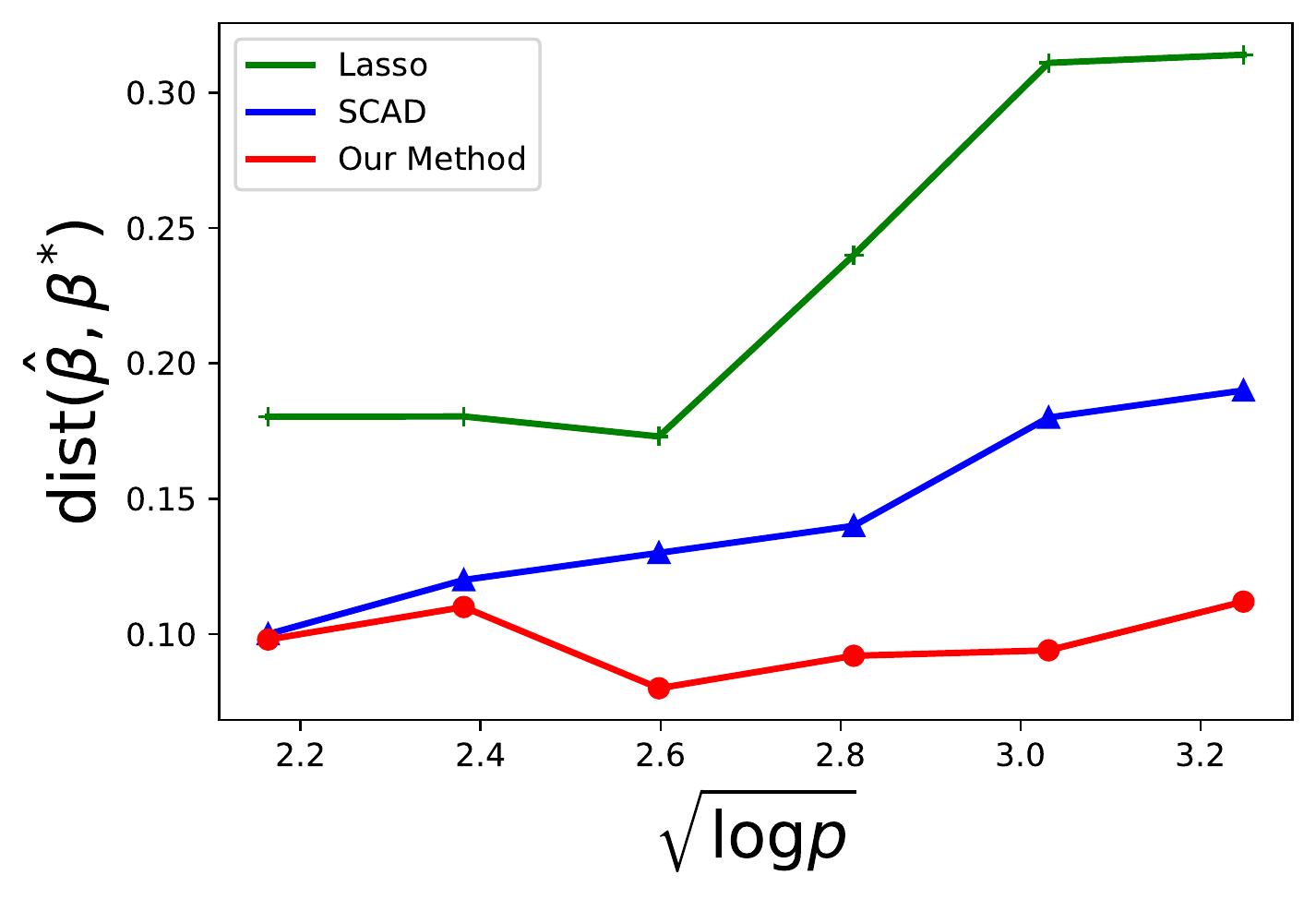}
		&\hskip-1pt\includegraphics[width=0.40\textwidth]{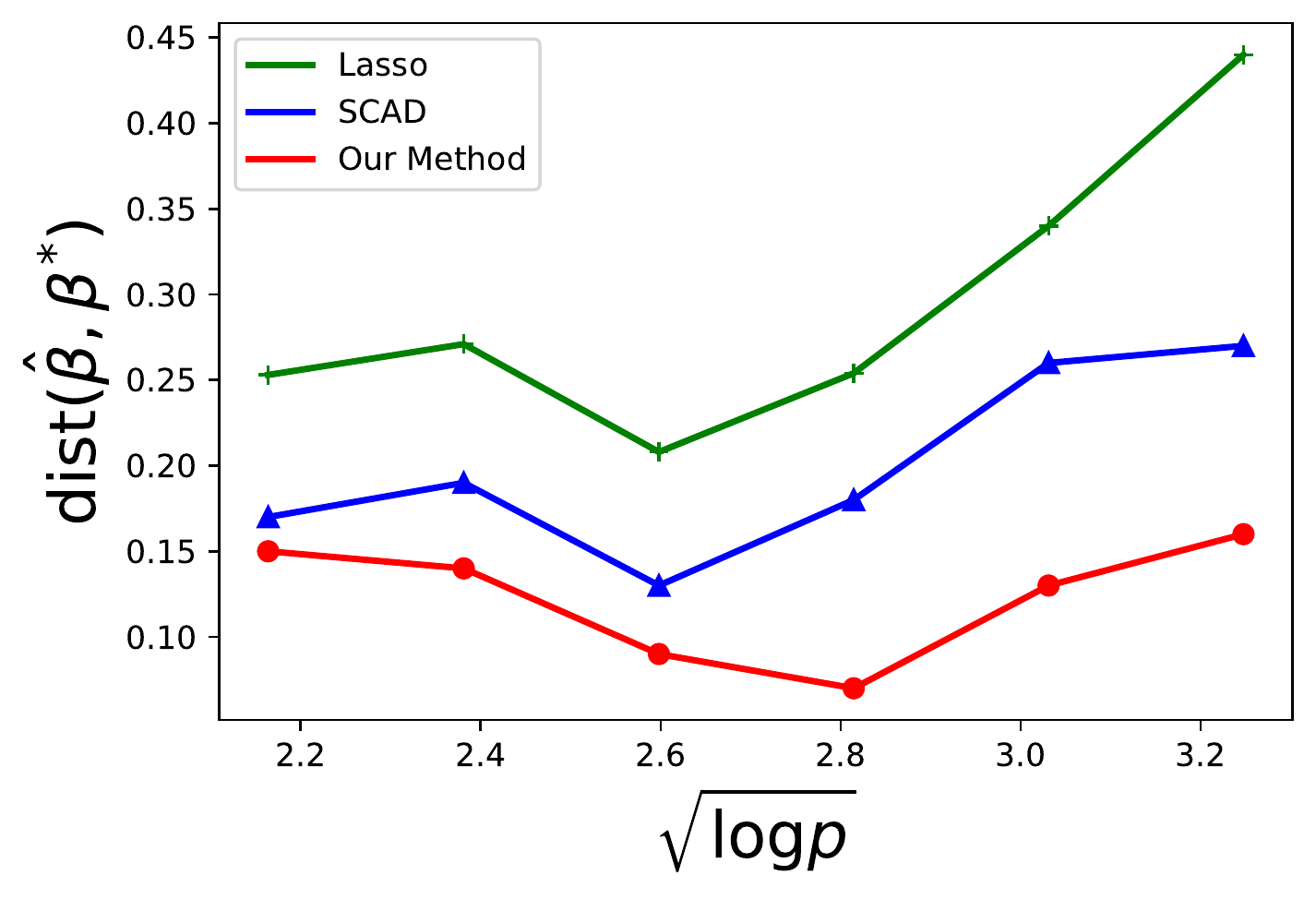}\\
		(a) &(b)
	\end{tabular}\\
	\caption{(a). $\ell_2$-convergence rates comparison between our method and Lasso, SCAD with Gaussian distributed covariate, the link function is set to be $f_1.$  (b). $\ell_2$-convergence rates comparison between our method and Lasso, SCAD with $t(5)$ distributed covariate and link function $f_3$. The truncating parameter $\tau$ is set to be $\tau=2(n/\log p)^{1/4}$ in Algorithm \ref{alg5}. In both figures, we repeat the aforementioned experiment for 50 independent trails and plot the averaged distance dist$(\hat\beta,\beta^{*})$ against the $\sqrt{\log p}$.}
	\label{compare_lasso_scad}
\end{figure}
In addition, we also compare the support recovery results achieved by our method and the regularized methods. The measures that we use to quantify the accuracy of the support recovery of a given estimator are False Discovery Rate (FDR) and  True Positive Rate (TPR). For a given estimator $\hat\beta$, they are defined as follows:
$$\textrm{FDR}=\frac{|\textrm{supp}(\hat\beta)\cap(S^{*})^{c}|}{\max\{|\textrm{supp}(\hat\beta)|,1\}},$$
$$\textrm{TPR}=\frac{|\textrm{supp}(\hat\beta)\cap S^{*} |}{|S^{*}|},$$
where $S^{*}$ denotes the true support. In terms of the experimental settings,  we fix $p=1000$, $s=\sqrt{p}$ and let $n/(s\log p)$ vary uniformly from $2.5\sim 25$. Moreover, we let the other settings be the same as the settings of the $\ell_2$-statistical rates comparison discussed above. We repeat the aforementioned experiment for 100 independent trails.  For every trail, we record the False Discovery Rate (FDR) and True Positive Rate (TPR) for our estimator $\tilde{\beta}_{T_1}$ defined in Theorem \ref{signconsist} and the regularized estimator $\beta_{\textrm{lasso}}$ and $\beta_{\textrm{scad}}$. The support recovery performances are illustrated in the following Figure \ref{compare_fdr_tpr}.
\begin{figure}[H]
	\centering
	\begin{tabular}{cc}
		\hskip-30pt\includegraphics[width=0.40\textwidth]{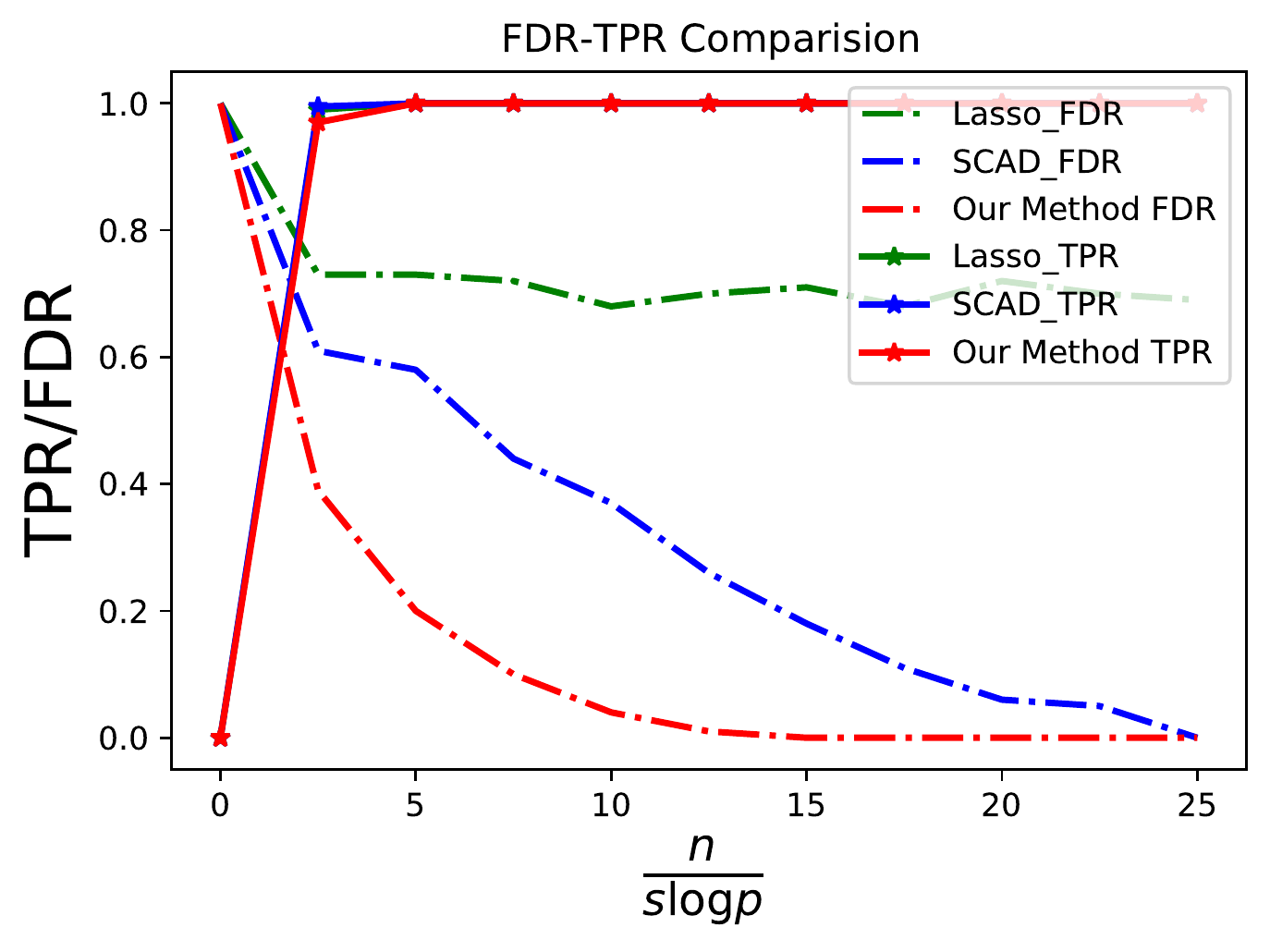}
		&\hskip-1pt\includegraphics[width=0.40\textwidth]{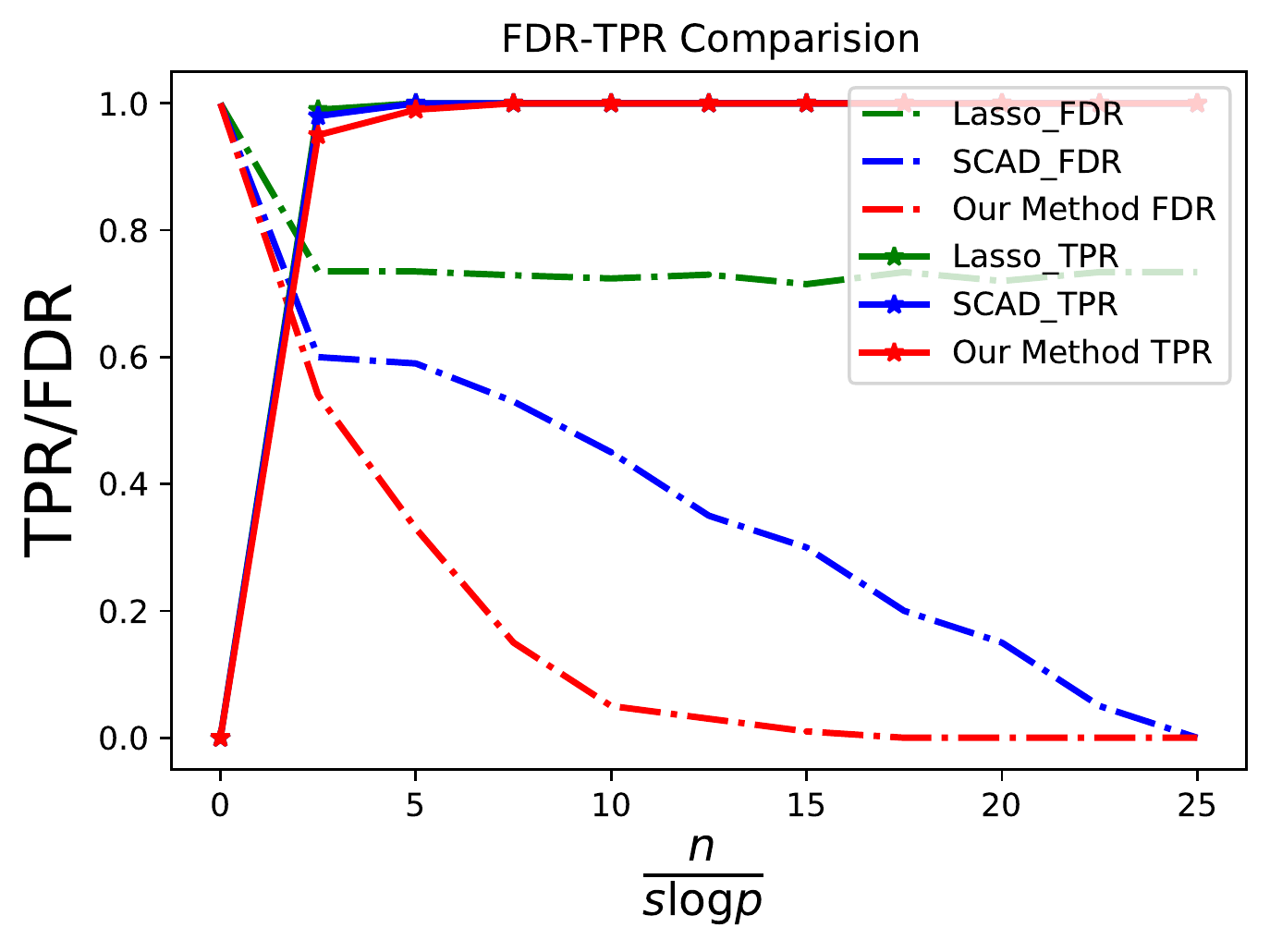}\\
		(a) &(b)
	\end{tabular}\\
	\caption{Support recovery comparison between our method and regularized methods. Figure (a) illustrates the results with standard Gaussian design and link function $f_2$. Figure (b) shows the results with $t(5)$ distributed covariate with link function $f_2$. We let the truncating parameter $\lambda=5\alpha$ given in Theorem \ref{signconsist} to construct $\tilde{\beta}_{T_1}$ in both cases. The tuning parameter $\lambda$ for Lasso or SCAD is selected via 5-fold cross-validation.}
	\label{compare_fdr_tpr}
\end{figure}
We tell from Figure \ref{compare_fdr_tpr}, our method achieves comparable TPR with the regularized methods and at the same has much lower FDR. The results above illustrate the robustness and efficiency of our methodology over the classical regularized methods in terms of support recovery.}
\revise{
\subsection{Application to Real Data}\label{real_data}
One important application of our methodology is image processing via compressed sensing especially under nonlinear links \citep{candes2008restricted,plan2013onebit,plan2012robust,goldstein2018structured,goldstein2019non}. In the following, we  extend our methodology to real-world data, where we consider the example of one-bit compressed sensing with sparse image recovery \citep{Jacques2013onebit,plan2013onebit}.
To be more specific, the response variables and the covariates collected by us satisfy
\begin{align*}
	y_i=\text{sign}(\langle \mathbf{x}_i, \beta^{*}\rangle)+\epsilon,\qquad \forall i\in[n],
\end{align*}
where $\text{sign}(x)=1$ for all  $x\ge 0$ and  $\text{sign}(x)=-1,\text{ for } x<0$,  and $n$ is the number of our observations. We summarized its corresponding theoretical results in Appendix \S\ref{secexamp}.\\ We let $\Mb\in \RR^{H\times W}$ be a sparse image, in which $H,W$ denote the high and width of the given matrix. We vectorize the matrix $\Mb$, and denote the new vector as $\beta^{*}$.
The original image given in Figure \ref{true_figure} is a image for stars with $H=375$ by $W=500$. After vectorizing this matrix we have a $\beta^{*}\in \RR^{187500}$. Due to the size of the image, we decompose the vector into $L=146$ disjoint parts, with  $p=1290$ for each part. To be more clear, we have $\textrm{vec}(\Mb)=(\Mb_1,\Mb_2,\dots,\Mb_L)$, with $\Mb_{\ell}\in \RR^{p}$. We denote the sparsity of $\Mb_{\ell}$ as $s_{\ell}$. For every $\ell\in [L]$, we set the link function as $\textrm{sign}(\cdot)$ and sample $n_l=5\cdot s_l\log p$ observations using standard Gaussian covariate. We then run Algorithm \ref{alg3} given in Appendix \S\ref{secexamp}  with initial value $\alpha=0.001$ and stepsize $\eta=0.1$ to get the estimator $\hat\beta_{\ell}$ of $\beta_{\ell}^{*}/\|\beta_{\ell}^{*}\|_2.$ Since we only obtain the sign information, we are only able to recover the direction of $\beta^{*}$, and without loss of generality, we assume we know the length of $\|\beta_{\ell}^{*}\|_2$ beforehand. Thus,  our finally estimator for $\beta_{\ell}^{*}$ is $\hat\beta_{\ell}\cdot \|\beta_{\ell}^{*}\|_2$.
\begin{figure}[H]
	\centering
	\begin{tabular}{cc}
		\hskip-30pt\includegraphics[width=0.45\textwidth]{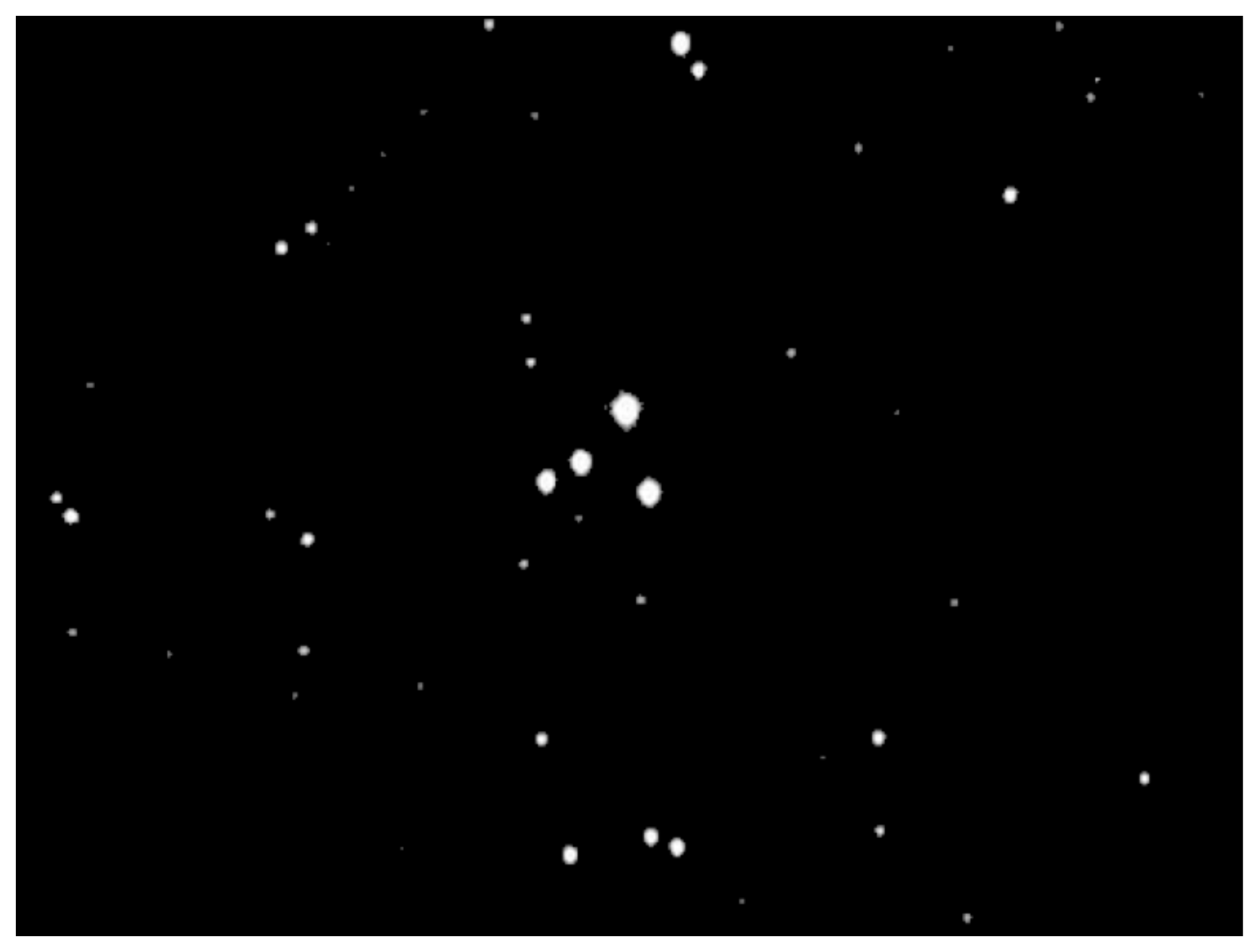}\\
		(a)
	\end{tabular}\\
	\caption{The true figure.}
	\label{true_figure}
\end{figure}
The image recovered by our method and Lasso method are shown in the following Figure \ref{lasso_over} respectively. And the error to the true image for both methods are shown in Figure \ref{error_overp_lasso}.  The cumulative errors to the true figure in $\ell_1$-norm are $449.27$ and $533.89$ respectively.  Observe that in terms of the error, our methodology nearly recovered the true image and outperforms the Lasso method.
\begin{figure}[H]
	\centering
	\begin{tabular}{cc}
		\hskip-30pt\includegraphics[width=0.45\textwidth]{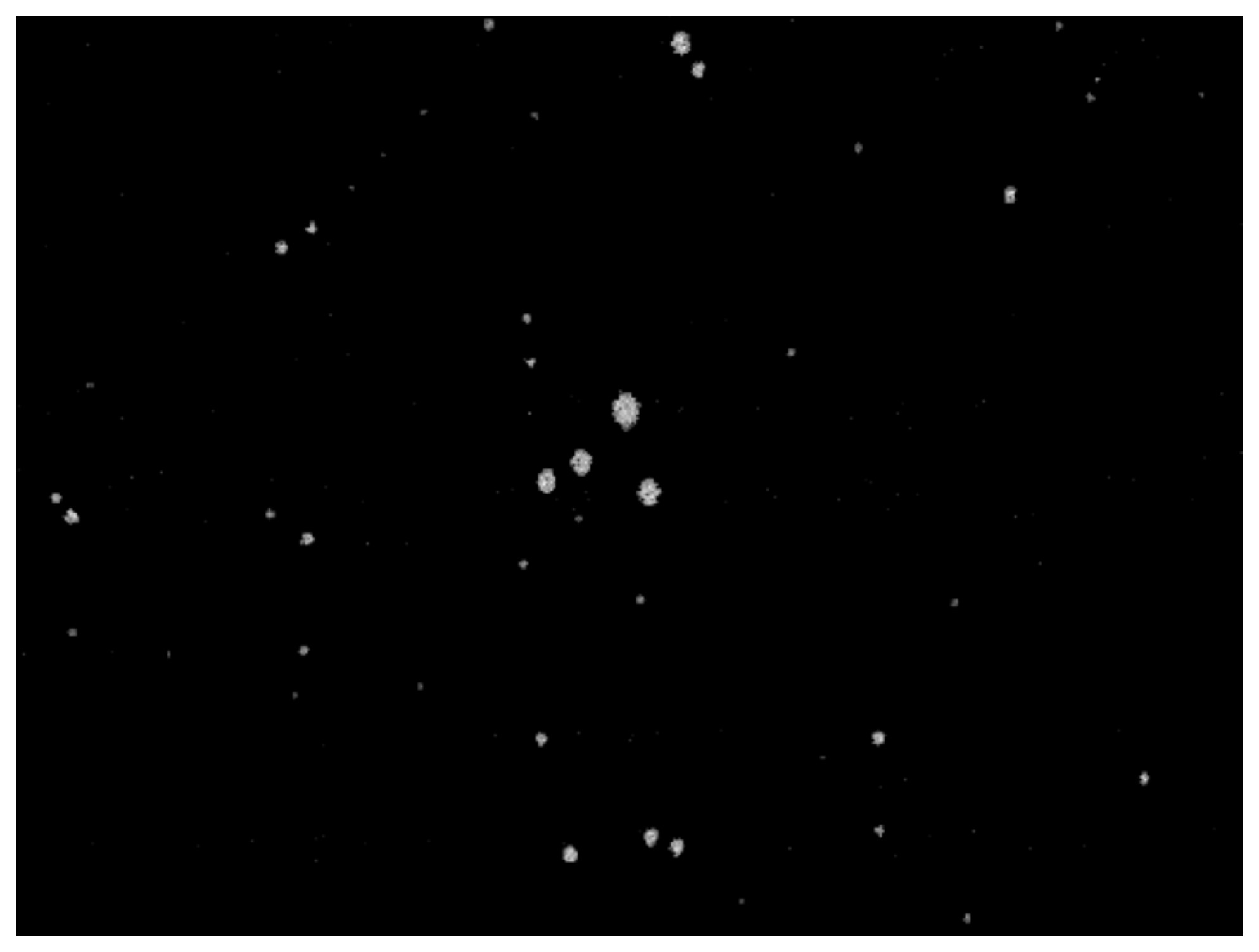}
		&\hskip-1pt\includegraphics[width=0.45\textwidth]{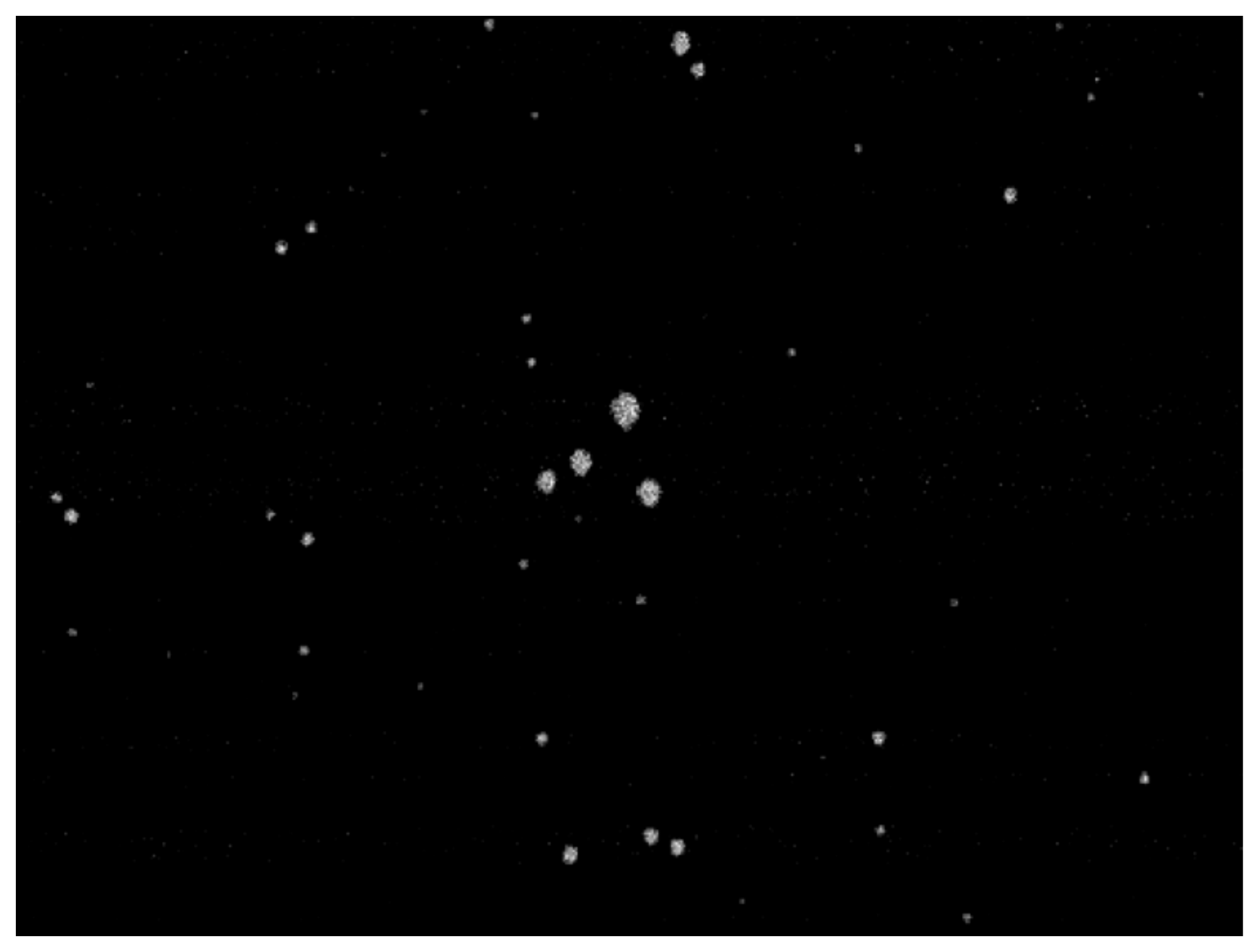}\\
		(a) &(b)
	\end{tabular}\\
	\caption{(a). Figure recoverd by our proposed method. (b) Figure recoverd by Lasso.  The tuning parameter is selected via $5$-fold cross validation.}
	\label{lasso_over}
\end{figure}
\begin{figure}[H]
	\centering
	\begin{tabular}{cc}
		\hskip-30pt\includegraphics[width=0.45\textwidth]{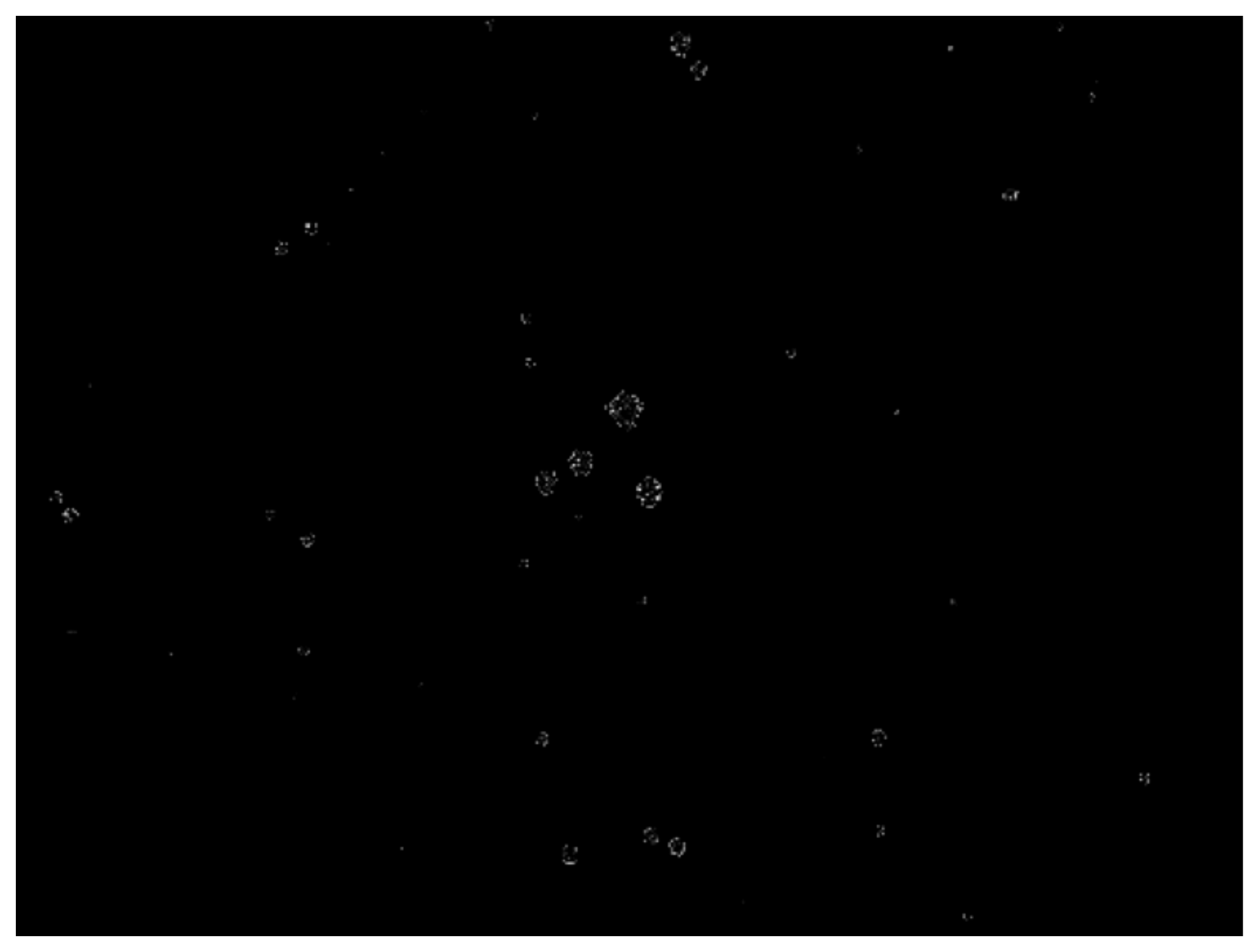}
		&\hskip-1pt\includegraphics[width=0.45\textwidth]{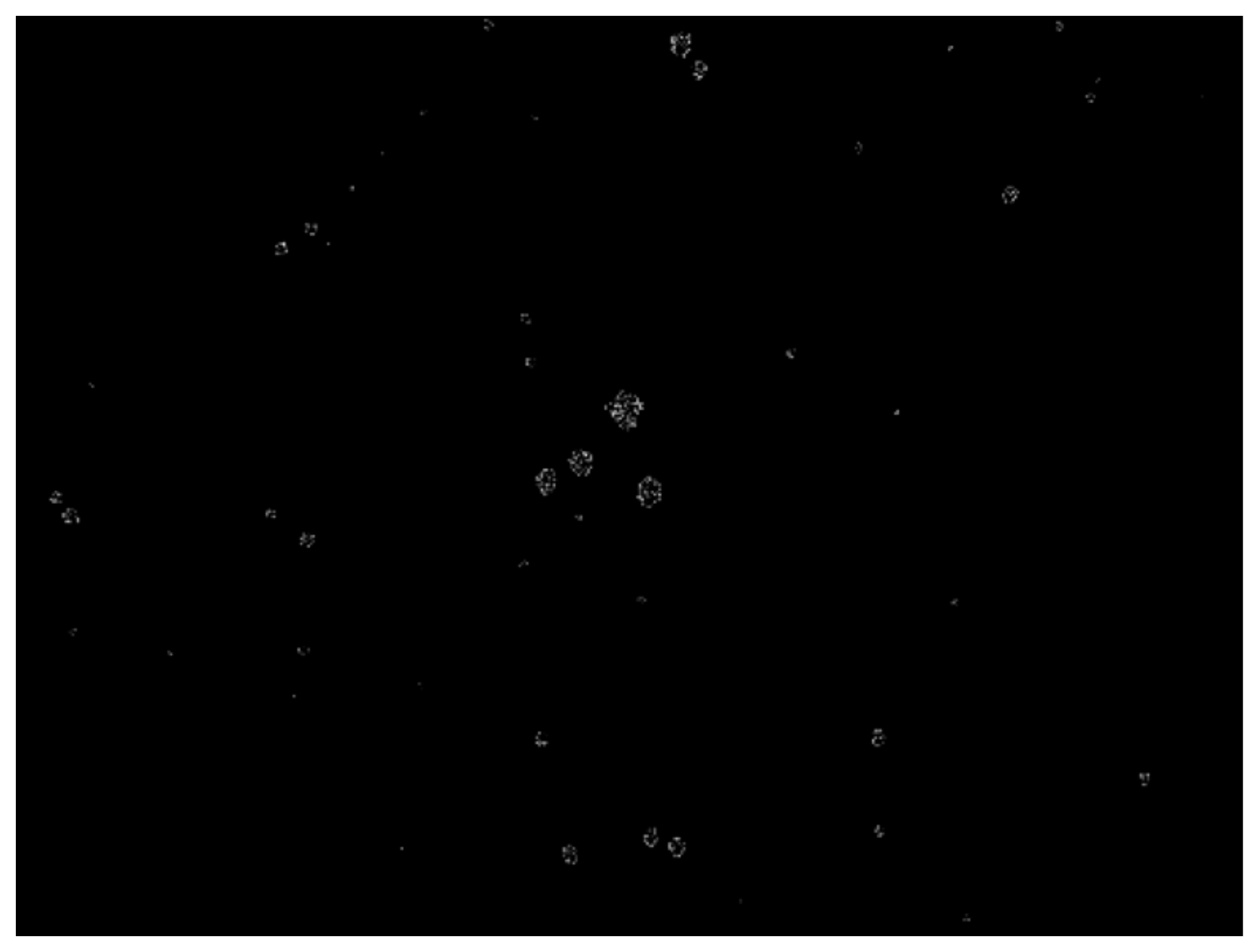}\\
		(a) &(b)
	\end{tabular}\\
	\caption{(a). Error of the recovered figure by our method. (b) Error of the recovered figure by Lasso.  }
	\label{error_overp_lasso}
\end{figure}
}

\section{Proofs of  Theoretical Results in \S\ref{sectionvec}} \label{sectappa}

In this section, we prove the results presented in \S\ref{sectionvec}. Specifically, we prove Theorem \ref{thmvec1} in \S\ref{gensig}  and present the proofs of Theorems \ref{thmpred} and \ref{thmvec2}  in \S\ref{MSEproof} and \S\ref{secheavy}, respectively.

\subsection{Proof of Theorem \ref{thmvec1}}\label{gensig}

In the following, we
  over-parameterize $\mu^{*}\beta^{*}$ by writing it as $\mathbf{w}\odot \mathbf{w}-\mathbf{v}\odot \mathbf{v}$, where  $\mathbf{w}$ and $\mathbf{v}$ are vectors with size $p\times 1$. Then we apply  gradient descent to the following optimization problem,
    	\begin{align}\label{loss2}
\min_{\mathbf{w},\mathbf{v}}L(\wb,\vb)=\langle \wb\odot \wb-\vb\odot \vb, \wb\odot \wb-\vb\odot \vb \rangle -2\Big\langle \wb\odot \wb-\vb\odot \vb, \frac{1}{n}\sum_{i=1}^{n}y_iS(\xb_i) \Big\rangle.
\end{align}
 The  gradient descent updates  with respect  to $\mathbf{w}$ and $\mathbf{v}$ are given by
 	\begin{align}\label{generalupd1}
\mathbf{w}_{t+1}&=\mathbf{w}_t-\eta\big (\mathbf{w}_t\odot\mathbf{w}_t-\mathbf{v}_t\odot \mathbf{v}_t-\Phi_n\big)\odot \mathbf{w}_t,\\
\mathbf{v}_{t+1}&=\mathbf{v}_t+\eta\big (\mathbf{w}_t\odot \mathbf{w}_t-\mathbf{v}_t\odot \mathbf{v}_t-\Phi_n\big)\odot \mathbf{v}_t.\label{generalupd2}
 \end{align}
  We first  remind readers of the notations. We divide the entries of $\beta^{*}$ into three different groups in terms of their strengths. 
The support set $S$ of the  signal  is defined as $S =\{i:\left|\beta_i^{*}\right|>0   \}$.
The set $S_0$ which contains the strong signals is defined as   $S_0 =\{i: |\beta_i^{*}|\ge C_s\sqrt{\log p/n} \}$ with $C_s$ being an absolute  constant that will be specified later in the proof.
In addition,
 we define $S_1$ as  $S_1 =\{i: 0<|\beta^{*}_i|< C_s\sqrt{ \log p/n }  \}$, which contains all indices of the weak signals. By such a construction, we have $S = S_0 \cup S_1$. Moreover, the  complement of $S$, denoted by  $S^c$, corresponds to the pure error part.
 Furthermore, the
  pure error parts of $\mathbf{w}_t$ and $\mathbf{v}_t$ are denoted by $\mathbf{e}_{1,t}:=\mathbf{1}_{S^{c}}\odot\mathbf{w}_t$ and $\mathbf{e}_{2,t}:=\mathbf{1}_{S^{c}}\odot\mathbf{v}_t$ respectively. In addition, strong signal parts of $\mathbf{w}_t$ and $\mathbf{v}_t$ are denoted by $\mathbf{s}_{1,t}=\mathbf{1}_{S_0}\odot \mathbf{w}_t$ and $\mathbf{s}_{2,t}=\mathbf{1}_{S_0}\odot\mathbf{v}_t,$ meanwhile, weak signals parts are written as $\mathbf{u}_{1,t}:=\mathbf{1}_{S_1}\odot\mathbf{w}_t$ and $\mathbf{u}_{2,t}:=\mathbf{1}_{S_1}\odot \mathbf{v}_t$. Here, we denote $\sqrt{n/\log p}$ by $\gamma$ and $\frac{1}{n}\sum_{i=1}^{n}S(\xb_i)y_i$ by $\Phi_n$, and let $s_0$ and $s_1$ be the size of set $S_0$ and $S_1$ respectively.

\begin{proof} [Proof of Theorem \ref{thmvec1}]
We prove Theorem \ref{thmvec1}  by analyzing the dynamics of the pure error, strong signal and weak signal components separately.
We first utilize the following lemma to upper bound the pure error parts $\mathbf{e}_{1,t}$ and $\mathbf{e}_{2,t}$.

 \begin{lemma}\label{prop1.3}
 	Under the assumptions in Theorem~\ref{thmvec1}, with probability $1-2p^{-1}$, there exists a constant $a_2=1/6$ such that
 	\begin{align}
 		&\left\|\mathbf{e}_{1,t}\right\|_\infty\leq \sqrt{\alpha} \le \frac{M_0}{\sqrt{p}},\qquad
 		\left\|\mathbf{e}_{2,t}\right\|_{\infty}\leq \sqrt{\alpha} \le \frac{M_0}{\sqrt{p}}\label{up_noise}, 
 	\end{align}
hold for all $t\le T=a_2\log(1/\alpha)\gamma/(\eta M_0)$,  where $M_0$  is an absolute constant that is  proportional to  the sub-Gaussian  constants $\max\{\sigma,\|f\|_{\psi_2}\}$.
 \end{lemma}

\begin{proof}
	See \S\ref{genererr} for a detailed proof.
	\end{proof}

By  Lemma \ref{prop1.3}, there exists a constant $a_2$ such that we are able to control the pure error part at the same level with $\sqrt{\alpha}$ within the time horizon $0\le t\le T=a_2\log(1/\alpha)\sqrt{n/\log p}/(\eta M_0)$, where $M_0$ is an absolute constant.
Thus, by direct computation, we have
\#\label{eq:noise_2_norm}
\| {\bf 1}_{S^c} \odot ( \beta_t - \mu^* \beta^*) \|_2 ^2 = p \cdot \bigl \| \mathbf{e}_{1,t} \odot \mathbf{e}_{1,t} - \mathbf{e}_{2,t} \odot \mathbf{e}_{2 ,t} \bigr \|_{\infty } ^2  \leq 2M_0^4 / p,
\#
where $S^c$ is the complement of $S$.

As for the signal parts, recall that
we separate $S$ into $S_0$ and $S_1$, which corresponds to the supports of the strong and weak signals, respectively.
The following lemma characterizes the entrywise convergence of strong signal component $\beta_t\odot \mathbf{1}_{S_0}$.

 \begin{lemma}\label{prop1.4}
  	Under the assumptions given in Theorem~\ref{thmvec1},
  	when the absolute constant $C_s$ satisfies $C_s\ge 2M_0/|\mu^{*}|$,
  	 if we further choose $0<\eta \le{1}/[12(|\mu^{*}|+M_0)]$, we have
  	\begin{align}\label{up_signal}
  		\big\|\beta_{t}\odot\mathbf{1}_{S_0}-\mu^{*}\beta^{*}\odot\mathbf{1}_{S_0}\big\|_{\infty}\le 2M_0\sqrt{  \log n / n} 
  	\end{align}
  	holds with probability $1-2n^{-1}$ for all $t\ge a_1/[\eta(|\mu^{*}|s_m-M_0\sqrt{\log p/n})] \cdot \log(1/\alpha),$ where $a_1=21$ and $s_m = \min_{i \in S_0} |\beta^{*}_i| $.
  \end{lemma}
\begin{proof}
	See \S\ref{genersig} for a detailed proof.
	\end{proof}

By  \eqref{up_signal},  with probability $1-2n^{-1}$,  we obtain
\begin{align}\label{sigconcen1}
\left\|\beta_t\odot\mathbf{1}_{S_0}-\mu^{*}\beta^{*}\odot\mathbf{1}_{S_0}\right\|_2^2\le M_1 \cdot s_0\log n / n ,
\end{align}
for  all $t\ge  a_1/[\eta(|\mu^{*}|s_m-M_0\sqrt{\log p/n})] \cdot \log(1/\alpha)$, where we denote  $M_1=4M_0^2$.

Note that Lemma \ref{prop1.4} requires that  $C_s \ge 2 M_0 / | \mu^*|$, where both $M_0$ and $| \mu^*|$ are absolute constants.
Furthermore, we would like to have
\eqref{up_noise} and \eqref{up_signal} hold simultaneously.
To this end,
it suffices to choose $C_s$ such that
\#\label{eq:require_T}
a_1/[\eta   (|\mu^{*}|s_m-M_0\sqrt{\log p/n})] \cdot \log(1/\alpha) \leq a_2\log(1/\alpha) \cdot \gamma/(\eta M_0),
\#
where we let $\gamma$ denote $\sqrt{ n / \log p}$.
Since $s_m \geq C_s \sqrt{\log p/n}$,
we have
\$
 \eta \big (|\mu^{*}|s_m-M_0\sqrt{\log p/n} \big ) \geq ( | \mu^* | \cdot C_s - M_0) \cdot \sqrt{\log p/n} ff.
\$
Thus, \eqref{eq:require_T} holds as long as $C_s  \ge(a_1/a_2+1)M_0/|\mu^{*}|.$
In other words, by choosing $C_s$ to be a proper absolute constant,
there exists an interval such that for all $t$ belongs to such an interval, Lemmas \ref{prop1.3} and \ref{prop1.4} holds simultaneously.

Next, we have the following lemma for the  dynamics of weak signal component $\beta_t\odot \mathbf{1}_{S_1}$.

 \begin{lemma}\label{prop1.5}
	Under the assumptions given in Theorem~\ref{thmvec1}, if we further choose $0<\eta \le{1}/[12(|\mu^{*}|+M_0)]$, then with probability at least $1 - 2p^{-1}$, we have
	\begin{align}\label{up_weak}
		\big\|\beta_{t}\odot\mathbf{1}_{S_1}-\mu^{*}\beta^{*}\odot\mathbf{1}_{S_1}\big\|_{\infty}\le (M_0+2C_s|\mu^{*}|) \cdot \sqrt{\log p / n} 
	\end{align}
	  for all $t\ge 0$.
\end{lemma}
\begin{proof}
See \S\ref{pf_weak} for a detailed proof.
\end{proof}


Finally, combining \eqref{eq:noise_2_norm}, \eqref{sigconcen1}, and  \eqref{up_weak},
with probability at least
$1 - 2n^{-1} - 2p^{-1}$,
for any $T_1$ belongs to the interval
\#\label{eq:interval}
\big[a_1\log(1/\alpha)/[\eta(|\mu^{*}|s_m-M_0\sqrt{\log p/n})],~ a_2\log(1/\alpha)\gamma/(\eta M_0) \big ],
\#
it holds that
\begin{align} \label{eq:rate_final}
	\left\|\beta_{T_1}-\mu^{*}\beta^{*}\right\|^2_2\le M_1\cdot \frac{s_0\log n}{n}+M_2 \cdot \frac{s_1\log p}{n}+ \frac{2M_0^4}{p} ,
\end{align}
where we define
  $M_2=(M_0+2C_s |\mu^{*} | )^2$.
  Since $p$ is much larger than $n$, the last term, $2M_0^4 / p$, is negligible.

Finally, it remains to
establish the  $\ell_2$-statistical  rate for the normalized   iterates.	
Note that  $\mu^*=\mathbb{E}[f'(\langle \mathbf{x},\beta^{*}\rangle)]$ is an absolute  constant.
 Without loss of generality, we assume $\mu^{*}>0$.
 Also recall that $\| \Sigma^{1/2} \beta^*\|_2  = 1$ and $\Sigma$ satisfies Assumption \ref{ass1}-(a).
 Thus,
 when $n$ is sufficiently large,
 with probability at least $1 - 2n^{-1} - 2p^{-1}$,
 for any  $T_1$ in the interval in \eqref{eq:interval}, we have
\begin{align*}
	\|\Sigma^{1/2}\beta_{T_1}\|_2& \ge\mu^{*} -\|\Sigma^{1/2}(\beta_{T_1}-\mu^{*}\beta^{*})\|_2
	\geq \mu^{*} - \| \Sigma^{1/2} \|_{\oper} \cdot \| \beta _{T_1} - \mu^* \beta ^* \|_2
	\notag \\
	& \ge \mu^{*}-\sqrt{C_{\max}}\cdot\sqrt{M_1\cdot  s_0\log n / n +M_2\cdot s_1\log p / n+2 M_0^4 / p}   \ge \frac{\mu^{*}}{2}.
\end{align*}
 Thus, by direct computation, we further obtain
\begin{align*}
	\left\|\frac{\beta_{T_1}}{\|\Sigma^{1/2}\beta_{T_1}\|_2}-\beta^{*}\right\|_2^2&=\frac{ \bigl\|\beta_{T_1}-  \|\Sigma^{1/2}\beta_{T_1}\|_2\cdot\beta^{*} \bigr \|_2^2 }{\|\Sigma^{1/2}\beta_{T_1}\|_2^2}\\&\le\frac{2\|\beta_{T_1}-\mu^{*}\beta^{*}\|_2^2+2\bigl\|\mu^{*}\beta^{*}- \|\Sigma^{1/2}\beta_{T_1}\|_2 \beta^{*} \bigr \|_2^2}{\|\Sigma^{1/2}\beta_{T_1}\|_2^2}\\
	&\le\frac{4}{\mu^{*2}} \cdot \|\beta_{T_1}-\mu^{*}\beta^{*}\|_2^2+\frac{4C_{\max}}{\mu^{*2}C_{\min}} \cdot \|\mu^{*}\beta^{*}-\beta_{T_1}\|_2^2\\&\le M_3\frac{s_0\log n}{n}+M_4\frac{s_1\log p}{n}+\frac{c_1}{p},
\end{align*}
where we define  $M_3=4M_1(1+C_{\max}/C_{\min})/\mu^{*2}$, $M_4=4M_2(1+C_{\max}/C_{\min})/\mu^{*2}$, and $c_1=8M_0^4(1+C_{\max}/C_{\min})/\mu^{*2},$
which are all absolute constants.
Here the second inequality follows from Assumption \ref{ass1}-(a).
Therefore, we conclude the proof
  of Theorem~\ref{thmvec1}.
\end{proof}
Next, in the following three subsections, namely  \S\ref{genererr}, \S\ref{genersig} and \S\ref{pf_weak}, we prove  Lemma \ref{prop1.3}, Lemma \ref{prop1.4} and Lemma \ref{prop1.5}, respectively.
\revise{
\subsection{Proof pf Theorem \ref{signconsist}}\label{pf_sign_vec}
\begin{proof}
	The proof of Theorem \ref{signconsist} follows directly follows from our results given in Lemma \ref{prop1.3}, Lemma \ref{prop1.4}, Lemma \ref{prop1.5}, so we omit the corresponding details.
	\end{proof}
}

\subsubsection{ Proof of Lemma \ref{prop1.3}} \label{genererr}
\begin{proof}
	 Here we prove  Lemma \ref{prop1.3} by induction hypothesis. It holds that our initializations $\|\eb_{1,0}\|_{\infty},\|\eb_{2,0}\|_{\infty}$ satisfy our conclusion given in Lemma \ref{prop1.3}.   As we initialize $\eb_{1,0}$ and $\eb_{2,0}$ with $\|\eb_{1,0}\|_{\infty}=\alpha\le \sqrt{\alpha}\le M_0/\sqrt{p}$ and $\eb_{2,0}$ with $\|\eb_{2,0}\|_{\infty}=\alpha\le \sqrt{\alpha}\le M_0/\sqrt{p}$,  Lemma \ref{prop1.3} holds when $t=0$. Next, for any $t^{*}$ with $0\le t^{*}< T=a_2\log(1/\alpha)\gamma/(\eta M_0)$, if the conclusion of Lemma \ref{prop1.3} holds for all  $t$ with $0\le t\le t^{*}$, we need to verify that it also holds at step $t^{*}+1$. 

	
	      From our gradient descent updates  given in \eqref{generalupd1}-\eqref{generalupd2}, the updates with respect to pure error parts $\eb_{1,t}$, $\eb_{2,t},\,t\ge 0$ are obtained as follows
	\begin{align}
	&\mathbf{e}_{1,t+1}=\mathbf{e}_{1,t}-\eta\big(\beta_t-\Phi_n\big)\odot \mathbf{e}_{1,t},\label{perr1}
	&\mathbf{e}_{2,t+1}=\mathbf{e}_{2,t}+\eta\big(\beta_t-\Phi_n\big)\odot \mathbf{e}_{2,t}.
	\end{align}
  As $\eb_{\ell,t}\odot\EE[\Phi_n]=\eb_{\ell,t}\odot\mu^{*}\beta^{*}=\mathbf{0}$ holds for any $l\in \{1,2\}$ by the definition of $\eb_{\ell,t}$, the following inequality always holds according to  $\beta_t\odot \mathbf{1}_{S^c}=\eb_{1,t}\odot \eb_{1,t}-\eb_{2,t}\odot\eb_{2,t}$ and the triangle inequality,
			\begin{align}				
	\left\|\mathbf{e}_{l,t+1}\right\|_{\infty}&\le\big[1+\eta\big(\left\|\mathbf{e}_{1,t}\odot \mathbf{e}_{1,t}\right\|_{\infty}+\left\|\mathbf{e}_{2,t}\odot \mathbf{e}_{2,t}\right\|_{\infty}  +\|\Phi_n-\EE[\Phi_n]\|_{\infty}\big)\big]\cdot\left\|\mathbf{e}_{l,t}\right\|_{\infty}.\label{egup1}
\end{align}
	According to our induction hypothesis, for all $l \in\{1,2\}$, we are able to bound $\|\mathbf{e}_{l,t}\|_{\infty},$  at the same order with $\sqrt{\alpha}$, when $t\le t^{*}$. Thus, we can replace $\|\mathbf{e}_{l,t}\odot \mathbf{e}_{l,t}\|_{\infty}$ by  $\alpha$ in \eqref{egup1}.
	Then
	we apply  the following lemma to obtain an upper bound on $\|\Phi_n-\EE[\Phi_n]\|_{\infty}$.

	\begin{lemma} \label{lembernstein}
		Under the assumptions given in Theorem \ref{thmvec1},  with probability $1-2p^{-1}$, we obtain
		\begin{align*}
			\big\|\Phi_n-\EE[\Phi_n] \big\|_{\infty}\le M_0\sqrt{\frac{\log p}{n}},
		\end{align*}  where  $M_0$ is an absolute  constant that  is proportional to $\max\{\|f\|_{\psi_2},\sigma\}$.
	\end{lemma}
	\begin{proof}
		See \S\ref{sect:lembernstein} for a detailed proof.
	\end{proof}
  Combining \eqref{egup1}, the induction hypothesis,  and  Lemma \ref{lembernstein},	we further obtain that
	\begin{align*}
		\left\|\mathbf{e}_{l,t+1}\right\|_{\infty}&\le\left[1+\eta\left(2\alpha+M_0\sqrt{{\log p}/{n}}\right)\right]\cdot\left\|\mathbf{e}_{l,t}\right\|_{\infty}
	\end{align*}
 holds with probability $1-2p^{-1}$ for all $t\le t^*$ and $l\in\{1,2\}$.
 By our assumption on $\alpha$ in Theorem \ref{thmvec1}, we obtain
 \begin{align*}
 	2\alpha=\frac{2M_0^2}{p} \le 2M_0\sqrt{\frac{\log p}{n}},
 \end{align*}
  since we have $M_0=\cO(1)$ and $M_0/p\le \sqrt{\log p/n}$ when $p\ge M_0\sqrt{n/\log p}$.
Now we define an absolute constant  $c_1'$ as $c_1'=1/(3M_0)$ and denote  $a_2=1/6$.
Let $T$ be defined as  $T=a_2\log(1/\alpha)\gamma/(\eta M_0)$.
By direct computation,  we then have
  \begin{align*}
 	 \left\|\mathbf{e}_{l,t^*+1}\right\|_{\infty}&\le\big[1+3\eta M_0\sqrt{\log p/n}\big]^{t^*+1}\cdot \left\|\mathbf{e}_{l,0}\right\|_{\infty}\\&= \big[1+\eta/(c_1'\gamma)\big]^{t^*+1}\cdot \left\|\mathbf{e}_{l,0}\right\|_{\infty}
 	\le\exp\big(T \cdot \log\big(1+{\eta}/{(c_1'\gamma)}\big)\big)\cdot\alpha\\&\le \exp(T\cdot \eta/(c_1'\gamma))\cdot\alpha=\exp((1/2)\cdot\log(1/\alpha))\cdot \alpha= \sqrt{\alpha}\le M_0/\sqrt{p},
 \end{align*}
for any $l\in\{1,2\}$.
  Here we denote
   $\gamma=\sqrt{n/\log p}$,
 the second inequality follows from $t^* + 1 \leq T$, and the third inequality follows from $\log (1+ x) <  x$ for all $x >0$.
   Thus, our induction hypothesis also holds for $t^{*}+1$. In addition, as $t^{*} < T$ is arbitrarily chosen, we conclude the proof of  Lemma \ref{prop1.3}.
\end{proof}
\subsubsection{Proof of Lemma \ref{prop1.4}} \label{genersig}
\begin{proof}
	Following  \eqref{generalupd1} and \eqref{generalupd2}, the dynamics of $\beta_{t}^{(1)}:=\mathbf{s}_{1,t}\odot \mathbf{s}_{1,t}=\mathbf{1}_{S_0}\odot \mathbf{w}_{t+1}\odot \mathbf{w}_{t+1},$  $\beta_{t}^{(2)}:=\mathbf{s}_{2,t}\odot \mathbf{s}_{2,t}=\mathbf{1}_{S_0}\odot \mathbf{v}_{t+1}\odot \mathbf{v}_{t+1}$ and $\beta_{t,S_0}:=\mathbf{1}_{S_0}\odot \beta_t$ are obtained as
	\begin{align}\label{upu_i2}
	\beta_{t+1}^{(1)}&=\big[\mathbf{1}-\eta\big(\beta_{t,S_0}-\mu^{*}\beta^{*}\odot\mathbf{1}_{S_0}+\mu^{*}\beta^{*}\odot\mathbf{1}_{S_0}-\Phi_n\odot\mathbf{1}_{S_0}\big)\big]^2\odot \beta_{t}^{(1)},\\\label{upv_i2}
		\quad\beta_{t+1}^{(2)}&=\big[\mathbf{1}+\eta\big(\beta_{t,S_0}-\mu^{*}\beta^{*}\odot\mathbf{1}_{S_0}+\mu^{*}\beta^{*}\odot\mathbf{1}_{S_0}-\Phi_n\odot\mathbf{1}_{S_0}\big)\big]^2\odot \beta_{t}^{(2)},\\
	   \beta_{t+1,S_0}&=\beta_{t+1}^{(1)}-\beta_{t+1}^{(2)}.\nonumber
	\end{align}
We denote the $i$-th entry of $  \Phi_n- \mu^{*}\beta^{*}$ as $\xi_i$ for all $i \in [p]$.
Applying the same proof as in Lemma \ref{lembernstein},
with probability at least $1-2n^{-1}$, we have
  $|\xi_i|\le M_0\sqrt{\log n/n}$   for all $ i\in S$.
  To simplify the notation, for any $i \in S_0$,
  we define $\beta_{i}^{'}:=\mu^{*}\beta_i^{*}+\xi_i$.
  Under the assumption that $ C_{s} \geq 2 M_0 / | \mu^*|$,
  for any $i \in S_0$,
  we have
  $$
  |\beta_i' | \geq | \mu^{*}\beta_i^{*}| - |\xi_i | \geq  C_s \cdot | \mu^*| \cdot \sqrt{\log p /n}  - M_0 \sqrt{\log n/n} \geq M_0\sqrt{\log p /n} >0.
  $$
Without loss of generality, here we just analyze entries $i\in S_0$ with $\beta_{i}^{'}:=\mu^{*}\beta_i^{*}+\xi_i> 0$.
 Analysis for  the case with $\beta_{i}^{'}<0,i\in S_0$ is almost the same and is thus omitted.  We  divide our analysis into three steps.

 In specific, in \textbf{Step I}, as we initialize $\beta_{0,i}=0$ for any fixed $i\in S_0$, we prove that after $4\log(\beta_i^{'}/\alpha^2)/(\eta\beta_i^{'})$ iterations, we have $\beta_{t,i}\ge \beta_{i}^{'}/2.$
Next, in \textbf{Step~II}, we quantify the number of iterations required for achieving $\beta_{t, i } \geq \beta_i' - \epsilon$, where $\epsilon=M_0\sqrt{\log n/n}$.
Finally, in {\bf Step III}, we show that, when $t$ is sufficiently large and the stepsize $\eta$ is sufficiently small,  $\beta_{t,i}$ will always stay in the interval
$[\beta_{i}^{'}-\epsilon,\beta_{i}^{'}]$, which enables us to conclude the proof.

\vspace{5pt}
{\noindent \textbf{Step I.}} For any fixed $\,\,i\in S_0,$, when we have $0\le\beta_{t,i}\leq {(\mu^{*}\beta^{*}_i+\xi_i)}/{2},$  by \eqref{upu_i2} and \eqref{upv_i2},     we   get geometric increment of $\beta_{t,i}^{(1)}$ and decrement of $\beta_{t,i}^{(2)}$ respectively
	\begin{align*}
	\beta_{t+1,i}^{(1)}=w_{t+1,i}^2 \ge \left[1+\frac{\eta(\mu^{*}\beta^{*}_i+\xi_i)}{2}\right]^2\cdot w_{t,i}^2, \qquad
	\beta_{t+1,i}^{(2)}=v_{t+1,i}^2 \le \left[1-\frac{\eta(\mu^{*}\beta_i^{*}+\xi_i)}{2}\right]^2\cdot v_{t,i}^2.
\end{align*}
This first stage ends when our $\beta_{t,i}$ exceeds $(\mu^{*}\beta^{*}_i +\xi_i)/2$, our goal in \textbf{Step I} is to  estimate    $t_{i,0}$ that satisfies
		\begin{align*}
\beta_{t,i}\ge \bigg[1+\frac{\eta(\mu^{*}\beta_i^{*}+\xi_i)}{2}\bigg]^{2t_{i,0}} \cdot \alpha^2-\left[1-\frac{\eta(\mu^{*}\beta_i^{*}+\xi_i)}{2}\right]^{2t_{i,0}} \cdot \alpha^2\ge \frac{\mu^{*}\beta^{*}_i+\xi_i}{2}.
	\end{align*}
	That is, $t_{i,0}$ is the time when $\beta_{t, i}$ first exceeds $(\mu^* \beta^* + \xi_i) / 2$.
It could be hard for us to pinpoint  $t_{i,0}$  exactly; instead, we find a sufficient condition for $t_{i,0}$, i.e. when $t\ge t_{i,0}$, we must have $\beta_{t,i}\ge(\mu^{*}\beta^{*}_i+\xi_i)/{2}.$
Observe that it is sufficient to solve the following inequality for $t_{i,0}$,
\begin{align*}
\left[1+\frac{\eta(\mu^{*}\beta_i^{*}+\xi_i)}{2}\right]^{2t_{i,0}} \cdot \alpha^2\ge\frac{\mu^{*}\beta^{*}+\xi_i}{2}+\alpha^2,
\end{align*}
which is equivalent to finding $t_{i,0}$ satisfying
\begin{align*}
t_{i,0}\ge T_{i,0}:=\frac{1}{2}\log\bigg(\frac{\mu^{*}\beta_i ^{*}+\xi_i}{2\alpha^2}+1\bigg)\Big{/}\log\bigg(1+\frac{\eta(\mu^{*}\beta_i ^{*}+\xi_i)}{2}\bigg).
\end{align*}
In the following, we use $\beta_i^{'}$ to represent $\mu^{*}\beta_i^{*}+\xi_i$ for simplicity and we obtain
\begin{align*}
T_{i,0} & =\frac{1}{2}\log\bigg(\frac{\beta_i^{'}}{2\alpha^2}+1\bigg)\Big{/}\log\bigg(1+\frac{\eta\beta_i^{'}}{2}\bigg)\\& \le 2\log\bigg(\frac{\beta_i^{'}}{\alpha^2}\bigg)\cdot\bigg(1+\frac{\eta\beta_i^{'}}{2}\bigg)\Big{/}\big(\eta\beta_i^{'}\big) \le 4\log\bigg(\frac{\beta_i^{'}}{\alpha^2}\bigg)\Big{/}\big(\eta\beta_i^{'}\big),
\end{align*}
where   the first inequality follows from $x\log(x)-x+1\ge 0,$ when $x\ge 0$ as well as our assumption on $\alpha$, and the second inequality holds due to the assumption  on $\eta$.
Thus, we set $t_{i,0}=4\log(\beta_i^{'}/{\alpha^2})/\eta\beta_i^{'}$ such that for all $t\ge t_{i,0}$ we get
$\beta_{t,i}\ge \beta_i^{'}/{2}$ for all  $i\in S_0$ and we complete the \textbf{Step I}.

\vspace{5pt}
{\noindent \textbf{Step II.}}  Recall that we denote $\beta_i' = \mu^* \beta_i^* + \xi _i$ for notational simplicity.
Besides, we define
$\epsilon=M_0\sqrt{\log n/n}$
and
$m_{i,1} = \lceil\log_2({\beta_i^{'}}/{\epsilon})\rceil$ for all $i \in S_0$.
For any fixed $i\in S_0$, if there exists some $m$ with  $1\le m\le  m_{i,1}$ such that   $(1-1/{2^{m}})\beta_i^{'}\le\beta_{t,i}\le (1-1/{2^{m+1}})\beta_i^{'}$,      according to  \eqref{upu_i2} and \eqref{upv_i2}, we obtain
 \begin{align*}
 w_{t+1,i}^2 \ge\bigg(1+\frac{\eta\beta_i^{'}}{2^{m+1}}\bigg)^2\cdot w_{t,i}^2, \qquad
 v_{t+1,i}^2 \le\bigg(1-\frac{\eta\beta_i^{'}}{2^{m+1}}\bigg)^2\cdot v_{t,i}^2.
 \end{align*}
 Note that {\bf Step I} shows that
 $\beta_{t,i} \geq (1- 1/2^m) \beta_i'$ for $m = 1$ when $t \geq t_{i, 0}$.
We define $t_{i, m}$ as the smallest $t$ such that
$\beta_{t + t_{i,m}, i} \geq (1- 1/2^{m+1})\beta_i'$.
Intuitively, given that the current iterate $\beta_{t,i}$ is sandwiched by $(1- 1/2^{m })\beta_i'$ and $(1- 1/2^{m+1})\beta_i'$,
$t_{i,m}$ characterizes the number of iterations required for the sequence $\{ \beta_{t, i}\}_{t\geq 0}$ to exceed $(1- 1/2^{m+1})\beta_i'$.

In the sequel,   we   first  aim to obtain a sufficient condition  for $t_{i,m}$. By   construction, $t_{i,m}$ satisfies
 \begin{align*}
 &\beta_{t+t_{i,m},i}\ge w_{t,i}^2\bigg(1+ \frac{\eta\beta_i^{'}}{2^{m+1}}\bigg)^{2t_{i,m}}-\bigg(1-\frac{\eta\beta_i^{'}}{2^{m+1}}\bigg)^{2t_{i,m}}v_{t,i}^2\ge\bigg(1-\frac{1}{2^{m+1}}\bigg)\beta_i^{'}.
 \end{align*}
We assume for now that $v_{t,i} \leq \alpha$.
As we will show in {\bf Step III},
 we have
$\beta_{t,i} \leq \beta_i'$ for all $t $.
Then, by \eqref{upv_i2},
  $\{ v_{t,i}^2\}_{t\geq 0} $ forms  a  decreasing sequence and thus is bounded by $\alpha^2$.
Hence,  it suffices to find a $t_{i,m}$ that satisfies
 \begin{align*}
t_{i,m}\ge T_{i,m}:=\frac{1}{2}\log\bigg(\frac{(1-1/2^{m+1})\beta_i^{'}+\alpha^2}{w_{t,i}^2}\bigg)\Big{/}\log\bigg(1+\frac{\eta\beta_i^{'}}{2^{m+1}}\bigg).
 \end{align*}
To get an upper bound of $T_{i,m}$, under the  assumption that $w_{t,i}^2\ge \beta_{t,i}\ge(1-{1}/{2^m})\beta_{i}^{'},$ we obtain
 \begin{align*}
 T_{i,m}&=\frac{1}{2}\log\bigg(\frac{(1-1/2^{m+1})\beta_i^{'}+\alpha^2}{w_{t,i}^2}\bigg)\Big{/}\log\bigg(1+\frac{\eta \beta_i^{'}}{2^{m+1}}\bigg)\\&\le\frac{1}{2}\log\bigg(\frac{(1-1/2^{m+1})\beta_i^{'}+\alpha^2}{(1-1/2^m)\beta_i^{'}}\bigg)\Big{/}\bigg(\frac{\eta\beta_i^{'}/2^{m+1}}{1+\eta\beta_i^{'}/2^{m+1}}\bigg)\\
 &=\frac{1}{2}\log\bigg(1+\frac{1/2^{m+1}}{1-1/2^{m}}+\frac{\alpha^2}{(1-1/2^m)\beta_i^{'}}\bigg)\cdot\bigg(1+\eta\frac{1}{2^{m+1}}\beta_i^{'}\bigg)\Big{/}\bigg(\frac{\eta\beta_i^{'}}{2^{m+1}}\bigg),
 \end{align*}
where the second inequality follows from $x\log(x)-x+1\ge 0,$ when $x\ge 0.$ By direct calculation, we further get
\begin{align}
T_{i,m}&\le\bigg(\frac{1/2^{m+1}}{1-1/2^{m}}+\frac{\alpha^2}{(1-1/2^m)\beta_i^{'}}\bigg)\Big{/}\bigg(\frac{\eta\beta_i^{'}}{2^{m+1}}\bigg) 
 \le \frac{2}{\eta\beta_i^{'}}+\frac{2^{m+2}\alpha^2}{\eta\beta_i^{'2}}.\label{eq:sigup1}
\end{align}
Recall that we assume $m \leq m_{i,1} = \lceil \log_2(\beta_i^{'}/\epsilon)\rceil$ with $\epsilon=M_0\sqrt{\log n/n}$.
Then we have $2^{m+2}\le 4{\beta_i^{'}}/{\epsilon}\le  \sqrt{{n}/{\log n}}\beta_i^{'}/M_0$.
Meanwhile, under the  assumption on the initial value $\alpha^2$ stated in Theorem \ref{thmvec1}, we have $\alpha^2\le {M_0^4}/{p^2}$.  Then we   bound $2^{m+2}{\alpha^2}/({\eta\beta_i^{'2}})$ given in \eqref{eq:sigup1} as
 \begin{align}\label{eq:sigup2}
  \frac{2^{m+2}\alpha^2}{\eta\beta_i^{'2}}\le \sqrt{\frac{n}{\log n}}\cdot\frac{M_0^3}{p^2\eta\beta_i^{'}}\le\frac{1}{\eta\beta_i^{'}},
 \end{align}
 where the second inequality holds when  $p\ge (nM_0^6/\log n)^{1/4}.$
Combining   \eqref{eq:sigup1} and \eqref{eq:sigup2},  we finally bound $T_{i,m}$ as
\begin{align*}
 T_{i,m}\le \frac{3}{\eta\beta_i^{'}},
 \end{align*}
for any $m\le m_{i,1}.$
Thus, if there exists an $m \leq m_{i,1}$ such that $\beta_{t,i} \in [ (1- 1/2^{m}) \beta_i',(1- 1/2^{m+1}) \beta_i' ]$,
when $t_{i,m}\ge 3/\eta\beta_i^{'}$, we have $\beta_{t+t_{i,m},i}\ge(1-{1}/{2^{m+1}})\beta_i^{'}$.

Furthermore,
by the definition of $m_{i,1}$, we have ${\beta_i^{'}}/{2^{m_1}}\le\epsilon$, where   $\epsilon= M_0\sqrt{{\log n}/{n}}$.
Therefore, with at most $ \sum_{m=0}^{m_{i,1}}  T_{i,m}$ iterations, we have $\beta_{t,i}\ge\beta_{i}^{'}-\epsilon$.
By the definition of $\alpha$ and $\epsilon$,
we have
\begin{align}\label{timax}
	\sum_{m=0}^{m_{i,1}}  T_{i,m} \leq 4\log\bigg(\frac{\beta_i^{'}}{\alpha^2}\bigg)\Big{/}\big(\eta\beta_i^{'}\big)+3\bigg\lceil\log_2\bigg(\frac{\beta_i^{'}}{\epsilon}\bigg)\bigg\rceil\Big{/}\big(\eta\beta_i^{'}\big) \leq
	 7\log\bigg(\frac{\beta_i^{'}}{\alpha^2}\bigg)\big{/}\big(\eta\beta_i^{'}\big).
	\end{align}
	Thus, when $t \geq 7\log ({\beta_i^{'}}/ {\alpha^2}  ) / (\eta\beta_i^{'} )$,
	we have $\beta_{t, i} \geq \beta_i' - \epsilon$.
	
	Now we conclude {\bf Step II}.
	It remains to characterize the
	   dynamics of $\beta_{t,i}$ when $\beta_{t,i}\ge (1-1/2^{m})\beta_i'$ with $m>m_1=\lceil\log_2({\beta_i^{'}}/{\epsilon})\rceil$, which is handled in the last step of the proof.


\vspace{5pt}
{\noindent \textbf{Step III.}}
In the following lemma,
we prove that, when the stepsize $\eta$ is sufficiently small,  for any $i\in S$,
absolute value of $\beta_{t,i}$  keeps increasing as $t$ grows but never exceed $|\beta_i'|$.

  \begin{lemma}\label{monotone_lem}
When  the stepsize $\eta $ satisfies  $\eta\le{1}/[12(|\mu^{*}|+M_0)]$, we have $|\beta_{t+1,i}|\ge|\beta_{t,i}|$ and $|\beta_{t,i}|\le|\beta_i^{'}|$  for all $t\ge 0$ and $i\in S$.
\end{lemma}

\begin{proof}
	See \S\ref{pf:monot} for the detailed proof.
	\end{proof}

 By Lemma \ref{monotone_lem} and \eqref{timax}, by setting $\eta \leq {1}/[12(|\mu^{*}|+M_0)]$,
 for any $t \geq 7 \log ( \beta_{i} ' / \alpha^2 ) /  (\eta \beta_{i} ' ) $,
 we have
 $\beta_i ' - \epsilon \leq \beta_{t, i} \leq \beta _i'$,
 where we denote  $\epsilon= M_0\sqrt{{\log n}/{n}}$.
 Meanwhile,
 recall that $\beta_i ' = \mu^* \beta_i^* + \xi_i$ for all $i \in [p]$,
 where $| \xi_i | \leq M_2 \sqrt{\log p /n}$ by Lemma \ref{lembernstein}.
 Thus, by the construction of $S_0$,
 for any $i \in S_0$,
 we have $|\mu^{*}|s_m-M_0\sqrt{\log p/n}\le |\beta_{i}'|\le |\mu^{*}|+M_0$.
 Hence,
 for all $i \in S_0$, with probability at least $1 - 2n^{-1}$,  as long as
 \$
 t \geq 7/ \big( \eta \cdot (|\mu^{*}|s_m-M_0\sqrt{\log p/n}) \bigr )   \cdot  \log\big ( {(|\mu^{*}|+M_0)}/{\alpha^2} \bigr ),
 \$
 we have
 \begin{align*}
 0< \mu^{*}\beta_i^{*}-2\epsilon\le \beta_{t,i}\le\mu^{*}\beta_{i}^{*}+\epsilon,
 \end{align*}
 where $\epsilon= M_0\sqrt{{\log n}/{n}}$.
 Here, $\mu^{*}\beta_i^{*}-2\epsilon >0$ follows from the construction of $C_s$.
 Then by the
 definition of $\alpha$, when $p\ge M_0^2(|\mu^{*}|+M_0)$, we have $\alpha\le 1/(|\mu^{*}|+M_0)$ and
 $$ \log({(|\mu^{*}|+M_0)}/{\alpha^2})\le 3\log(1/\alpha).$$
 Then for all $t\ge 21/ [
 \eta(|\mu^{*}|s_m-M_0\sqrt{\log p/n}) ] \cdot  \log(1/\alpha) $, we have
 $ |\beta_{t, i} - \mu^*  \beta^*_i| \leq 2 \epsilon$ for all $i \in S_0$.
Thus, we conclude that,  with probability at least $1-2n^{-1}$, for all
$t$ satisfying
$$t\geq  21/ [
\eta(|\mu^{*}|s_m-M_0\sqrt{\log p/n}) ] \cdot  \log(1/\alpha), $$ we have
 \begin{align*}
 \left\|\beta_t\odot\mathbf{1}_{S_0}-\mu^{*}\beta^{*}\odot \mathbf{1}_{S_0}\right\|_{\infty}\le 2M_0\sqrt{ {\log n}/{n}}.
 \end{align*}
 Therefore, we  conclude the proof of  Lemma \ref{prop1.4}.
\end{proof}

\subsubsection{Proof of Lemma \ref{prop1.5}}\label{pf_weak}
\begin{proof}
The conclusion in Lemma \ref{prop1.5} follows directly from our conclusion in Lemma \ref{monotone_lem}, where we prove
\begin{align*}
\|\beta_t\odot \mathbf{1}_{S_1}\|_{\infty}\le\max_{i\in S_1}|\beta_i'|\le (C_s|\mu^{*}|+M_0) \cdot \sqrt{\frac{\log p}{n}}.
\end{align*}
Then combining the definition of $S_1:=\{ i: |\beta_i^{*}|\le C_s\sqrt{\log p/n}\}$, we finish the proof of  Lemma \ref{prop1.5}.
\end{proof}
\subsubsection{Proof of Lemma \ref{lembernstein}}\label{sect:lembernstein}
\begin{proof}
	By our definition of $\Phi_n$, we have
	\begin{align*}
		\big\|\Phi_n-\EE[\Phi_n] \big\|_{\infty}&=\bigg\|\frac{1}{n}\sum_{i=1}^{n}S(\mathbf{x}_i)y_i-\frac{1}{n}\sum_{i=1}^{n}\mathbb{E}\left[S(\mathbf{x}_i)y_i\right]\bigg\|_{\infty}\\&\le \bigg\|\frac{1}{n}\sum_{i=1}^{n}S(\mathbf{x}_i)f(\mathbf{x}_i^\mathsf{T}\beta^{*})-\frac{1}{n}\sum_{i=1}^{n}\mathbb{E}\big[S(\mathbf{x}_i)f(\mathbf{x}_i^\mathsf{T}\beta^{*})\big]\bigg\|_{\infty}+\bigg\|\frac{1}{n}\sum_{i=1}^{n}S(\mathbf{x}_i)\epsilon_i\bigg\|_{\infty}.
	\end{align*}
	For simplicity, we denote $f(\mathbf{x}_j^{T}\beta^{*})$ as $f_j,j\in[n]$ and the $i$-th row of $\frac{1}{n} \sum_{i=1}^{n}S(\mathbf{x}_i)f_i-\frac{1}{n}\sum_{i=1}^{n}\EE[S(\mathbf{x}_i)f_i]$ as $W_i$, $i\in[p].$ Then we get the expression of $W_i$ as
	\begin{align*}
		W_i=\frac{S(\xb_1)_if_1+\cdots+S(\xb_n)_if_n}{n}-\mathbb{E}\bigg[\frac{S(\xb_1)_if_1+\cdots+S(\xb_n)_if_n}{n}\bigg],
	\end{align*}
	which can be regarded as a concentration of $n$ i.i.d. sub-exponential variables with sub-exponential norm $$\|S(\xb_1)_if_1-\mathbb{E}[S(\xb_1)_if_1]\|_{\psi_1}\le\sup_i\|S(\xb_1)_i\|_{\psi_2}\|f_1\|_{\psi_2}:=\|f_1\|_{\psi_2}/\sqrt{C_{\min}}:=K.$$ After applying Bernstein inequality given in Corollary 2.8.3 of \cite{vershynin_2018}, we have
	\begin{align}\label{bensteineq}	
		\mathbb{P}\Big(\max_{i\in[p]}|W_i|\ge t\Big)\le 2\exp\Big(-c\min\left\{{t^2}/{K^2},{t}/{K}\right\}\cdot n+\log p\Big),
	\end{align}
	in which $c$ is a universal constant. We further set $t=K\sqrt{2\log p/(cn)}$ in  \eqref{bensteineq}, then we claim
	\begin{align*}
		\max_{i\in[p]}{\left|W_i\right|}\le K\sqrt{ \frac{2\log p}{cn}}
	\end{align*}
	holds with probability $1-p^{-1}$. Similarly, we also get $\|n^{-1} \sum_{i=1}^{n}S(\xb_i)\epsilon_i\|_{\infty}\le \sigma\sqrt{2\log p/(cn)}/\sqrt{C_{\min}},$ with probability $1-p^{-1}$.
	After denoting $2\sqrt{2/c} \cdot \max\{\|f\|_{\psi_2},\sigma\}/\sqrt{C_{\min}}$ as $M_0$, we  obtain that
	\begin{align*}
		\big\|\Phi_n-\EE[\Phi_n]\big\|_{\infty}\le M_0\sqrt{\frac{\log p}{n}}
	\end{align*}
	holds with probability $1-2p^{-1}$. Thus, we conclude the proof   of Lemma \ref{lembernstein}.
\end{proof}
\subsubsection{Proof of Lemma \ref{monotone_lem}}\label{pf:monot}
\begin{proof}
	Without loss of generality, we assume $\beta_{i}^{'} >0$.   First, we prove that $\beta_{t,i}\le\beta_{t+1,i}$ holds for all $t\ge 0$ and $i\in S$ when $0\le\beta_{t,i}<\beta_i^{'}$.
	
	For any fixed $i\in S$, there exists an $m_2\ge 0$ such that we have $(1-{1}/{2^{m_2}})\beta_i^{'}\le\beta_{t,i}\le(1-{1}/{2^{m_2+1}})\beta_{i}^{'}.$ Then by \eqref{upu_i2} and \eqref{upv_i2} we obtain a lower bound of $\beta_{t+1,i}$ as
	\begin{align*}
		\beta_{t+1,i}&=w_{t+1,i}^2-v_{t+1,i}^2\ge\bigg(1+\frac{\eta\beta_{t,i}^{'}}{2^{m_2+1}}\bigg)^2w_{t,i}^2-\bigg(1-\frac{\eta\beta_{t,i}^{'}}{2^{m_2+1}}\bigg)^2v_{t,i}^2\\
		&=w_{t,i}^2+\frac{\eta\beta_{t,i}^{'}w_{t,i}^2}{2^{m_2}}+\frac{\eta^2\beta_{t,i}^{'2}w_{t,i}^2}{2^{m_2+2}}-v_{t,i}^2+\frac{\eta\beta_{i}^{'}v_{t,i}^2}{2^{m_2}}-\frac{\eta^2\beta_{i}^{'2}v_{t,i}^2}{2^{m_2+2}} \ge w_{t,i}^2-v_{t,i}^2=\beta_{t,i}.
	\end{align*}
	
	We get $\eta{\beta_{i}^{'}}v_{t,i}^2/{2^{m_2}}-\eta^2{\beta_{i}^{'2}}v_{t,i}^2/{2^{m_2+2}} \ge 0$ as long as $\eta\beta_i'\le 4$. In addition, we know $|\beta_i'|\le |\mu^{*}\beta_i^{*}|+M_0$ for all $i\in S$ and $M_0$ given in Lemma \ref{lembernstein}. Moreover, we obtain $\max_{i\in S}|\beta^{*}_i|\le \|\beta^{*}\|_2=1$, so if we set $\eta\le 4/(|\mu^{*}|+M_0)$, we then have the inequality hold  for all $i\in S$.
	
 For the second part of Lemma \ref{monotone_lem}, we prove  $\beta_{t,i}\le \beta_i^{'}$ for all $t\ge 0,\,i\in S$ by induction. First, for any fixed $i\in S$  we know $0=\beta_{0,i}<\beta_i'$ and we assume $\beta_{t',i}\le\beta_i^{'}$ for all $0\le t'\le t$. Then, we will verify this conclusion also holds for step $t+1$. Without loss of generality, for the $t$-th iterate $\beta_{t,i}$, we assume that
	\begin{align}\label{ineqbeta}
		\bigg(1-\frac{1}{2^{m_2}}\bigg) \beta_i^{'}\le\beta_{t,i}\le\bigg(1-\frac{1}{2^{m_2+1}}\bigg)\beta_{i}^{'},
	\end{align}
	holds for some $m_2\ge 0$. According to equation $\beta_{t,i}=w_{t,i}^2-v_{t,i}^2$ and  \eqref{ineqbeta},  we further have
	\begin{align*}
		\left(1-\frac{1}{2^{m_2}}\right)\beta_i^{'}\le w_{t,i}^2\le\left(1-\frac{1}{2^{m_2+1}}\right)\beta_{i}^{'}+v_{t,i}^2.
	\end{align*}

	Following updates of $w_{t,i}^2$ and $v_{t,i}^2$ given in \eqref{upu_i2} and \eqref{upv_i2}, we obtain an upper bound of $w_{t+1,i}^2$ as well as a lower bound of $v_{t+1,i}^2$ as
	\begin{align*}
		w_{t+1,i}^2&\le\bigg(1+ \frac{\eta\beta_i^{'}}{2^{m_2}}\bigg)^2\cdot w_{t,i}^2\le\bigg(1+\frac{\eta \beta_i^{'}}{2^{m_2-1}}+\frac{\eta^2\beta_i^{'2}}{2^{2m_2}}\bigg)\cdot\left[\left(1-\frac{1}{2^{m_2+1}}\right)\beta_i^{'}+v_{t,i}^2\right],\\
		v_{t+1,i}^2&\ge\bigg(1-\frac{\eta\beta_i^{'}}{2^{m_2}}\bigg)^2\cdot v_{t,i}^2=\bigg(1-\frac{\eta\beta_i^{'}}{2^{m_2-1}}+\frac{\eta^2\beta_i^{'2}}{2^{2m_2}}\bigg)\cdot v_{t,i}^2.
	\end{align*}

	Then we further get an upper bound of $\beta_{t+1,i}$ as
	\begin{align}
		\beta_{t+1,i}&=w_{t+1,i}^2-v_{t+1,i}^2\nonumber\\&\le\bigg(1-\frac{1}{2^{m_2+1}}\bigg)\beta_{i}^{'}+\frac{\eta\beta_{i}^{'2}}{2^{m_2-1}}
		+\frac{\eta\beta_{i}^{'}v_{t,i}^2}{2^{m_2-2}}+\left(1-\frac{1}{2^{m_2+1}}\right)\frac{\eta^2\beta_i^{'3}}{2^{2m_2}}.\label{betaupper}
	\end{align}

	By the updating rule on $v_{t,i}^2$ given in \eqref{upv_i2}, we obtain that as long as $\beta_{t',i}\le\beta_i^{'}$ for all $t'\le t$, we always have $v_{t,i}^2\le \alpha^2$.
	Then our goal is to make sure
	\begin{align}\label{suff_a17}
		\frac{\eta\beta_{i}^{'2}}{2^{m_2-1}}+\frac{\eta\beta_{i}^{'}v_{t,i}^2}{2^{m_2-2}}+\left(1-\frac{1}{2^{m_2+1}}\right)\cdot \frac{\eta^2\beta_i^{'3}}{2^{2m_2}}\le \frac{\beta_{i}^{'}}{2^{m_2+1}},
	\end{align}
	in order to prove $\beta_{t+1,i}\le \beta_i'$. Thus, when $p\ge 2M_0^4/(|\mu^{*}|+M_0)$ and after setting $\eta\le{1}/[12(|\mu^{*}|+M_0)]$ in \eqref{betaupper}, we have every element at the left hand side of \eqref{suff_a17} is smaller than ${\beta_{i}^{'}}/{(3\cdot2^{m_2+1})}$, so \eqref{suff_a17} is satisfied for all $i\in S$ and  we have $\beta_{t+1,i}\le\beta_{i}^{'}$ for all $i\in S$. Thus, we have finished our proof of the second part in   Lemma \ref{monotone_lem} above.
\end{proof}

\subsection{ Proof of Theorem \ref{thmpred}}\label{MSEproof}
\begin{proof}
 In this subsection, we will prove our results on the MSE of kernel regression with gaussian covariates.
As a reminder, in \S\ref{predriskvec} we define $Z^{*}=\mathbf{x}^\mathsf{T}\beta^{*}$, $Z=\mathbf{x}^\mathsf{T}\hat{\beta}$ and $Z_i=\mathbf{x}_i^\mathsf{T}\hat{\beta}$, and event $\{Z,|Z-\mu^\top\hat\beta|\le R \}$, where $R=2\sqrt{\log n}$  and $\xb$ is a new observation. We further define our prediction function $\hat{g}(Z)$ as
\begin{equation}  \label{predfunc}
\hat{g}(Z) =\left\{
\begin{aligned}
&\frac{\sum_{i=1}^{n}y_iK_h(Z-Z_i)}{\sum_{i=1}^{n}K_h(Z-Z_i)},|Z-\mu^\top\hat\beta|\le R,   \\  &\quad \quad \quad 0, \quad\quad\quad\quad\quad \text{otherwise}
\end{aligned}
\right.
\end{equation}
in which we assume $0/0=0$. Note that $Z-\mu^\top\hat\beta$ is a random variable which follows standard Gaussian distribution under our settings given in \S\ref{predriskvec}, then we get a tail bound for $Z$ as
\begin{align}\label{predineq}
\mathbb{P}\big( |Z-\mu^\top\hat\beta|\ge t\big)=2\exp\left(-t^2/2\right)
\end{align}
In other words, by letting $t=2\sqrt{\log n}$ in \eqref{predineq}, with probability $1-2/n^2$, we have
$
|Z-\mu^\top\hat\beta|\le2\sqrt{\log n}.
$
Next,  we separate our prediction error into two parts
\begin{align*}
\mathbb{E}\left[(\hat{g}(Z)-f(Z^{*}))^2\right]=\underbrace{\mathbb{E}\left[(\hat{g}(Z)-f(Z^{*}))^2 \cdot \mathbb{I}_{\{|Z-\mu^\top\hat\beta|\le R\}}\right]}_{(\mathbf{I})}+\underbrace{\mathbb{E}\left[(\hat{g}(Z)-f(Z^{*}))^2 \cdot \mathbb{I}_{\{|Z-\mu^\top\hat\beta|>R\}}\right]}_{(\mathbf{II})}.
\end{align*}
 For term $(\mathbf{II})$, by our definition of $\hat{g}(Z)$ given in \eqref{predfunc}, we have
\begin{align*}
\mathbb{E}\left[(\hat{g}(Z)-f(Z^{*}))^2 \mathbb{I}_{\{|Z-\mu^\top\hat\beta|>R\}}\right]&\le 2\mathbb{E}\left[f(Z^{*})^2\mathbb{I}_{\{|Z-\mu^\top\hat\beta|>R\}} \right]\\&\le2\sqrt{\mathbb{E}\left[f(Z^{*})^4\right]}\sqrt{\mathbb{P}(|Z-\mu^\top\hat\beta|>R)}  \lesssim \|f\|_{\psi_2}\cdot \frac{1}{n},
\end{align*}
where the second inequality follows from Cauchy-Schwartz Theorem. In addition, the third inequality above is given by our assumption on $f(Z^{*})$, in which we assumed $f(Z^{*})$ is a sub-Gaussian random variable with variance proxy $\|f\|_{\psi_2}$.

For term $(\mathbf{I})$, we  further separate it into $(\mathbf{III})$ and $(\mathbf{IV})$ which are  regarded as integrated mean square error and approximation error respectively.
\begin{align}\label{mse_err}
\mathbf{(I)}=\underbrace{\mathbb{E}\left[(\hat{g}(Z)-g(Z))^2\mathbb{I}_{ \{ |Z-\mu^\top\hat\beta|\le R \}}\right]}_{\mathbf{(III):}\text{(MSE)}}+\underbrace{\mathbb{E}\left[(g(Z)-f(Z^{*}))^2\mathbb{I}_{ \{ |Z-\mu^\top\hat\beta|\le R \}}\right].}_{\mathbf{(IV):}\text{(Approximation error)}}
\end{align}
For $(\mathbf{III})\,\text{(MSE)}$, we define $g_0(Z)$ as
\begin{align*}
g_0(Z)=\frac{\sum_{i=1}^{n}g(Z_i)K_h(Z-Z_i)}{\sum_{i=1}^{n}K_h(Z-Z_i)}.
\end{align*}
Then we see $(\mathbf{III})$ can also be controlled by two terms, namely variance and bias of our approximation
\begin{align}\label{var_bias}
\mathbf{(III)}\le 2\underbrace{\mathbb{E}\left[(\hat{g}(Z)-g_0(Z))^2\mathbb{I}_{ \{ |Z-\mu^\top\hat\beta|\le R \}}\right]}_{\mathbf{(V):}\text{(Variance)}}+2\underbrace{\mathbb{E}\left[(g_0(Z)-g(Z))^2\mathbb{I}_{ \{ |Z-\mu^\top\hat\beta|\le R \}}\right]}_{\mathbf{(VI):}\text{(Bias)}}.
\end{align}
Combining \eqref{mse_err} and \eqref{var_bias}, we see that the $\ell_2$-risk can be bounded by a sum of the approximation error, bias, and variance. In the sequel, we bound these three terms separately.

\vspace{5pt}
{\noindent{\bf Step I: Approximation error.}}
By our settings in \S\ref{predriskvec}, both $Z-\mu^\top\hat\beta$ and $Z^*-\mu^\top\beta^*$ are standard Gaussian random variables.
Moreover, we have
\$
Z^{*}&=\mu^\top\beta^{*}+\langle \Sigma^{1/2}\hat{\beta},\Sigma^{1/2}\beta^{*}\rangle\cdot(Z-\mu^\top\hat{\beta})+\sqrt{1-\langle \Sigma^{1/2}\hat{\beta},\Sigma^{1/2}\beta^{*}\rangle^2}\cdot\zeta\\
&:=\cos\alpha\cdot Z+\sin\alpha\cdot\zeta+\mu^\top\beta^{*}-\cos\alpha\cdot\mu^\top \hat{\beta},
\$
where  $\alpha  \in [0,\pi / 2 ]$ and $\zeta \sim N(0, 1)$ is  independent of  $Z$.
In addition, by Assumption \ref{ass1}-(a) and \eqref{eq:assumption}, it holds that
\begin{align*}
	\sin\alpha^2=1-\langle \Sigma^{1/2}\hat\beta,\Sigma^{1/2}\beta^{*} \rangle^2=o(n^{-2/3}).
\end{align*}
{Thus, the single index model can be equivalently written as}
\#\label{eq:transform}
Y = f(Z^*) + \epsilon, ~~~~ Z^*= \cos\alpha\cdot (Z-\mu^\top\hat\beta)+\sin\alpha\cdot\zeta+\mu^\top\beta^{*}.
\#
For simplicity, we denote $\tilde{Z}(z)$ as $\tilde{Z}(z)= \cos\alpha\cdot (z-\mu^\top\hat\beta)+\mu^\top\beta^{*}.$ 
{Then, according to \eqref{eq:transform}, the regression function is given by}
\#\label{eq:regression_func}
g(z) = \EE[ Y \given Z = z]  = \EE \big [ f\big (\tilde{Z}(z) + \sin \alpha \cdot  \zeta \big ) \given   Z = z  \big ] = \int_{\RR} f\big(\tilde{Z}(z) + \sin \alpha \cdot \zeta  \big) \cdot \phi(\zeta) \ud \zeta,
\#
where $\phi$ is the density   of the standard Gaussian distribution.
To bound the approximation error $(\mathbf{IV})$, we first use $f(\cos \alpha \cdot (Z-\mu^\top\hat\beta)+\mu^\top\beta^{*})$ to approximate $f(Z^*)$ as well as  $g(Z)$. For simplicity, we denote $\tilde{Z}$ as $\tilde{Z}=\cos \alpha \cdot (Z-\mu^\top\hat\beta)+\mu^\top\beta^{*}$ with $\cos \alpha = \la \Sigma^{1/2}\beta^*, \Sigma^{1/2}\hat \beta \ra $, then the approximation error is bounded as
\#\label{eq:approx1}
&\EE \left[ (f(Z^*) - g(Z))^2\mathbb{I}_{ \{|Z-\mu^\top\hat\beta|\le R\} } \right ] \\&\quad\le 2 \EE \left[ \{f(Z^*) - f(\tilde{Z})\}^2 \mathbb{I}_{ \{|Z-\mu^\top\hat\beta|\le R\}}  \right]+ 2\EE \left [ (f(\tilde{Z}) - g(Z))^2\mathbb{I}_{ \{|Z-\mu^\top\hat\beta|\le R\}}  \right ],\label{2approx}
\#
 For the first term on the right-hand side of  \eqref{2approx}, by Taylor expansion we have
\$
f(Z^*) - f(\tilde{Z}) = f( \tilde{Z}+ \sin \alpha \cdot \zeta ) - f(\tilde{Z}) = f'(\tilde{Z}+ t_1 \sin \alpha \cdot \zeta ) \cdot \sin \alpha \cdot\zeta,
\$
which implies that
\#
&\EE \left [ (f(Z^*) - f(\tilde{Z}))^2\mathbb{I}_{ \{|Z-\mu^\top\hat\beta|\le R\} } \right ]\nonumber\\&\quad  = \sin^2\alpha \int_{|Z-\mu^\top\hat\beta|\le R}\int_{\RR} f'^2(\tilde{Z} + t_1(Z,\zeta) \sin \alpha \cdot \zeta)\zeta^2\phi(\zeta)\ud\zeta \ud F(Z) \lesssim \sin^2\alpha,\label{eq:approx22}
\#
where $t_1(Z,\zeta)$ is a constant lines in $[0,1]$ which depends on $Z,\zeta$. For \eqref{eq:approx22} given above, we utilize Assumption \ref{assume:nonparam}. For the second term, by the definition of $g$ given in \eqref{eq:regression_func} we have
\$
& \left | f(\tilde{Z} ) - g(Z) \right | = \left | f(\tilde{Z} ) - \int_{\RR} f(\tilde{Z} + \sin \alpha \cdot \zeta)\phi(\zeta)\ud \zeta \right | \\
&\quad  =\left| \sin\alpha\cdot  \int_{\RR}   f'(\tilde{Z} +t_2(Z,\zeta) \sin \alpha \cdot \zeta) \zeta \phi(\zeta) \ud \eta\right|, 
\$
which implies that
\#
&\EE \left [ (f(\tilde{Z} ) - g(Z))^2\II_{\{|Z-\mu^\top\beta|\le R\}} \right ] \nonumber\\&\quad\le \sin^2\alpha \int_{|Z-\mu^\top\hat\beta|\le R} \left(\int_{\RR}  f'(\tilde{Z} +t_2(Z,\zeta) \sin \alpha \cdot \zeta) \zeta \phi(\zeta) \ud \zeta \right)^2\ud F(Z) \nonumber\\&\quad\le\sin^2\alpha \int_{|Z-\mu^\top\hat\beta|\le R}\int_{\RR} f'^2(\tilde{Z} + t_2(Z,\zeta) \sin \alpha \cdot \zeta)\zeta^2\phi(\zeta)\ud\zeta \ud F(Z) \lesssim \sin^2\alpha.\label{eq:approx33}
\#
Combining  \eqref{eq:approx1}, \eqref{eq:approx22}, and \eqref{eq:approx33} we bound the approximation error term by
\#\label{eq:approx_final}
\mathbf{(IV)}=\EE \left [\left ( f(Z^* ) - g(Z) \right )^2 \mathbb{I}_{ \{ |Z-\mu^\top\hat\beta|\le R \}} \right ]  \lesssim  \sin^2 \alpha\lesssim o(n^{-2/3}).
\#
Next, we control the strength of term $(\mathbf{V}),$ which is regarded as the variance of our approximation.

\vspace{5pt}
{\noindent{\bf Step II: Variance control.}}
For term $(\mathbf{V})$, by definition, we obtain
\begin{align*}
\mathbf{(V)}=\int_{|Z-\mu^\top\hat\beta|\le R}\int\mathbb{E}\left[(\hat{g}(Z)-g_0(Z))^2\given Z_1,\dots, Z_n\right]\ud F(Z_1,\dots,Z_n)\ud F(Z).
\end{align*}
For any fixed $Z$, we let $B_n(z):=\{ Z:nP_n(B(Z,h))>0\}$, where $\PP_n(B(Z,h))=\frac{1}{n}\sum_{i=1}^{n}\mathbb{I}_{(\|Z_i-Z\|_2\le h)}.$
Then we further have
\begin{align*}
\mathbb{E}\left[(\hat{g}(Z)-g_0(Z))^2\given Z_1,\dots, Z_n\right]&=\mathbb{E}\left[\left[\frac{\sum_{i=1}^{n}(y_i-g(z_i))\mathbb{I}_{\{\|Z_i-Z\|_2\le h\}}}{
	\sum_{i=1}^{n}\mathbb{I}_{ \{\|Z_i-Z\|_2\le h  \}  }}\right ]^2 \Big{|}\,Z_1,\dots, Z_n\right]\\
&=\frac{\sum_{i=1}^n\text{Var}(Y_i\given Z_i)\mathbb{I}_{\{ \|Z_i-Z\|_2\le h   \}}}{n^2\PP_n(B(Z,h))^2}\le \frac{\sigma^2}{n\mathbb{P}_n(B(Z,h))}\cdot\mathbb{I}_{B_n(Z)}.
\end{align*}
For the last inequality, we have that $\text{Var}(Y_i\given Z_i)\le\mathbb{E}[Y_i^2\given Z_i]\le\sigma^2\lesssim \text{polylog}(n)$ holds by our following Lemma \ref{lemlipschitzg}-\textbf{(ii)}.
\begin{lemma}\label{lemlipschitzg}
	Under our settings given in \S\ref{predriskvec}, under  Assumption \ref{assume:nonparam},  the following arguments hold true.
	\begin{itemize}
		\item [\textbf{(i)}.] $g(z)$ function defined in  \eqref{eq:regression_func} is Lipschitz over area $\{|z-\mu^\top\hat\beta|\le R \}$, whose Lipschitz constant $L$ is bounded by $\textrm{poly}(R)$.
		\item [\textbf{(ii)}.] The variance of $Y$ given $Z=z$ with $|z-\mu^\top\hat\beta|\le R+h,h=o(1)$  is bounded by $\text{poly}(R)$.
		\item [\textbf{(iii)}.]$\sup_{|z-\mu^\top\hat\beta|\le R} g(z)\le \textrm{ploy}(R).$
		\end{itemize}
\end{lemma}
\begin{proof}
	See \S\ref{sect:predlemma} for a detailed proof.
	\end{proof} 
So we obtain
\begin{align*}
\mathbf{(V)}\le\int_{|Z-\mu^\top\beta|\le R}\int\frac{\sigma^2\mathbb{I}_{B_n(Z)}}{n\mathbb{P}_n(B(Z,h))}\ud F(Z_1,\dots,Z_n)\ud F(Z).
\end{align*}
As we have $n\mathbb{P}_n(B(Z,h))=\sum_{i=1}^{n}\mathbb{I}_{(\|Z_i-Z\|_2\le h)}\sim\text{Binomial}(n,q)$, with $q=\mathbb{P}(Z_1\in B(Z,h))$,  we then obtain
\begin{align*}
\int\frac{\sigma^2\II_{B_n(Z)}}{n\mathbb{P}_n(B(Z,h))}\ud F(Z_1,\dots, Z_n)&=\int \frac{\sigma^2\mathbb{I}_{B_n(Z)}}{n\mathbb{P}_n(B(Z,h))}{\ud F(Z_1,\dots,Z_n)}\\
&=\mathbb{E}\bigg[ \frac{\sigma^2\mathbb{I}_{(n\mathbb{P}_n(B(Z,h))>0)}}{n\mathbb{P}_n(B(Z,h))}\bigg]\le \frac{2\sigma^2}{nq}.
\end{align*}
The last inequality follows from Lemma 4.1 in \cite{Gyorfi2002}. Then we further get an upper bound for $(\mathbf{V})$ as
\begin{align*}
\mathbf{(V)}&\le2\sigma^2\int_{|Z-\mu^\top\hat\beta|\le R}\frac{\ud F(Z)}{n\mathbb{P}(Z_1\in B(Z,h))} .
\end{align*}
 As $\{|Z-\mu^\top\hat\beta|\le R\}$ is a bounded area, we choose $x_1,\dots,x_m$ such that $\{|Z-\mu^\top\hat\beta|\le R\}$ is covered by $\cup_{j=1}^{M}B(x_i,h/2)$ with $M\le cR/h$.
Then we finally bound term $(\mathbf{V})$ as
\begin{align}
\mathbf{(V)}&\le 2\sigma^2\int_{|Z-\mu^\top\hat\beta|\le R}\frac{\ud F(Z)}{n\mathbb{P}(B(Z,h))} \le\sum_{j=1}^{M}2\sigma^2\int \frac{\mathbb{I}_{\{Z\in B(x_j,h/2)\}}\ud F(Z)}{n\mathbb{P}(B(Z,h))}\nonumber\\
&\le\sum_{j=1}^{M}2\sigma^2\int \frac{\mathbb{I}_{\{Z\in B(x_j,h/2)\}}\ud F(Z)}{n\mathbb{P}(B(x_j,h/2))} \le \frac{2\sigma^2M}{n}\le\frac{C\sigma^2R}{nh}.\label{var_control}
\end{align}
In the next step, we will get an upper bound for the bias term of our approximation.

\vspace{5pt}
{\noindent{\bf Step III: Bias control.}}
For term $(\mathbf{VI})$, we  first bound the difference between $g_0(Z)$ and $g(Z)$
\begin{align*}
\left|g_0(Z)-g(Z)\right|^2&=\left|  \frac{\sum_{i=1}^{n}(g(Z_i)-g(Z))K_h(Z-Z_i)}{\sum_{i=1}^{n}K_n(Z-Z_i)}    \right|^2 \le L^2h^2+g^2(Z)\cdot \mathbb{I}_{B_n(Z)^{c}},
\end{align*}
where the last inequality follows from Lemma \ref{lemlipschitzg}-\textbf{(i)}, which yields $g$ is a Lipschitz function with Lipschitz constant $L$ bounded by polylog$(n)$. Then we obtain
\begin{align}
&\mathbb{E}\left[\left| g_0(Z)-g(Z)\right|^2\mathbb{I}_{\{|Z-\mu^\top\hat\beta|\le R\}} \right]\nonumber\\&\quad\le L^2h^2+\int_{|Z-\mu^\top\hat\beta|\le R} g^2(Z)\mathbb{E}\left[\mathbb{I}_{B_n(Z)^c}\right]\ud F(Z)\nonumber\\&\quad\le L^2h^2+\sup_{|Z-\mu^\top\beta|\le R}g^2(Z)\int_{|Z-\mu^\top\hat\beta|\le R}[1-\PP(Z_1\in B(Z,h))]^n \ud F(Z)
\nonumber\\&\quad\le L^2h^2+\sup_{|Z-\mu^\top\hat\beta|\le R} g^2(Z)\int_{|Z-\mu^\top\hat\beta|\le R} \exp(-n\mathbb{P}(Z_1\in B(Z,h)))\cdot\frac{n\mathbb{P}(Z_1\in B(Z,h))}{n\mathbb{P}(Z_1\in B(Z,h))}\ud F(Z)
\nonumber\\&\quad\le L^2h^2+\sup_{|Z-\mu^\top\hat\beta|\le R}g^2(Z)\sup_u\{u e^{-u}\}\int_{|Z-\mu^\top\hat\beta|\le R}\frac{\ud F(Z)}{n\mathbb{P}(B(Z,h))}
\nonumber\\&\quad\le L^2h^2+\frac{\text{polylog}(n)}{nh}. \label{biascotrol}
\end{align}
The last inequality \eqref{biascotrol} also follows from our Lemma \ref{lemlipschitzg}-\textbf{(iii)}. Thus, combining our conclusions from  \eqref{eq:approx_final}, \eqref{var_control} and \eqref{biascotrol}, and by letting $h=n^{-1/3}$, we   bound the  $\ell_2$-error as
\begin{align*}
\mathbb{E}\left[(\hat{g}(Z)-f(Z^{*}))^2\right]\lesssim \frac{\text{polylog}(n)}{n^{2/3}},
\end{align*}
which concludes the proof of   of Theorem \ref{thmpred}.
\end{proof}
\subsubsection{Proof of Lemma \ref{lemlipschitzg}}\label{sect:predlemma}
\begin{proof}
For term \textbf{(i)}, by mean value theorem, we have
\begin{align}
&|g(z_1)-g(z_2)|\nonumber\\&\quad\le \bigg \{ \int_{\RR} \big |f' \bigl (\cos\alpha \cdot  [ z_1-\mu^\top\hat\beta+t(\zeta)\cdot(z_2-z_1)] +\sin\alpha\cdot \zeta+\mu^\top\beta^{*} \bigr ) \big| \cdot \phi(\zeta) \ud \zeta \bigg \} \cdot|z_1-z_2|,\label{lipschitzg}
\end{align}
where $t(\zeta)$ is a constant inside $[0,1]$ that depends on $\zeta$. Here, if $\alpha_1\le 1$, the right hand side of \eqref{lipschitzg} is bounded as
\begin{align*}
&\bigg\{\int_{\RR}|f'(\cos\alpha \cdot  [ z_1-\mu^\top\hat\beta+t(\zeta)\cdot(z_2-z_1)] +\sin\alpha\cdot \zeta+\mu^\top\beta^{*})|\phi(\zeta)\ud\zeta \bigg\}\cdot|z_1-z_2|\\
&\quad \le \bigg\{\int_{\RR} C+|\cos\alpha \cdot  [ z_1-\mu^\top\hat\beta+t(\zeta)\cdot(z_2-z_1)] +\sin\alpha\cdot \zeta+\mu^\top\beta^{*}|^{\alpha_1}\phi(\zeta)\ud\zeta\bigg\}\cdot|z_1-z_2|\\&\quad\le (C+1)\cdot|z_1-z_2|+\bigg\{\int_{\RR}|\cos\alpha \cdot  [ z_1-\mu^\top\hat\beta+t(\zeta)\cdot(z_2-z_1)] +\sin\alpha\cdot \zeta+\mu^\top\beta^{*}|\phi(\zeta)\ud\zeta\bigg\}\cdot|z_1-z_2|\\&\quad\le (C+1+ R\cdot|\cos\alpha|+\mu^\top\beta^{*})\cdot|z_1-z_2|+|\sin\alpha|\cdot |z_1-z_2|\cdot\int_{\RR}|\zeta|\phi(\zeta)\ud\zeta\\&\quad\le (C_1+R)\cdot |z_1-z_2|,
\end{align*}
in which $C_1$ is a constant. The second inequality follows from $|x|^{\alpha_1}\le1+|x|$ with $\alpha_1\le1.$ The third inequality follows from $| \cos\alpha \cdot  [ z_1-\mu^\top\hat\beta+t(\zeta)\cdot(z_2-z_1)]|\le R\cdot |\cos\alpha|$ by definition of $z_1$ and $z_2$.

In addition, if $\alpha_1>1,$ by Assumption \ref{assume:nonparam}, i.e. $|f'(x)|\le C+|x|^{\alpha_1}$ and the convexity property of function $f(x)=|x|^{\alpha_1},$ we then have
\begin{align*}
&\bigg\{\int_{\RR}|f'(\cos\alpha \cdot  [ z_1-\mu^\top\hat\beta+t(\zeta)\cdot(z_2-z_1)] +\sin\alpha\cdot \zeta+\mu^\top\beta^{*})|\phi(\zeta)\ud\zeta\bigg\}\cdot|z_1-z_2|\\&\quad\le [C+3^{\alpha_1-1}|\cos\alpha|^{\alpha_1} R^{\alpha_1}+3^{\alpha_1-1}(\mu^\top\beta^{*})^{\alpha_{1}}]\cdot |z_1-z_2|+3^{\alpha_1-1}|\sin\alpha|^{\alpha_1}\cdot|z_1-z_2|\cdot\int_{\RR}|\zeta|^{\alpha_1}\phi(\zeta)\ud\zeta\\&\quad\le(C_2+3^{\alpha_1-1}R^{\alpha_1})\cdot |z_1-z_2|.
\end{align*}
The second inequality follows from inequality $[(x+y+z)/3]^{\alpha_1}\le [x^{\alpha_1}+y^{\alpha_1}+z^{\alpha_1}]/3$ with $\alpha_1>1$.
Thus, we claim that our $g$ function is Lipschitz over area $\{|z-\mu^\top\hat\beta|\le R\}.$

For terms \textbf{(ii)} and \textbf{(iii)}, by definitions, we know
\begin{align*}
\text{Var}(Y\given Z=z)&\le \mathbb{E}[Y^2\given Z=z] \le \int_{\RR} f^2(\tilde{Z}(z)+\sin\alpha \cdot \zeta)\phi(\zeta)\ud\zeta+\sigma^2,\textrm{\,\,and\,\,}\\
g(z)&=\EE[Y|Z=z]=\int_{\RR}f(\tilde{Z}+\sin\alpha\cdot \zeta)\cdot\phi(\zeta)\ud\zeta.
\end{align*}
in which $\sigma^2$ denotes the variance of $\epsilon.$  By our assumption on $f$ given in Assumption \ref{assume:nonparam}, after following similar procedures given by us of proving part \textbf{(i)}, we claim our conclusion for terms \textbf{(ii)} and \textbf{(iii)}.
\end{proof}

\subsection{Proof of Theorem \ref{thmvec2}}\label{secheavy}
\begin{proof}
	The proof of Theorem \ref{thmvec2} is almost the same with the proof of Theorem \ref{thmvec1}. The major differences between them are two folds. Firstly, we need to replace the estimator $\Phi_n:=\frac{1}{n}\sum_{i=1}^{n}y_i\mathbf{x}_i$ by $\frac{1}{n}\sum_{i=1}^{n}\widecheck{y_i}\widecheck{S}(\mathbf{x}_i)$ in  \eqref{perr1} and \eqref{upu_i2}-\eqref{upv_i2}. In addition, we  establish a new concentration inequality between $\frac{1}{n}\sum_{i=1}^{n}\widecheck{y_i}\widecheck{S}(\mathbf{x}_i)$ and $\mu^{*}\beta^{*}$ in the following Lemma \ref{lemgenconc}.
	\begin{lemma}\label{lemgenconc}
		Under Assumption \ref{genass1}, by choosing threshold $\tau=(M\cdot n/\log p)^{1/4}/2$, we have
		\begin{align}\label{genconc}
			\bigg\| \frac{1}{n}\sum_{i=1}^{n}\widecheck{y_i}\widecheck{S}(\mathbf{x}_i)-\mu^{*}\beta^{*} \bigg\|_{\infty}\le M_{g}\sqrt{\frac{\log p}{n}}
		\end{align}
		holds with probability $1-2/p^2$ with $M_g$ being a constant that only depending on $M$ given in Assumption \ref{genass1}.
	\end{lemma}
	\begin{proof}
	See  \S\ref{sect:prooflemmagen} for a detailed proof.
		\end{proof}
	By our conclusion from Lemma \ref{lemgenconc} and following the proof procedure of Lemma \ref{prop1.3} and \ref{prop1.4} above,  with probability $1-2/p^2$, there exists a constant $a_4$ such that we obtain
	\begin{align*}
	&\|\mathbf{e}_{1,t}\|_{\infty}\le \sqrt{\alpha}\le \frac{M_g}{\sqrt{p}}, \qquad	\|\mathbf{e}_{2,t}\|_{\infty}\le \sqrt{\alpha}\le \frac{M_g}{\sqrt{p}},
	\end{align*}
	for all $t\le T:=a_4\log(1/\alpha)\sqrt{n/\log p}/\eta$. Similarly, for signal parts, with probability $1-2/p^2$, there also exists a constant $a_3$ such that we have
	\begin{align*}
	\big\|\beta_{t}\odot\mathbf{1}_{S_0}-\mu^{*}\beta^{*}\odot\mathbf{1}_{S_0}\big\|_{\infty}&\lesssim \sqrt{\frac{\log p}{n}},
	\end{align*}
 for all $t\ge a_3\log({1}/{\alpha})/[\eta(\mu^{*} s_m-M_g\sqrt{\log p/n})]$ and
 \begin{align*}
	\big\|\beta_{t}\odot\mathbf{1}_{S_1}-\mu^{*}\beta^{*}\odot\mathbf{1}_{S_1}\big\|_{\infty}&\lesssim \sqrt{\frac{\log p}{n}},
	\end{align*}
  for all $t\ge 0$ with probability $1-2/p^2$. The way of choosing $a_3$, $a_4$ is the same with choosing $a_1$ and $a_2$, for simplicity, we omit the details here.	Combining  two conclusions above, we claim our proof  of Theorem \ref{thmvec2}.
	Next, we will prove Lemma \ref{lemgenconc} which we have applied in the process of proving Theorem \ref{thmvec2}.
	\end{proof}
\subsubsection{Proof of Lemma \ref{lemgenconc}}\label{sect:prooflemmagen}
\begin{proof}
	We separate the left hand side of  \eqref{genconc} into two parts, namely
	\begin{align*}
	&\bigg\|\frac{1}{n}\sum_{i=1}^{n}\widecheck{y_i}\widecheck{S}(\mathbf{x}_i)-\mu^{*}\beta^{*}\bigg\|_{\infty}\\&\quad=\bigg\| \frac{1}{n}\sum_{i=1}^{n}\widecheck{y_i}\widecheck{S}(\mathbf{x}_i)-\mathbb{E}[\widecheck{y_1}\widecheck{S}(\mathbf{x_1})]+\mathbb{E}[\widecheck{y_1}\widecheck{S}(\mathbf{x_1})]-\mathbb{E}[y_1\cdot S(\mathbf{x}_1)]     \bigg\|_{\infty}\\
	&\quad\le \bigg\| \frac{1}{n}\sum_{i=1}^{n}\widecheck{y_i}\widecheck{S}(\mathbf{x}_i)-\mathbb{E}[\widecheck{y_1}\widecheck{S}(\mathbf{x_1})] \bigg\|_{\infty}+\Big\|   \mathbb{E}[\widecheck{y_1}\widecheck{S}(\mathbf{x_1})]-\mathbb{E}[y_1\cdot S(\mathbf{x}_1)] \Big\|_{\infty}.
	\end{align*}
To simplify the notations, within this proof, we define
\$
\frac{1}{n}\sum_{i=1}^{n}\widecheck{y_i}\widecheck{S}(\mathbf{x}_i)-\mathbb{E}[\widecheck{y_1}\widecheck{S}(\mathbf{x_1})] = \tilde{\Psi}, \qquad \mathbb{E}[\widecheck{y_1}\widecheck{S}(\mathbf{x_1})]-\mathbb{E}[y_1\cdot S(\mathbf{x}_1)] = \tilde{\Phi}.
\$
In addition, we define event $C_j$ as $C_j=\{|y_1|\le\tau,|S(\mathbf{x}_1)_j|\le\tau\}$, then we are able to control $j$-th entry of $\tilde{\Phi}$ as
	\begin{align*}
	\tilde{\Phi}_j&=\mathbb{E}\big[\widecheck{y_1}\widecheck{S}(\mathbf{x_1})_j\big]-\mathbb{E}\big[y_1\cdot S(\mathbf{x}_1)_j\big] \\
	&\le\mathbb{E}\big[(|y_1|-\tau)\cdot (|S(\mathbf{x}_1)_j|-\tau) \cdot \mathbb{I}_{C_j^{c}}\big]\\
	&\le\sqrt{\mathbb{E}\big[y_1^2S(\mathbf{x}_1)^2_j\big]\cdot \big[\mathbb{P}(|y_1|>\tau)+\mathbb{P}(|S(\mathbf{x}_1)|>\tau)\big]}\\
	&\le \left\{\mathbb{E}\big[y_1^4\big]\cdot \mathbb{E}\big[S(\mathbf{x}_1)^4_j\big]     \right\}^{1/4}\cdot \sqrt{2}M^{1/2}/\tau^2 \leq \sqrt{2}M/\tau^2.
			\end{align*}
	The third and fourth inequalities are established by Cauchy Schwartz inequality and Chebyshev inequality respectively. In addition, the last inequality follows from   Assumption \ref{genass1}. Note that the inequality above holds for any $j\in[d]$ so that we have
	$
	 \| {\tilde{\Phi}} \|_{\infty}\le  \sqrt{2}M / \tau^2.
$
	For term $\tilde{\Psi}$, by definition, we know that $|\widecheck{y}_i\widecheck{S}(\mathbf{x}_i)_j|\le\tau^2$ and $\sum_{i=1}^{n}\text{Var}(\widecheck{y_i}\widecheck{S}(\mathbf{x}_i)_j)\le n\cdot M$ with $j\in[p]$.  After directly applying Bernstein inequality and we further obtain
	\begin{align}\label{trunbernst}
	\mathbb{P}\left(\bigg\|\frac{1}{n}\sum_{i=1}^{n}\widecheck{y_i}\widecheck{S}(\mathbf{x}_i)-\mu^{*}\beta^{*}\bigg\|_{\infty}\ge\frac{\sqrt{2}M}{\tau^2}+t\right)\le 2p\cdot\exp\bigg(-\frac{nt^2}{M+\tau^2 t/3}\bigg).
	\end{align}
 We set $t=m_1\sqrt{\log p/n}$ and $\tau=m_2^{1/2}(n/\log p)^{1/4}$ in \eqref{trunbernst}, where  $m_1$ and $m_2$ are constants that we will specify later.  We aim at establishing the following inequality
 \begin{align*}
 2 p\cdot\exp\bigg(-\frac{nt^2}{M+\tau^2 t/3}\bigg)=2p\cdot\exp\bigg(-\frac{3m_1^2\log p}{3M+m_1m_2}\bigg)\le \frac{2}{p^2}.
 \end{align*}
 Then by setting $m_1=2\sqrt{M}$ and $m_2=\sqrt{M}/4$, we  obtain
 \begin{align*}
 \frac{3m_1^2}{3M+m_1m_2}\ge 3.
 \end{align*}
 Thus, we obtain that
 \begin{align*}
\bigg\| \frac{1}{n}\sum_{i=1}^{n}\widecheck{y_i}\widecheck{S}(\mathbf{x}_i) -\mu^{*}\beta^* \bigg\|\le (4\sqrt{2}+2)\sqrt{M}\sqrt{\frac{\log p}{n}}
 \end{align*}
holds with probability $1-2/p^2,$ and we conclude the proof  of Lemma \ref{lemgenconc}.
	\end{proof}

\subsection{Algorithm in \S\ref{secgenvec}}\label{alggenvec}
\begin{algorithm}[H]
	\KwData{Training covariates $\{\xb_{i}\}_{i=1}^{n},$  response vector $\{y_i\}_{i=1}^{n}$, truncating parameter $\tau$, initial value $\alpha$, step size $\eta$;}
	Initialize variables $\wb_0=\alpha\cdot\mathbf{1}_{p\times 1}$, $\vb_0=\alpha\cdot\mathbf{1}_{p\times 1}$ and set iteration number $t=0$;\\
	\While{$t<T_1$}{
		$\mathbf{w}_{t+1}=\mathbf{w}_{t}-\eta\big(\mathbf{w}_t\odot \mathbf{w}_t-\mathbf{v}_t\odot \mathbf{v}_t-\frac{1}{n}\sum_{i=1}^{n}\widecheck{S}(\xb_i)\widecheck{y_i}\big)\odot \mathbf{w}_t$;\\
		$\, \mathbf{{v}}_{t+1}=\mathbf{v}_{t}\,+\eta\big(\mathbf{w}_t\odot \mathbf{w}_t-\mathbf{v}_t\odot \mathbf{v}_t-\frac{1}{n}\sum_{i=1}^{n}\widecheck{S}(\xb_i)\widecheck{y_i}\big)\odot \mathbf{v}_t$; \\
		$\beta_{t+1}=\mathbf{w}_t\odot \mathbf{w}_t-\mathbf{v}_t\odot \mathbf{v}_t;$\\
		$\, t=t+1$;\\
	}
	\KwResult{ 
		Output the final estimate $\widehat \beta^{*}=\beta_{T_1}$.
	}\label{alg5}
	\caption{ Algorithm for  Vector SIM with General Design}
\end{algorithm}

\section{Proof of General Theorems in  \S\ref{sectionmat}}\label{proofmat}

\subsection{Algorithm in \S\ref{secgaussmat}}\label{gauss_mat_alg}
\begin{algorithm}[htpb]
	\KwData{Training design matrix $\mathbf{X}_i \in \mathbb{R}^{d\times d}$, $i\in[n]$, response variables $\{y_i\}_{i=1}^{n}$, initial value $\alpha$ and step size $\eta$;}
	Initialize  $\Wb_0=\alpha\cdot\mathbb{I}_{d\times d}$, $\Vb_0=\alpha\cdot\mathbb{I}_{d\times d}$ and set iteration number $t=0$;\\
	\While{$t<T_1$}{
		$\mathbf{W}_{t+1}=\mathbf{W}_{t}-\eta(\mathbf{W}_t \mathbf{W}_t^\top-\mathbf{V}_t \mathbf{V}_t^\top-\frac{1}{2n}\sum_{i=1}^{n}\mathbf{X}_{i}y_i-\frac{1}{2n}\sum_{i=1}^{n}\mathbf{X}_{i}^\top y_i) \mathbf{W}_t$;\\
		$\,\mathbf{V}_{t+1}\,=\mathbf{V}_{t}\,\,+\eta(\mathbf{W}_t \mathbf{W}_t^\top-\mathbf{V}_t \mathbf{V}_t^\top-\frac{1}{2n}\sum_{i=1}^{n}\mathbf{X}_{i}y_i-\frac{1}{2n}\sum_{i=1}^{n}\mathbf{X}_{i}^\top y_i) \mathbf{V}_t$; \\
		$\,\,\,\beta_{t+1}\,=\mathbf{W}_t \mathbf{W}_t^\top-\mathbf{V}_t \mathbf{V}_t^\top$;\\
		$\quad t\,=t+1$;\\
	}
	\KwResult{ Output the final estimate $\widehat \beta=\beta_{T_1}$.
	}\label{alg2}
	\caption{ Algorithm for Low Rank Matrix SIM with Gaussian Design}
\end{algorithm}

\subsection{Proof of Theorem \ref{thmmat1}}
 As we assume $\mu^{*}\beta^{*}:=E[f'(\langle \Xb,\beta^{*} \rangle)]\beta^{*}$ is symmetric in \S\ref{sectionmat}, so we over-parameterize $\mu^{*}\beta$ as $\Wb\Wb^\top-\Vb\Vb^\top,$ in which $\Wb$ and $\Vb$ are matrices with dimension $d\times d.$ Then our loss function related to $\Wb,\Vb$ becomes
	\begin{align*}
\min_{\Wb,\Vb}\,L(\Wb,\Vb):=\, \langle \Wb\Wb^\top-\Vb\Vb^\top,\Wb\Wb^\top-\Vb\Vb^\top\rangle-2\Big\langle\Wb\Wb^\top-\Vb\Vb^\top,\frac{1}{n}\sum_{i=1}^{n}y_i\Xb_i \Big\rangle.
\end{align*}
The gradient updates with respect to $\Wb, \Vb$ and $\beta$ are given by
 	\begin{align}
\Wb_{t+1}&=\Wb_t-\eta\Big(\Wb_t\Wb_t^\top-\Vb_t\Vb_t^\top-\frac{1}{2n}\sum_{i=1}^{n}y_i\Xb_i-\frac{1}{2n}\sum_{i=1}^{n}y_i\Xb_i^\top\Big)\Wb_t,\label{matup1}\\
\Vb_{t+1}&=\Vb_t+\eta\Big(\Wb_t\Wb_t^\top-\Vb_t\Vb_t^\top-\frac{1}{2n}\sum_{i=1}^{n}y_i\Xb_i-\frac{1}{2n}\sum_{i=1}^{n}y_i\Xb_i^\top\Big)\Vb_t,\label{matup2}\\
\beta_{t+1}&=\Wb_{t+1}\Wb_{t+1}^\top-\Vb_{t+1}\Vb_{t+1}^\top.\label{matup3}
 \end{align}
For simplicity, let $\Mb^{*} = \frac{1}{2n}\sum_{i=1}^{n}y_i\Xb_i+\frac{1}{2n}
\sum_{i=1}^{n}y_i\Xb_i^\top,$ whose spectral decomposition is $\Mb^{*}:=\Qb^{*}\Sigma^{*}\Qb^{*\top}.$  Here for identifiability, through this section, we always assume eigenvalues are sorted in order of decreasing value in the diagonal matrix for any spectral decomposition. We then define $\Wb_{1,t}$ and $\Vb_{1,t}$ as $\Wb_{1,t}=\Qb^{*\top}\Wb_t\Qb^{*}$ and $\Vb_{1,t}=\Qb^{*\top}\Vb_t\Qb^{*}$, meanwhile, the corresponding gradient updates with respect to $\Wb_{1,t}$ and $\Vb_{1,t}$ are given by
\begin{align*}
\Wb_{1,t+1}&=\Wb_{1,t}-\eta\big(\Wb_{1,t}\Wb_{1,t}^\top-\Vb_{1,t}\Vb_{1,t}^\top-\Sigma^{*}\big)\Wb_{1,t},\\
\Vb_{1,t+1}&=\Vb_{1,t}+\eta\big(\Wb_{1,t}\Wb_{1,t}^\top-\Vb_{1,t}\Vb_{1,t}^\top-\Sigma^{*}\big)\Vb_{1,t},\\
\beta_{1,t+1}&=\Wb_{1,t+1}\Wb_{1,t+1}^\top-\Vb_{1,t+1}\Vb_{1,t+1}^\top.
\end{align*}
If we initialize $\Wb_{1,0}$ and $\Vb_{1,0}$ as diagonal matrices, then all of their following updates will keep being diagonal matrices. In this case, our analysis on symmetric low rank matrices can be relaxed to the analysis on sparse vectors. Likewise, we also remind readers of the notations before formally proving Theorem \ref{thmmat1}.

Similar to the vector case, here we also divide eigenvalues of $\beta^{*}$ into different groups by their strengths. We let $r_i^{*},i\in[n]$ be the $i$-th eigenvalue of $\beta^{*}$. The support set $R$ of our eigenvalues is defined as $R:=\{i:\left|r_i^{*}\right|>0   \}.$  In addition, the set $R_0$ that  contains strong signals is defined as $R_0:=\{i: |r_i^{*}|\ge C_{ms}\sqrt{{d\log d}/{n}} \}$, and the set $R_1:=\{i: 0<|r_i^{*}|< C_{ms}\sqrt{ {d\log d}/{n}  }  \}$  denotes the collection of weak signals. The constant $C_{ms}$ will be specified  in the proof.  Likewise, pure error parts of $\Wb_{1,t}$ and $\Vb_{1,t}$ are denoted by $\Eb_{w,t}:=\mathbb{I}_{R^{c}}\Wb_{1,t}$ and $\Eb_{v,t}:=\mathbb{I}_{R^{c}}\Vb_{1,t}$ respectively.  Here, $\mathbb{I}_{R_0}$ is the  diagonal matrix with ones  in the index set $R_0$ and zeros elsewhere.   Moreover, strong signal parts of $\Wb_{1,t}$ and $\Vb_{1,t}$ are denoted by $\mathbf{S}_{w,t}=\mathbb{I}_{R_0}\Wb_{1,t}$ and $\mathbf{S}_{v,t}=\mathbb{I}_{R_0}\Vb_{1,t}$ and at the same time, weak signal parts are written as $\Ub_{w,t}:=\mathbb{I}_{R_1}\Wb_{1,t}$ and $\Ub_{v,t}:=\mathbb{I}_{R_1}\Vb_{1,t}$. The cardinality of set $R_0$ and $R_1$ are denoted by $r_0$ and $r_1$ respectively. For simplicity, we denote $\gamma^*$ as $\gamma^*=\sqrt{{n}/{d\log d}}$ through our proof in \S\ref{proofmat}. Next, we   formally prove  Theorem \ref{thmmat1}. 

\begin{proof}The proof idea behind Theorem \ref{thmmat1} is similar to  that of Theorem \ref{thmvec1}.
	We first prove that the strength of pure error part of the eigenvalues stay  small  for a large number of iterations.
	\begin{lemma}\label{matrixerror}(Error Dynamics)
		Under assumptions in Theorem~\ref{thmmat1}, there exists an absolute constant $a_6$ such that,
		with probability $1-1/(2d)-3/n^2$,
		we obtain
 	\begin{align*}
			&\left\|\Eb_{w,t}\right\|_{\oper}\le  \sqrt{\alpha}\le \frac{M_m}{\sqrt{d}},\qquad 
			\left\|\Eb_{v,t}\right\|_{\oper}\le \sqrt{\alpha}\le \frac{M_m}{\sqrt{d}}, 
		\end{align*}
		 for all $t \geq 0$ with $ t\le T\colon=a_6\log(1/\alpha)\gamma^*/(\eta M_m)$,
	 where $M_m$ is an absolute  constant that  proportional to $\max\{\sigma, \|f\|_{\psi_2}\}$.
	\end{lemma}
\begin{proof}
	See  \S\ref{prooflemmamat} for a detailed proof.
	\end{proof}
For the t-th iterate $\beta_t$, we separate it into three parts, namely, $\Qb^*\mathbb{I}_{R_0}\beta_{1,t}\Qb^{*\top},$ $\Qb^*\mathbb{I}_{R_1}\beta_{1,t}\Qb^{*\top}$ and $\Qb^*\mathbb{I}_{R^c}\beta_{1,t}\Qb^{*\top}.$
	By our conclusion from Lemma \ref{matrixerror}, with probability $1-1/(2d)-3/n^2,$ we  obtain
	\begin{align}\label{errordymat}
	\Big\|\Qb^*\mathbb{I}_{R^c}\beta_{1,t}\Qb^\mathsf{*T}\Big\|_{\oper}\lesssim \frac{1}{d},
	\end{align}
	for all $t$ with  $t\le T=a_6\log(1/\alpha)\gamma^*/(\eta M_m)$.

  Next, we  analyze the dynamics of strong signal components of $\{\beta_t\}_{t\ge 0}$ in the following Lemma \ref{matrixsignal}, which shows that  the  strong signals converges rapidly  to their corresponding ground truths.
	\begin{lemma}\label{matrixsignal}(Strong Signal Dynamics)
		Let the spectral decomposition of $\mu^{*}\beta^{*}$ be $\mu^{*}\beta^{*}=\Pb^{*}\Rb^{*}\Pb^{*\top}.$ We denote the minimum absolute value of our strong signals $\beta^{*}$ as $r_m$. Under assumptions in Theorem~\ref{thmmat1}, if we further choose   $0<\eta \le{1}/[12(|\mu^{*}|+M_m)]$, and $C_{ms}\ge 2M_m/|\mu^{*}|$, then there exists an absolute constant $a_5$ such that,
		with probability at least $1-1/(2d)-3/n^2$, we have
		\begin{align}\label{singdymat}
			\Big\|\Qb^*\beta_{1,t}\mathbb{I}_{R_0}\Qb^\mathsf{*T}-\Pb^*\Rb^{*}\II_{R_0\cup R_1}\Pb^{*\top}\Big\|_{\oper}\lesssim \sqrt{\frac{d\log d}{n}}
		\end{align}
 for all $t\ge a_5/[\eta(|\mu^{*}|r_m-M_m\sqrt{d\log d/n })]\cdot \log(1/\alpha)$.
	\end{lemma}
	\begin{proof}
		See  \S\ref{prooflemmamat} for a detailed proof.
	\end{proof}

 Finally, we characterize the dynamics of weak signal parts in the following lemma, which shows that  the     weak signals  is always bounded by  $\cO(\sqrt{d\log d/n})$ when the stepsize $\eta$ is properly chosen.
	\begin{lemma}\label{matrixsignal_w}(Weak Signal Dynamics)
		Under assumptions in Theorem~\ref{thmmat1}, if we further choose  the stepsize $\eta$ such that  $0<\eta \le{1}/[12(|\mu^{*}|+M_m)]$,
		then with probability at least $1-1/(2d)-3/n^2$,  for all $t\geq 0$
		 we have
		\begin{align}\label{singdymat_w}
			\Big\|\Qb^*\beta_{1,t}\mathbb{I}_{R_1}\Qb^\mathsf{*T}\Big\|_{\oper}\lesssim \sqrt{\frac{d\log d}{n}}.
		\end{align}
	\end{lemma}
	\begin{proof}
		See  \S\ref{prooflemmamat} for a detailed proof.
	\end{proof}

	Combining  \eqref{errordymat},  \eqref{singdymat} and \eqref{singdymat_w} above, we control the difference between $\beta_t$ and  $\mu^{*}\beta^{*}$ as
	\begin{align*}
	\big\|\beta_{t}-\mu^{*}\beta^{*}\big\|^2_{F}&\le 2\Big\|\Qb^*\mathbb{I}_{R_0}\beta_{1,t}\Qb^{*\top}-\Pb^*\Rb^{*}\mathbb{I}_{R_0\cup R_1}\Pb^{*\top}\Big\|_{F}^2+2\Big\|\Qb^*\mathbb{I}_{R_1}\beta_{1,t}\Qb^{*\top}\Big\|_{F}^2\\&\quad+2\Big\|\Qb^*\mathbb{I}_{R^c}\beta_{1,t}\Qb^{*\top}\Big\|_{F}^2\\
	&\lesssim \frac{(r_0+r_1)\cdot d\log d}{n}.
	\end{align*}
	Here the inequality holds for all $t$ in the interval
$$ \bigl [a_5/ \big ( \eta(|\mu^{*}|r_m-M_m\sqrt{d\log d/n })\bigr )\cdot \log(1/\alpha),~ a_6\log(1/\alpha)\cdot \gamma^*/(\eta M_m ) \bigr ].
$$
In addition, in order to make sure such a time interval exists, it suffices to choose $C_{ms}$ such that  $C_{ms}\ge \max\{(a_5/a_6+1)M_m/|\mu^{*}| ,2M_m/|\mu^{*}|\}.$
	The proof   of statistical rate of    the normalized estimator is almost the same as  that in \S\ref{gensig}, so we  omit the  details for brevity.
	\end{proof}

\subsection{Proof of Theorem \ref{rankconsist}}
\begin{proof}
	\revise{The conclusion of Theorem \ref{rankconsist} follows directly from results of Lemma \ref{matrixerror}, Lemma \ref{matrixsignal} and Lemma \ref{matrixsignal_w}.}
	\end{proof}

 In the following subsection, we will prove Lemmas \ref{matrixerror}, \ref{matrixsignal},  and   \ref{matrixsignal_w} respectively. 
\subsubsection{Proof Idea of Lemma \ref{matrixerror}, Lemma \ref{matrixsignal} and Lemma \ref{matrixsignal_w}}\label{prooflemmamat}
\begin{proof}
	As $\Eb_{w,t}, \Eb_{v,t}, \Ub_{w,t}, \Ub_{v,t},\mathbf{S}_{w,t},\mathbf{S}_{v,t}$ are all diagonal matrices, then our proof of Lemma \ref{matrixerror}, Lemma \ref{matrixsignal} and Lemma \ref{matrixsignal_w} are relaxed to the proof of Lemma \ref{prop1.3}, Lemma \ref{prop1.4} and Lemma \ref{prop1.5}. The only difference between them lies on the concentration in spectral norm between   $\Mb^{*}:=\frac{1}{2n}\sum_{i=1}^{n}y_i\Xb_i+\frac{1}{2n}\sum_{i=1}^{n}y_i\Xb_i^\top=\Qb^{*}\Sigma^{*}\Qb^{*}$ and the true signal $\mu^{*}\beta^{*}:=\Pb^*\Rb^*\Pb^{*\top}$. We will depict this concentration upper bound in the following Lemma \ref{propgauss}.
	\begin{lemma}\label{propgauss}
		With probability $1-1/(2d)-3/n^2$, we have
		\begin{align*}
			\bigg\|\frac{1}{n}\sum_{i=1}^{n} \bigl ( \Xb_iy_i-\mathbb{E}[\Xb_iy_i] \bigr )\bigg\|_{\oper}\le M_m\sqrt{\frac{d\log d}{n}} ,
		\end{align*}
		 where $M_m$ is proportional to $\max\{\|f\|_{\psi_2},\sigma\}.$ 
	\end{lemma}

	Combining our result in Lemma \ref{propgauss} and Wely's inequality, we have $\Sigma^{*}$ is an entrywise perturbation of $\Rb^*$ with a perturbation upper bound of order $\cO(\sqrt{d\log d/n})$.
	
Then for Lemma \ref{matrixerror}, by using similar induction hypothesis given in proving Lemma \ref{prop1.3}, we verify that there exists an absolute constant $a_6$ such that we obtain upper bounds in spectral norm  for error components as $\|\Eb_{w,T}\|_{\oper}\lesssim \sqrt{\alpha},$ $\|\Eb_{v,T}\|_{\oper}\lesssim \sqrt{\alpha},$ for all $t\le T= a_6\log(1/\alpha){\gamma^*}/{(\eta M_m)}$.   Then we claim our conclusion of Lemma \ref{matrixerror}.

	 For Lemma \ref{matrixsignal}, by following similar proof procedures given in Lemma \ref{prop1.4} and our definition of set $R_0$, we have
	 \begin{align}\label{stconcen}
	  \|\Qb^*\beta_{1,t}\mathbb{I}_{R_0}\Qb^{*\top}-\Qb^*\Sigma^{*}\mathbb{I}_{R_0}\Qb^{*\top}\|_{\oper}=\|\beta_{1,t}\mathbb{I}_{R_0}-\Sigma^{*}\mathbb{I}_{R_0}\|_{\oper}\lesssim \sqrt{\frac{d\log d}{n}},
	  \end{align}
  for all   $t\ge a_5\log(1/\alpha)/[\eta(\mu^{*}r_m-M_m\sqrt{d\log d/n})]$. Then  we further  have
  \begin{align}
   \|\Qb^*\Sigma^*\II_{R_0}\Qb^{*\top}-\Pb^*\Rb^*\II_{R_0\cup R_1}\Pb^{*\top}\|_{\oper}&=\|\Qb^*\Sigma^{*}\Qb^{*\top}-\Pb^*\Rb^{*}\Pb^{*\top}-\Qb^*\Sigma^{*}\II_{R_1\cup R^{c}}\Qb^{*\top}\|_{\oper}\nonumber\\
   &\le\|\Qb^*\Sigma^{*}\Qb^{*\top}-\Pb^*\Rb^{*}\Pb^{*\top}\|_{\oper}+\|\Qb^*\Sigma^{*}\II_{R_1\cup R^{c}}\Qb^{*\top}\|_{\oper}\nonumber\\
   &\lesssim \sqrt{\frac{d\log d}{n}}, \label{weakconcen2}
   \end{align}
where the last inequality follows from Lemma \ref{propgauss} and Wely's inequality and our definition on $R,R_0,R_1$. After combining our results in \eqref{stconcen}-\eqref{weakconcen2}, we  complete our proof of Lemma \ref{matrixsignal}.

 In terms of Lemma \ref{matrixsignal_w}, the proof   is similar to that of Lemma \ref{prop1.5}, so we omit the corresponding details here for brevity.
\end{proof}

\subsubsection{Proof of Lemma \ref{propgauss}}
\begin{proof}
	For any fixed $n$ and $d$, first, we denote event $C_i,i\in[n]$ as
	\begin{align}\label{defici}
	C_i:=\mathbb{I}\left\{ |y_i|\le\sigma_y\sqrt{6\log n},\|\Xb_i\|_{\oper}\le 3\left(\sqrt{d}+3\sqrt{{\log d}/{\log(3/2)}}\right)+2\sqrt{3\log n}  \right\}.
	\end{align}
In order to illustrate that with high probability, $|y_i|$, and $\|\Xb_i\|_{\oper}$ lie in the support set of $C_i$ for all $i\in[n]$, we first introduce the following two Lemmas, namely Lemma \ref{subgvec} and Lemma \ref{randmat}.
	
\begin{lemma} \label{subgvec}
	We get a union upper bound for $\{|y_i|\}_{i=1}^{n}$, to be more specific, with probability $1-2/n^2$ we obtain
	$
		\max_i \left|y_i\right|\le \sigma_y\sqrt{6\log n}
	$, where $\sigma_y$ is the sub-Gaussian norm of $y_i$ and it is proportional to $\max\{\|f\|_{\psi_2},\sigma\}$.
\end{lemma}
\begin{proof}
	The proof is straight forward by sub-Gaussian tail bound, please refer to Proposition 2.5.2 in \cite{vershynin_2018} for more details.
\end{proof}
\begin{lemma}\label{randmat}
	For $n$ independent random matrices $\Xb_i\in\mathbb{R}^{d\times d},i\in[n]$ with independent standard normal entries, 	   with probability $1-1/n^2$, we have
	\begin{align*}
		\max_{i\in [n]}\left\|\Xb_i\right\|_{\oper}\le 3\Big( \sqrt{d}+3 \sqrt{{\log d}/{\log(3/2)}} \Big)+2\sqrt{3\log  n}.
	\end{align*}

\end{lemma}
\begin{proof}
	By Corollary 3.11 in \cite{vanhandel2016}, we have
	\begin{align*}
	\PP\Big[\|\Xb_i\|_{\oper}\ge\left(1+\epsilon\right)\Big( 2\sqrt{d}+{6}\sqrt{{\log d}/{\log(1+\epsilon)}}\Big)+t\Big]\le e^{-t^2/4},
	\end{align*}
	for any $0<\epsilon\le 1/2$, $t\ge 0$ and $i\in[n]$. Taking $\epsilon=1/2$, we get a tail bound for  $\max_{i\in[n]}\|\Xb_i\|_{\oper}$~as
	\begin{align*}
		&\PP\Big[\max_{i\in[n]} \|\Xb_i\|_{\oper}\ge   \frac{3}{2}\Big( 2\sqrt{d}+6 \sqrt{{\log d}/{\log(3/2)}}\Big)+t     \Big]\\
		&\quad \le n\cdot\PP\Big[\|\Xb_i\|_{\oper}\ge 3\Big( \sqrt{d}+3 \sqrt{{\log d}/{\log(3/2)}}\Big)+t\Big] \le n\cdot e^{-t^2/4}=e^{-t^2/4+\log n}.
	\end{align*}
	By choosing $t=2\sqrt{3\log n}$, we have
	\begin{align*}
		\max_{i\in [n]}\left\|\Xb_i\right\|_{\oper}\le 3\bigg( \sqrt{d}+3 \sqrt{{\log d}/{\log(3/2)}}\bigg)+2\sqrt{3\log  n},
	\end{align*}
	with probability $1-1/n^2$, which completes the proof of Lemma \ref{randmat}.
\end{proof}
From the Lemma \ref{subgvec} and Lemma \ref{randmat} given above, we obtain
\begin{align}
\PP\left(C_i^{c}\right)&\le\PP\Big(\bigcup_{i}C_i^{c}\Big)\nonumber\\&\le\PP\left[\max_{i\in [n]}\left\|\Xb_i\right\|_{\oper}\ge 3\left(\sqrt{d}+3\sqrt{{\log d}/{\log(3/2)}}\right)+2\sqrt{3\log n}\right]\nonumber\\
&\qquad+\PP\left(\max_i \left|y_i\right|\ge \sigma_y\sqrt{6\log n}\right)\le\frac{3}{n^2}\label{ineq3}.
\end{align}
We further denote event $A$ as
	\begin{align*}
	A=\mathbb{I}\left\{\bigg\|\frac{1}{n}\sum_{i=1}^{n}\Xb_iy_i\cdot\mathbb{I}_{C_i}-\mathbb{E}[\Xb_1y_1\cdot\mathbb{I}_{C_1}]\bigg\|_{\oper}\ge \frac{t}{2}\right\}.
	\end{align*}
	Then we have
	\begin{align*}
\PP\left(\bigg\|\frac{1}{n}\sum_{i=1}^{n}\Xb_iy_i-\mathbb{E}[\Xb_1y_1]\bigg\|_{\oper}\ge t    \right)&\le \PP\left(A\right)+\PP\bigg(\frac{1}{n}\sum_{i=1}^{n}\Xb_iy_i\cdot \mathbb{I}_{C_i}\neq\frac{1}{n}\sum_{i=1}^{n}\Xb_iy_i\bigg)\\&\quad+\PP\left(  \left\|\mathbb{E}[\Xb_1y_1\cdot \mathbb{I}_{C_1^c}]\right\|_{\oper}\ge \frac{t}{2}\right)\\&:=\mathbf{(I)}+\mathbf{(II)}+\mathbf{(III)}.
	\end{align*}
First, we obtain an upper bound for term $\mathbf{(II)}$ according to \eqref{ineq3} as
	\begin{align}\label{term2}
	\mathbf{(II)}=\PP\bigg(\frac{1}{n}\sum_{i=1}^{n}\Xb_iy_i\cdot \mathbb{I}_{C_i}\neq\frac{1}{n}\sum_{i=1}^{n}\Xb_iy_i\bigg)=\PP\bigg(\bigcup_i C_i^{c}\bigg)\le \frac{3}{n^2}.
	\end{align}
Next, in order to bound term \textbf{(I) }, $\PP(A)$, we first figure out the spectral upper bound of $$\frac{1}{n}\sum_{i=1}^{n}\Xb_iy_i\cdot \mathbb{I}_{C_i}-\mathbb{E}[\Xb_1y_1\cdot \mathbb{I}_{C_1}].
	$$
	 By the definition of $C_i$ given in \eqref{defici}, for any fixed $n,d$, with probability 1 we have
	\begin{align*}
	&\left\|\Xb_iy_i\cdot \mathbb{I}_{C_i}\right\|_{\oper}\\&\quad\le U:={3\sqrt{6}\sigma_y}\left( \sqrt{d\log n}+3\sqrt{{\log d\cdot \log n}/{\log(3/2)}}\right)+6\sqrt{2}\sigma_y\log n.
	\end{align*}
	By denoting $\Zb_i$ as $\Zb_i=\Xb_iy_i\cdot\mathbb{I}_{C_i}-\mathbb{E}[\Xb_iy_i\cdot\mathbb{I}_{C_i}],$ we have $\|\mathbb{E}[\Zb_i\Zb_i^\mathsf{T}]\|_{\oper}\lesssim\|\mathbb{E}[y_i^2\Xb_i\Xb_i^\mathsf{T}]\|_{\oper}\lesssim \sigma_y^2 d$.
	Furthermore, by letting
	\begin{align*}
	\sigma_z=\max\bigg\{\bigg\|\frac{1}{n}\sum_{i=1}^{n}\mathbb{E}[\Zb_i\Zb_i^\mathsf{T}]\bigg\|_{\oper}^{1/2},  \bigg\|\frac{1}{n}\sum_{i=1}^{n}\mathbb{E}[\Zb_i^\mathsf{T}\Zb_i]\bigg\|_{\oper}^{1/2}\bigg\},
	\end{align*}
	we  get $\sigma_z\lesssim \sigma_y\cdot \sqrt{d}$.
	Then, after applying matrix Bernstein inequality from Proposition 1 in \cite{KLT2011}, we have that
	\begin{equation}\label{matbern}
	\bigg\|\frac{1}{n}\sum_{i=1}^{n}\Xb_iy_i\cdot \mathbb{I}_{C_i}-\mathbb{E}[\Xb_1y_1\cdot \mathbb{I}_{C_1}]\bigg\|_{\oper}\le 2 \max\bigg\{\sigma_z\sqrt{\frac{2\log(2d)}{n}} , U\cdot\frac{2\log(2d)}{n}\bigg\}
	\end{equation}
	holds with probability $1-1/(2d)$.
	
	For term $(\mathbf{III})$, likewise, we first get the spectral norm of $\mathbb{E}[\Xb_1y_1\cdot \mathbb{I}_{C_1^c}].$ For any unit vector $\ub,\vb\in \mathbb{R}^{d_1}$, by Cauchy-Schwartz inequality we have
	\begin{align*}
		\mathbb{E}\big[\ub^\mathsf{T}\Xb_1\vb y_1\cdot \mathbb{I}_{C_1^c}\big]\le\sqrt{\mathbb{E}\big[\big(\ub^\mathsf{T}\Xb_1 \vb\big)^2\big]\cdot\mathbb{E}\left[y_1^2\mathbb{I}_{C_1^c}\right]}.
	\end{align*}
	As all elements of $\Xb_1$ are independent standard Gaussian variables, then we get
	\begin{align}\label{third_term1}
		\mathbb{E}\big[(\ub^\mathsf{T}\Xb_1\vb)^2\big]=\sum_{i,j=1}^nu_i^2v_j^2=\left\|\ub\right\|_2^2\cdot\left\|\vb\right\|_2^2=1.
	\end{align}
		In addition, as we have assumed $\{y_i\}_{i=1}^{n}$ are i.i.d. sub-Gaussian random variables with sub-Gaussian norm $\sigma_y$, so we obtain
	\begin{align}\label{third_term2}
		\sqrt{\mathbb{E}\left[y_1^2\cdot \mathbb{I}_{C_1^c}\right]}\le (\mathbb{E}\left[y_1^4\right])^{1/4}\cdot \mathbf{P}(C_1^c)^{1/4}\lesssim \frac{\sigma_y}{\sqrt{n}} .
	\end{align}

	Next, after setting
	\begin{equation}\label{t}
	t=4 \max\bigg\{\sigma_z\sqrt{\frac{2\log(2d)}{n}} , U\cdot\frac{2\log(2d)}{n}\bigg\},
	\end{equation}
	 we have $ t=\Omega(\sqrt{d}\sigma_y/\sqrt{n})$. So for term $(\mathbf{III})$, after combining \eqref{third_term1}, \eqref{third_term2} we obtain
	\begin{align}\label{term3}
	\PP\left(\left\|\mathbf{E}\left[\Xb_1y_1\cdot\mathbb{I}_{C_1^c}\right]\right\|_{\oper}\ge t/2\right)=0.
	\end{align}
	For term $(\mathbf{I})$, by \eqref{matbern} and the definition of $t$ given in  \eqref{t}, we  get
	\begin{align}\label{p_a}
\PP(A)\le \frac{1}{2d}.
	\end{align}
	Thus,  combining our conclusions from \eqref{term2}, \eqref{term3} and \eqref{p_a}, we finally obtain that
	\begin{align*}
	\bigg\|\frac{1}{n}\sum_{i=1}^{n}\Xb_iy_i-\mathbb{E}[\Xb_1y_1]\bigg\|_{\oper}\le 4 \max\bigg\{\sigma_z\sqrt{\frac{2\log(2d)}{n}} , U\cdot\frac{2\log(2d)}{n}\bigg\}
	\end{align*}
	holds with probability $1-1/(2d)-3/n^2$ with $U={3\sqrt{6}\sigma_y}\left( \sqrt{d\log n}+3\sqrt{{\log d\cdot \log n}/{\log(3/2)}}\right)+6\sqrt{2}\sigma_y\log n$, $\sigma_z\lesssim \sqrt{d}\sigma_y$ and $\sigma_y$  being proportional to $\max\{\|f\|_{\psi_2},\sigma\}$.
	By our assumption that when $\sigma_yd\log d\lesssim n=o(d^2)$, we have $\max\{\sigma_z\sqrt{{2\log(2d)}/{n}} , U\cdot{2\log(2d)}/{n}\}\lesssim \sigma_y\sqrt{d\log d/n}$, so we conclude the proof of Lemma \ref{propgauss} after setting $M_m$ to be proportional to $\sigma_y$.
\end{proof}
\subsection{Proof of Theorem \ref{genthmmat1}}
\begin{proof}
		The proof of Theorem \ref{genthmmat1} is similar to the proof of Theorem \ref{thmmat1}. We need to replace $\frac{1}{2n}\sum_{i=1}^{n}y_i\Xb_i$ with $\frac{1}{2n}\sum_{i=1}^{n}
		\mathcal{H}( y_iS(\Xb_i),\kappa)$ in \eqref{matup1}-\eqref{matup3}. The definition of $\mathcal{H}( y_iS(\Xb_i),\kappa)$ is given \S\ref{secgenmat}. In this case, we define $\Mb_2^{*}$ as
		 $$ \Mb_2^{*}=\frac{1}{2n}\sum_{i=1}^{n}\mathcal{H}( y_iS(\Xb_i),\kappa)+\frac{1}{2n}\sum_{i=1}^{n}\mathcal{H}( y_iS(\Xb_i),\kappa)^\top  $$
		  and the spectral decomposition of $\Mb_2^{*}$ as $\Mb_2^{*}:=\Qb_2^{*}\Sigma_2^{*}\Qb_2^{*\top}.$  We then let $\Wb_{2,t}=\Qb_2^{*\top}\Wb_t\Qb_2^{*}$ and $\Vb_{2,t}=\Qb_2^{*\top}\Vb_t\Qb_2^{*}.$ The corresponding gradient updates with respect to $\Wb_{2,t}$ and $\Vb_{2,t}$ are given~by
\begin{align*}
\Wb_{2,t+1}&=\Wb_{2,t}-\eta\Big(\Wb_{2,t}\Wb_{2,t}^\top-\Vb_{2,t}\Vb_{2,t}^\top-\Sigma_2^{*}\Big)\Wb_{2,t},\\
\Vb_{2,t+1}&=\Vb_{2,t}+\eta\Big(\Wb_{2,t}\Wb_{2,t}^\top-\Vb_{2,t}\Vb_{2,t}^\top-\Sigma_2^{*}\Big)\Vb_{2,t},\\
\beta_{2,t+1}&=\Wb_{2,t+1}\Wb_{2,t+1}^\top-\Vb_{2,t+1}\Vb_{2,t+1}^\top.
\end{align*}
By selecting $\kappa$ properly, the following Lemma \ref{matrixconc} gives a concentration between our new estimator and $\mu^{*}\beta^{*}$.
 \begin{lemma}\label{matrixconc} Suppose $y_i=f(\langle \Xb_i,\beta^{*}\rangle)+\epsilon$, $\Xb_i\in \RR^{d\times d},$ and entries of $\Xb_i$ are i.i.d.\,\,random variables with density function $p_0(x)$. Under assumptions in Theorem \ref{genthmmat1} we have
	\begin{align*}
		\left\|\frac{1}{n}\sum_{i=1}^{n}\mathcal{H}( y_iS(\Xb)_i,\kappa)-\mathbb{E}[Y\cdot S(\Xb_1)]\right\|_{\oper}\le M_{mg} \cdot \sqrt{\frac{d\log d}{n}}
	\end{align*}
	holds with probability $1-(2d)^{-2}$ where $M_{mg}$ is an absolute constant only depending on $M$ given in Assumption \ref{genass1mat}.
\end{lemma}
\begin{proof} Please see \S\ref{proofmatconc} for the detailed proof.\end{proof}
Thus, after following the same proof procedures of Lemma \ref{matrixerror}, \ref{matrixsignal} and Lemma \ref{matrixsignal_w}, we  claim our conclusion of Theorem \ref{genthmmat1}. Next, we will give a detailed proof of Lemma \ref{matrixconc}.
	\end{proof}
 \subsubsection{Proof of Lemma \ref{matrixconc}}\label{proofmatconc}
 \begin{proof}
 	Before applying results in \cite{Minsker2018}, we need to get an upper bound of $$ \|\mathbb{E}[y_1S(\Xb_1)-\mu^{*}\beta^{*}][y_1S(\Xb_1)-\mu^{*}\beta^{*}]^\mathsf{T}\|_{\oper} $$ and it is sufficient for us to bound $\|\mathbb{E}[y_1^2\cdot S(\Xb_1)S(\Xb_1)^{\mathsf{T}}]\|_{\oper} .$
 Then for any unit vector $\mathbf{u}\in \RR^{d\times 1}$ we have
 	\begin{align*}
 	\mathbb{E}\Big[y_1^2 \ub^{\mathsf{T}}\cdot S(\Xb_1)S(\Xb_1)^{\mathsf{T}}\cdot \ub\Big]&=\mathbb{E}\Big[y_1^2\sum_{i=1}^{d}(\ub^\mathsf{T}S(X_1)_{[:,i]})^2\Big]=\sum_{i=1}^{d}\mathbb{E}\left[y_1^2(u^\mathsf{T}S(\Xb_1)_{[:,i]})^2\right]\\
 	&\le\sum_{i=1}^{d}\sqrt{\mathbb{E}\left[y_1^4\right]\cdot \mathbb{E}\Big[\big(\mathbf{u}^\mathsf{T}S(\Xb_1)_{[:,i]}\big)^4\Big] }\le d_1\cdot \sqrt{M}\cdot \sqrt{\mathbb{E}\Big[\big(\sum_{k=1}^{d}u_kS(\Xb_1)_{[k,1]}\big) ^4\Big] }.
 	\end{align*}
 	In order to get an upper bound of term $\mathbb{E}\big[(\sum_{k=1}^{d}u_kS(\Xb_1)_{[k,1]}) ^4 \big]$, we need to take advantage of the independence property between entries of $\Xb_1$, so that we get
 	\begin{align*}
 	\mathbb{E}\Big[(\sum_{k=1}^{d}u_kS(\Xb_1)_{[k,1]}) ^4 \Big]&=\sum_{i,j=1}^{d}u_i^2u_j^2\mathbb{E}\Big[S(\Xb_1)_{[i,1]}^2S(\Xb_1)_{[j,1]}^2\Big]\\
 	&\le\sum_{i,j=1}^{d}u_i^2u_j^2\sqrt{\mathbb{E}\Big[S(\Xb_1)_{[i,1]}^4\Big]\mathbb{E}\Big[S(\Xb_1)_{[j,1]}^4\Big]} \le M\sum_{i,j=1}^{d}u_i^2u_j^2=M.
 	\end{align*}
 	The last inequality follows from our Assumption \ref{genass1mat}.
 	Then we get an upper bound for $\|\mathbb{E}(y_1^2\cdot \Xb_1\Xb_1^{\mathsf{T}})\|_{\oper}$ as
 	\begin{align*}
 	\left\|\mathbb{E}\Big[y_1^2\cdot S(\Xb_1)S(\Xb_1)^{\mathsf{T}}\Big]\right\|_{\oper}\le d\cdot M.
 	\end{align*}
 	Similarly, we also get an upper bound for term
 	$ \|\mathbb{E}[y_1^2\cdot S(\Xb_1)^\mathsf{T} S(\Xb_1)]\|_{\oper}$ as $$\|\mathbb{E}[y_1^2\cdot S(\Xb_1)^\mathsf{T}S(\Xb_1)]\|_{\oper}\le d\cdot M.$$
 	By applying Corollary 3.1 in \cite{Minsker2018}, we get the following inequality
 	\begin{align}\label{genineqmat}
 	\mathbb{P}\bigg( \bigg\| \frac{1}{n}\sum_{i=1}^{n}\mathcal{H}( y_iS(\Xb_i), \kappa)-\mathbb{E}[y_1S(\Xb_1)]    \bigg\|_{\oper}\ge t\bigg)\le 4d\exp\left(-\kappa t\cdot n+\frac{\kappa^2\sigma_n^2}{2}\right),
 	\end{align}
 where  $$\sigma_n^2=\max \left (\|\sum_{i=1}^{n}\mathbb{E}[y_i^2\cdot S(\Xb_i)S(\Xb_i)^\mathsf{T}]\|_{\oper},\| \sum_{j=1}^{n}\mathbb{E}[y_j^2\cdot S(\Xb_j)^\mathsf{T}S(\Xb_j)]\|_{\oper} \right )\le 2d \cdot M\cdot n.
 $$ Here we choose $t=4\sqrt{(d\cdot M\log(4d))/n} $, and we further let $\kappa=\sqrt{{\log(4d)}/({n\cdot d\cdot M})}$ in \eqref{genineqmat}, so that we obtain
 	\begin{align*}
 	\left\| \frac{1}{n}\sum_{i=1}^{n}\mathcal{H}( y_iS(\Xb_i),\kappa)-\mathbb{E}\left[y_1S(\Xb_1)\right]    \right\|_{\oper}\le 4\sqrt{M}\sqrt{\frac{d\log(4d)}{n}},
 	\end{align*}
 	 with probability $1-(4d)^{-2}$ . Then we complete our proof of Lemma \ref{matrixconc}.
 \end{proof}
\subsection{Algorithm in \S\ref{secgenmat}}\label{alggenmat}
\begin{algorithm}[H]
	\KwData{Training design matrix $\mathbf{X}_i \in \mathbb{R}^{d\times d}$, $i\in[n]$, response variables $\{y_i\}_{i=1}^{n}$, truncating parameter $\kappa$, initial value $\alpha$ and step size $\eta$;}
	Initialize  $\Wb_0=\alpha\cdot\mathbb{I}_{d\times d}$, $\Vb_0=\alpha\cdot\mathbb{I}_{d\times d}$ and set iteration number $t=0$;\\
	\While{$t<T_1$}{
		$\mathbf{W}_{t+1}=\mathbf{W}_{t}-\eta(\mathbf{W}_t \mathbf{W}_t^\top-\mathbf{V}_t \mathbf{V}_t^\top-\frac{1}{2n}\sum_{i=1}^{n}\mathcal{H}(y_i S(\Xb_i),\kappa)-\frac{1}{2n}\sum_{i=1}^{n}\mathcal{H}(y_i S(\Xb_i),\kappa)^\top) \mathbf{W}_t$;\\
		$\,\mathbf{V}_{t+1}\,=\mathbf{V}_{t}\,+\eta(\mathbf{W}_t \mathbf{W}_t^\top-\mathbf{V}_t \mathbf{V}_t^\top-\frac{1}{2n}\sum_{i=1}^{n}\mathcal{H}(y_i S(\Xb_i),\kappa)-\frac{1}{2n}\sum_{i=1}^{n}\mathcal{H}(y_i S(\Xb_i),\kappa)^\top) \mathbf{V}_t$; \\
		$\,\,\,\beta_{t+1}\,=\mathbf{W}_t \mathbf{W}_t^\top-\mathbf{V}_t \mathbf{V}_t^\top$;\\
		$\quad t\,=t+1$;\\
	}
	\KwResult{ Output the final estimate $\widehat \beta=\beta_{T_1}$.
	}\label{alg6}
	\caption{ Algorithm for Low Rank Matrix SIM with General Design}
\end{algorithm}
\subsection{Comparison with \cite{Li2017AlgorithmicRI}}\label{compare_li}
\revise{Compared  with \cite{Li2017AlgorithmicRI}, our strengths are four-fold. First, under the setting of standard Gaussian design with signals at constant level, our sample complexity is only at the order of $\tilde{\cO}(rd)$ whereas they need at least $\tilde{\cO}(r^2d)$ samples so as to establish their RIP condition \citep{candes2008restricted}.
Second, our results also hold under the existence of  weak signals, i.e. $0<\min_{i\in R}|r_i^{*}|\lesssim\tilde{\cO}(\sqrt{d/n})$.  When we fix $d$ and $r$, in order to meet the RIP condition with parameter $\delta$, the sample size  $n$ needs to satisfy $n\gtrsim \cO(1/\delta^2)$ according to Theorem 4.2 in \cite{MinirankRecht2010}. As \cite{Li2017AlgorithmicRI} requires an RIP parameter $\delta$ with $\delta\lesssim {\cO}(\min_{i\in R}|{r_i^{*}}|^3/\sqrt{r})$ in its Theorem 1, the corresponding minimum signal strength $\min_{i\in R}|r_i^{*}|$ should satisfy $\min_{i\in R}|r_i^{*}|\gtrsim \cO((1/n)^{1/6})$ which brings a stronger assumption than us. }\revise{Third, when we study the noisy statistical model,
	compared with the result in  \cite{Li2017AlgorithmicRI},  our statistical rate is independent with the conditional number and is optimal up to logarithmic factors.  To be more specific, when they study $y=\langle \Xb,\beta^{*}\rangle+\sigma$, with $\sigma\sim N(0,1)$, \cite{Li2017AlgorithmicRI}'s final statistical convergence rate is $\|\beta_{T_1}-\beta^{*}\|_F^2\lesssim \sigma^2rd\log d/(nr_m)$ by following their methodology, which is suboptimal when $r_m$ is close to the lower bound of the strong signal set $R_0$, i.e. $\sqrt{d\log d/n}$.  Fourth, as mentioned above, we allow  a more general class of statistical models and symmetric signal matrices. }

\revise{Next, we briefly describe the proof difference between our work with \cite{Li2017AlgorithmicRI}. In terms of the RIP condition, we leverage the score transformation and study the loss function given in \eqref{matoverpara}, whose gradient not necessarily requires us to approximate the design matrix to identity matrix, thus we do not need the RIP condition and reduce sample complexity in \cite{Li2017AlgorithmicRI} to $\tilde{\cO}(rd).$ In terms of the optimal statistical rate under the noisy setting, the major difference lies in the methodology of trajectory analysis for eigenvalues. To be more clear, \cite{Li2017AlgorithmicRI} divide their analysis into two stages. In the first stage, they prove their minimum absolute eigenvalue of $\beta_t$ in the strong signal set $R_0$ grows exponentially until exceeding a threshold $r_m/2$, meanwhile the maximum eigenvalue of $\beta_t$ in $R^{c}$ remains bounded with order $\cO(1/d).$ In the second stage, they prove that their $\beta_t$ falls in a contraction region
	\begin{align*}
		\|\beta_{t+1}-\beta^{*}\|^2_F\le(1-\eta r_m) \|\beta_t-\beta^{*}\|^2_F+\cO(\eta \sigma^2{{dr \log d}/{n}}).
	\end{align*}
	This results in $\|\beta_{t+1}-\beta^{*}\|^2_F=\cO(\sigma^2dr\log d/(r_m n))$ after logarithmic iterations. The second stage ends until the maximum magnitude of eigenvalues in set $R^c$ exceeds $\cO(1/d)$.  Instead of only dividing the trajectory analysis into only two stages and studying the dynamics of the minimum eigenvalue in $R_0$, we study the dynamics of eigenvalues entrywisely in $R_0$ with multiple stages. To be more clear,  in the $k$-th stage with $k \ge 1$ and $\epsilon=\cO(\sqrt{d\log d/n})$, for all $i$ in $R_0$, we prove that the $i$-th eigenvalue of $\beta_t$ in magnitude, $|r_{t,i}|$, exceeds $|\mu^{*}r_i^{*}(1-1/2^k)+\epsilon|$ after logarithmic steps. Thus, when $k\ge \log_2(|r_i^{*}|/\epsilon)$ we obtain $|r_{t,i}|\ge |\mu^{*}r_i^{*}-\epsilon|$. After that, we use mathematical induction to prove that our iterates $\beta_t$ falls in the region with $\|\beta_t-\mu^{*}\beta^{*}\|_F^2\le r\epsilon^2$ as long as maximum magnitude of eigenvalues in set $R^c$ does not exceed $\cO(1/d)$. This finally gives a tighter $\ell_2$-statistical rate which is independent with the condition number $(1/r_m)$ compared with \cite{Li2017AlgorithmicRI}'s error bound under the noisy setting. Please kindly refer to \S\ref{proofmat} for more details.
}

 \section{Extension to One-bit Compressed Sensing}\label{secexamp}

 As a concrete example, in the following, we consider the one-bit compressed sensing model
 \citep{Jacques2013onebit, plan2013onebit}.
The response variables and the covariates satisfy
 \begin{align*}
 	y_i=\text{sign}(\langle \mathbf{x}_i, \beta^{*}\rangle)+\epsilon,\qquad \forall i\in[n],
 \end{align*}
 where $\text{sign}(x)=1$ for all  $x\ge 0$ and  $\text{sign}(x)=-1,\text{ for } x<0$,  and $n$ is the number of our observations.
 Moreover, for both the vector and matrix settings, we assume that each entry of $\xb_i$ are i.i.d. $N(0, 1)$ random variables and $\{\epsilon_i\}_{i\in[n]}$ are i.i.d. sub-Gaussian random variables. As $\{y_i\}_{i\in[n]}$ doesn't convey any information about the length of our signal $\beta^{*}$, we are only able to recover the direction of $\beta^{*}$ by utilizing measurements $\{\xb_i,y_i\}_{i\in [n]}$.
 By following iterating procedures in Algorithm \ref{alg3} and Algorithm \ref{alg4}, we next summarize our theoretical results into the following Corollary \ref{signcor}.
 \begin{corollary}\label{signcor}
 	In the scenario of vector SIM,
 	in   Algorithm \ref{alg3},  we  let the  initial value $\alpha$ satisfy $0<\alpha\le  {M^2_{\textrm{sgn}}}/{p}$ and set stepsize $\eta$ as such that    $0<\eta \le{1}/[12(\sqrt{\pi/2}+M_{\textrm{sgn}})]$   with $M_{\textrm{sgn}}$ being a constant proportional to $\max\{\|\textrm{sign}(\langle \xb_i,\beta^{*} \rangle)\|_{\psi_2}, \sigma\}$.
 	Then there exist absolute constants $a_9,a_{10}$ such that,
 	with probability at least
 	$1-2p^{-1}-2n^{-1},$
 	we have
 	\begin{align*}
 		\left \|\frac{\beta_{t}}{\|\beta_t\|_2}-\beta^{*}\right\|^2_2\lesssim \frac{s_0\log n }{n}+\frac{s_1\log p }{n}
 	\end{align*}
 	for all $t\in[a_9{\log(1/\alpha)}/(\eta(\sqrt{\pi/2}s_m-M_{\textrm{sgn}}\sqrt{\log p/n})) ,a_{10}\log(1/\alpha)\sqrt{n/\log p}/(\eta M_{\textrm{sgn}})]$.

Moreover, for the case of low rank matrix recovery,
in  Algorithm \ref{alg4},
we choose $\alpha$ with $0<\alpha\le {M^2_{\textrm{sgn}}}/{d}$ and the  stepsize $\eta$ satisfying   $0<\eta \le{1}/[12(\sqrt{\pi/2}+M_{\textrm{sgn}})]$. Then there exist absolute constants $a_{11},a_{12}$ such that,
with probability at least   $1-1/(2d)-3/n^2$,
we have
 	\begin{align*}
 		\left\|\frac{\beta_{t}}{\|\beta_t\|_F}-\beta^{*}\right\|^2_F\lesssim \frac{rd \log d}{n}
 	\end{align*}
 for all  $t\in[a_{11}{\log(1/\alpha)}/ ( \eta({\sqrt{\pi/2}r_m}-M_{\textrm{sgn}}\sqrt{d\log d/n}) ) , a_{12}\log(1/\alpha)\sqrt{n/(d\log d)}/(\eta M_{\textrm{sgn}})]$. 
 \end{corollary}

 \begin{proof}
 	The proof of Corollary \ref{signcor} is straight forward by following the proof procedures of Theorem \ref{thmvec1} and Theorem \ref{thmmat1}, so we just omit relevant details here. The only difference between them is that we have $Y\cdot X$ as an unbiased estimator of $\sqrt{2/\pi}\beta^{*}$ by using properties of standard Gaussian distribution instead of Stein's lemma since $f(x)=\text{sign}(x)$ is not a differentiable function. The proof of this property can be found in Lemma 4.1 in \cite{plan2012robust}.
 \end{proof}
 Comparing to existed works on high dimensional one-bit compressed sensing \citep{plan2013onebit,goldstein2018structured,Thrampoulidis2018TheGL}, instead of adding $\ell_1$-regularizers and tuning parameters, here we are able to achieve minimax optimal (up to logarithmic terms) $\ell_2$-statistical rates under both settings of sparse vector and low rank matrix by simply running gradient descent on over-parameterized loss functions \eqref{loss1bk}, \eqref{matoverpara} and adopting early  stopping via out-of-sample prediction.

  \begin{algorithm}[htpb]
 	\KwData{Training data $\{\xb_i\}_{i=1}^{n}$  $\{y_i\}_{i=1}^{n}$, testing data $\{\xb'_{i}\}_{i=1}^{n}$, $\{y_i'\}_{i=1}^{n}$, initial value $\alpha$, step size $\eta$ and maximal iteration number $T_m$;}
 	Initialize variables $\wb_0=\alpha\cdot\mathbf{1}_{p\times 1}$, $\vb_0=\alpha\cdot\mathbf{1}_{p\times 1}$ and set iteration number $t=0$;\\
 	\While{$t<T_m$}{
 		$\mathbf{w}_{t+1}=\mathbf{w}_{t}-\eta(\mathbf{w}_t\odot \mathbf{w}_t-\mathbf{v}_t\odot \mathbf{v}_t-\frac{1}{n}\sum_{i=1}^{n}\xb_{i}y_i)\odot \mathbf{w}_t$;\\
 		$\,\mathbf{{v}}_{t+1}=\mathbf{v}_{t}\,+\eta(\mathbf{w}_t\odot \mathbf{w}_t-\mathbf{v}_t\odot \mathbf{v}_t-\frac{1}{n}\sum_{i=1}^{n}\xb_{i}y_i)\odot \mathbf{v}_t$; \\
 		$\,\beta_{t+1}=\mathbf{w}_t\odot \mathbf{w}_t-\mathbf{v}_t\odot \mathbf{v}_t;$\\
 		
 		$\,\, t=t+1$;\\
 	}
 	\KwResult{ Choose $\tilde{t}$ such that $\frac{1}{n}\sum_{i=1}^n[y'_{i}-f(\mathbf{x'}_{i}^\mathsf{T}\beta_t/\|\beta_t\|_2)]^2$$<\frac{1}{n}\sum_{i=1}^n[y'_{i}-f(\mathbf{x'}_{i}^\mathsf{T}\beta_{t+1}/\|\beta_{t+1}\|_2)]^2$ or $\frac{1}{n}\sum_{i=1}^n[y'_{i}-f(\mathbf{x'}_{i}^\mathsf{T}\beta_t/\|\beta_t\|_2)]^2$ is minimized over all iterations, then
 		output the final estimate $\widehat \beta=\beta_{\tilde{t}}$.
 	}\label{alg3}
 	\caption{ Algorithm for  Vector SIM with Known Link Function}
 \end{algorithm}

 \begin{algorithm}[htpb]
 	\KwData{Training data $\mathbf{X}_i \in \mathbb{R}^{d\times d}$, $i\in[n]$, $\mathbf{y} \in\mathbb R^n$, testing data $\Xb'_{i}\in\RR^{d\times d},\,i\in[n]$, $\mathbf{y}'\in\RR^{n}$, initial value $\alpha$, step size $\eta$ and maximal iteration number $T_m'$;}
 	Initialize  $\Wb_0=\alpha\cdot\mathbb{I}_{d\times d}$, $\Vb_0=\alpha\cdot\mathbb{I}_{d\times d}$ and set iteration number $t=0$;\\
 	\While{$t<T_m'$}{
 		$\mathbf{W}_{t+1}=\mathbf{W}_{t}-\eta(\mathbf{W}_t \mathbf{W}_t^\mathsf{T}-\mathbf{V}_t \mathbf{V}_t^\mathsf{T}-\frac{1}{2n}\sum_{i=1}^{n}\mathbf{X}_{i}y_i-\frac{1}{2n}\sum_{i=1}^{n}\mathbf{X}_{i}^\mathsf{T}y_i) \mathbf{W}_t$;\\
 		$\,\mathbf{V}_{t+1}\,=\mathbf{V}_{t}\,\,+\eta(\mathbf{W}_t \mathbf{W}_t^\mathsf{T}-\mathbf{V}_t \mathbf{V}_t^\mathsf{T}-\frac{1}{2n}\sum_{i=1}^{n}\mathbf{X}_{i}y_i-\frac{1}{2n}\sum_{i=1}^{n}\mathbf{X}_{i}^\mathsf{T}y_i) \mathbf{V}_t$; \\
 		$\,\,\,\beta_{t+1}\,=\mathbf{W}_t \mathbf{W}_t^\mathsf{T}-\mathbf{V}_t \mathbf{V}_t^\mathsf{T}$;\\
 		$\quad t\,=t+1$;\\
 	}
 	\KwResult{ Choose $\tilde{t}$ such that {
 		$\frac{1}{n}\sum_{i=1}^n[y'_{i}-f(\tr(\mathbf{X'}_{i}^\mathsf{T}\beta_t/\|\beta_t\|_F))]^2<\frac{1}{n}\sum_{i=1}^n[y'_{i}-f(\tr(\mathbf{X'}_{i}^\mathsf{T}\beta_{t+1}/\|\beta_{t+1}\|_F))]^2$ or
 		$\frac{1}{n}\sum_{i=1}^n[y'_{i}-f(\tr(\mathbf{X'}_{i}^\mathsf{T}\beta_t/\|\beta_t\|_F))]^2$} is minimized over all iterations, then
 		output the final estimate $\widehat \beta=\beta_{\tilde{t}}$..
 	}\label{alg4}
 	\caption{ Algorithm for Low Rank Matrix SIM with Known Link Function}
 \end{algorithm}

\clearpage

\newpage
\bibliographystyle{ims}
\bibliography{implicit}
\end{document}